\documentclass[11pt,a4paper]{article}
\usepackage[utf8]{inputenc}
\usepackage[T1]{fontenc}
\usepackage{lmodern}
\usepackage[english]{babel}
\usepackage[margin=2.5cm]{geometry} 
\usepackage{graphicx} 
\usepackage{booktabs} 
\usepackage{natbib} 
\usepackage{setspace} 
\usepackage{caption} 
\usepackage[colorlinks=true, citecolor=blue, linkcolor=blue, urlcolor=blue]{hyperref} 
\usepackage{subcaption} 
\usepackage{tabularx}
\usepackage{multirow}

\usepackage{kotex}
\usepackage{amsmath,amssymb,amsthm}
\usepackage{enumitem} 
\usepackage{algorithm}
\usepackage{algorithmic}
\usepackage{iftex}
\usepackage[table]{xcolor}
\usepackage[most]{tcolorbox}
\usepackage{tikz}
\usepackage{float}

\newtheorem{theorem}{Theorem}
\newtheorem{example}[theorem]{Example}

\newtheorem{lemma}[theorem]{Lemma}
\newtheorem{proposition}[theorem]{Proposition}
\newtheorem{remark}[theorem]{Remark}

\newtheorem{definition}[theorem]{Definition}
\newtheorem{assumption}{Assumption}

\newcommand{\argmax}{\mathop{\mathrm{argmax}}}

\newcommand{\E}{\mathbb{E}}

\newcommand{\KL}{\mathrm{KL}}
\newcommand{\TV}{d_{\mathrm{TV}}}
\newcommand{\Ham}{d_{\mathrm{Ham}}}

\newcommand{\Gap}{\Delta_\eta}

\DeclareMathOperator{\Var}{Var}

\setcitestyle{round}

\let\oldfootnote\footnote
\renewcommand{\footnote}{\fontsize{9}{11}\selectfont\oldfootnote}

\title{Privacy-Preserving Reinforcement Learning from Human Feedback via Decoupled Reward Modeling}
\author{
    Young Hyun Cho \\
    Purdue University \\
    \and
    Will Wei Sun\thanks{Corresponding author: sun244@purdue.edu} \\
    Purdue University
}
\date{} 

\begin{document}

\maketitle

\begin{abstract}
Preference-based fine-tuning has become an important component in training large language models, and the data used at this stage may contain sensitive user information. A central question is how to design a differentially private pipeline that is well suited to the distinct structure of reinforcement learning from human feedback. We propose a privacy-preserving framework that imposes differential privacy only on reward learning and derives the final policy from the resulting private reward model. Theoretically, we study the suboptimality gap and show that privacy contributes an additional additive term beyond the usual non-private statistical error. We also establish a minimax lower bound and show that the dominant term changes with sample size and privacy level, which in turn characterizes regimes in which the upper bound is rate-optimal up to logarithmic factors. Empirically, synthetic experiments confirm the scaling predicted by the theory, and experiments on the Anthropic HH-RLHF dataset using the Gemma-2B-IT model show stronger private alignment performance than existing differentially private baseline methods across privacy budgets.
\end{abstract}




\section{Introduction}\label{sec:intro}

Large language models (LLMs) are increasingly used as general-purpose tools across a growing range of downstream and domain-specific applications, including medicine, finance, law, and science \citep{nie2024survey,qin2024exploring,zheng2025large,zhou2026demystifying}. In practice, adapting a pretrained model to such settings typically relies on post-training or fine-tuning, which has become a standard part of modern LLM deployment \citep{ouyang2022training,bai2022training,chia2025post}. Preference-based LLM fine-tuning is one widely used form of this adaptation. Such fine-tuning is commonly carried out through reinforcement learning from human feedback (RLHF) \citep{christiano2017deep,stiennon2020learning,ouyang2022training}, where comparative feedback is used to improve a policy when reward feedback is not directly observed.

In particular, a typical LLM alignment pipeline begins with a pretrained model, obtains a reference policy through supervised fine-tuning (SFT) on expert-written or carefully curated instruction-response pairs, and then refines that policy using pairwise preferences over candidate responses, typically through RLHF. Because SFT data are costly to scale and may not fully capture the nuanced judgments needed for alignment, preference-based fine-tuning has become a common extension beyond SFT \citep{chia2025post}. Figure~\ref{fig:llm_pipeline} illustrates this canonical RLHF pipeline. Modern preference-based alignment also includes direct formulations such as direct preference optimization (DPO), which bypass explicit reward-model training and instead optimize the policy directly from preference data \citep{rafailov2024direct}.

\begin{figure}[t]
    \centering
    \includegraphics[width=0.82\linewidth]{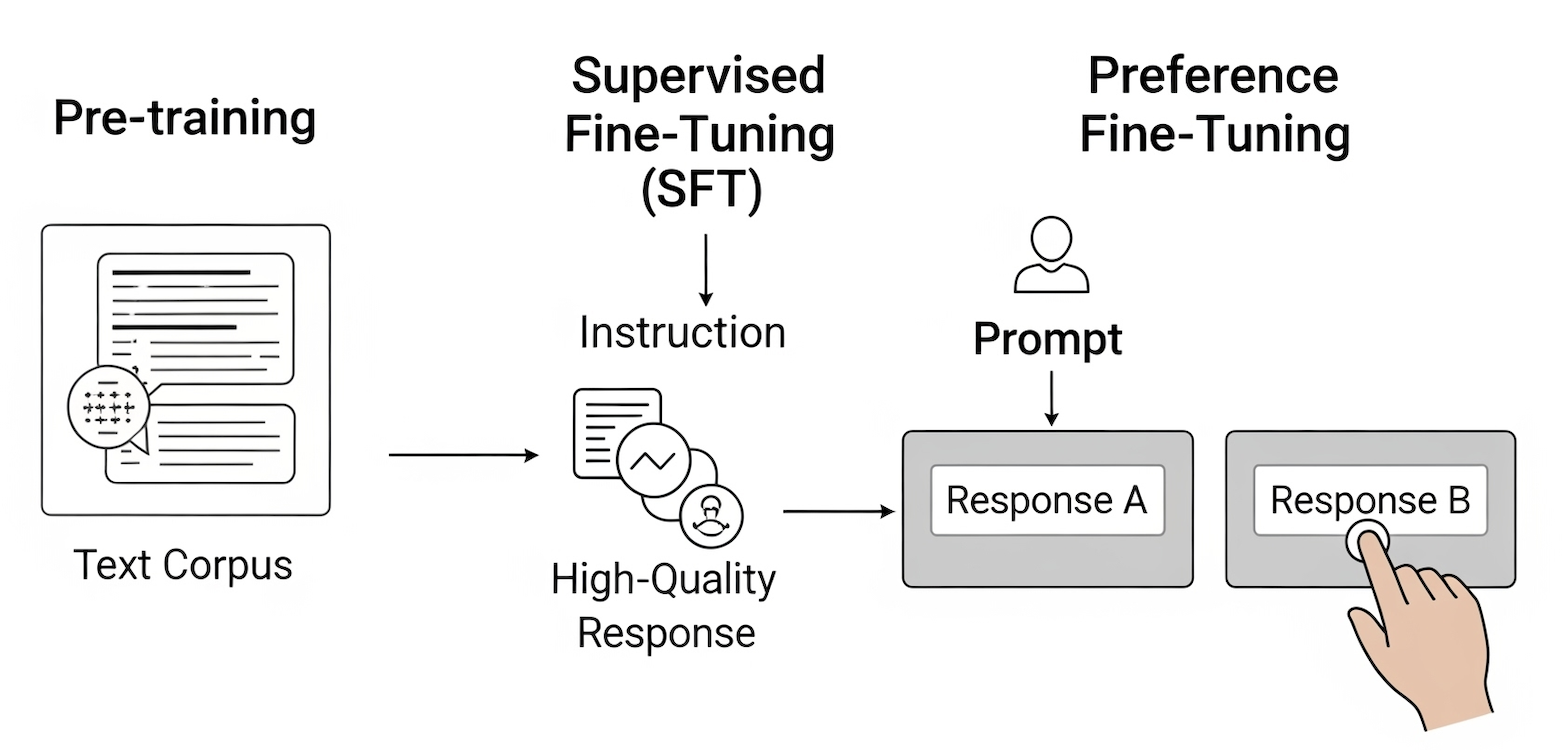}
    \caption{A typical large language model adaptation pipeline. We focus on privacy during the preference fine-tuning
    stage, where sensitive user interactions can be directly reflected in training records.}
    \label{fig:llm_pipeline}
\end{figure}

On the other hand, this preference-based fine-tuning pipeline can raise privacy concerns. Modern models are susceptible to training-data extraction attacks \citep{carlini2021extracting,nasr2023scalable} and membership inference attacks \citep{shokri2017membership,wu2025membership}, which can reveal whether a user’s data was used for training and, in some cases, expose training examples. In preference-based LLM fine-tuning beyond SFT, the training signal is not merely a scalar label, but a full interaction tuple \((x_i,a_i^1,a_i^2,y_i)\), where \(x_i\) is the user prompt, \((a_i^1,a_i^2)\) are candidate responses, and \(y_i\) is the associated preference label. Even when explicit identifiers are absent, the prompt \(x_i\) may contain sensitive or potentially identifying context, and the derived responses and label may also reveal private user information. Figure~\ref{fig:sensitive_dialogue} illustrates such an example. These concerns call for principled protection against leakage at the tuple level.

\begin{figure}[t]
    \centering
    \includegraphics[width=\linewidth]{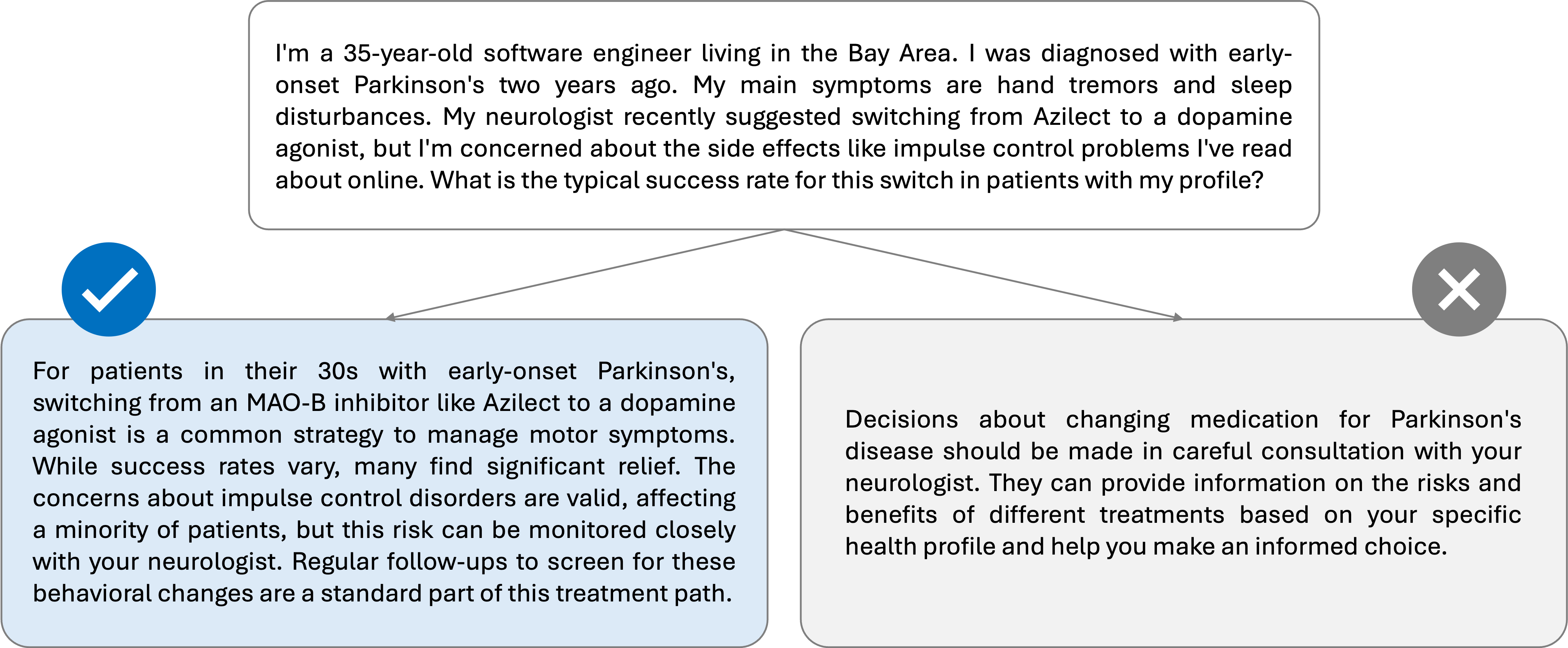}
    \caption{A sensitive interaction record. Even without direct identifiers, prompts can contain quasi-identifiers
    whose combination may re-identify an individual, motivating tuple-level protection of $(x_i,a_i^1,a_i^2,y_i)$.}
    \label{fig:sensitive_dialogue}
\end{figure}

Differential privacy (DP) \citep{dwork2006calibrating} provides a widely used formal framework for privacy-preserving data analysis and machine learning. A central question is what object should be protected by the privacy definition. Much of the existing literature adopts label-DP \citep{chowdhury2024differentially,zhang2025towards,teku2025props,wu2025offline}, which protects only the preference label and is therefore most naturally aligned with protecting the privacy of the preference-label annotator. Our focus is different. In the settings we study, the relevant target is the end-user, so privacy should apply to the full interaction tuple rather than to \(y_i\) alone. This motivates tuple-level privacy for \((x_i,a_i^1,a_i^2,y_i)\).

A key distinction from standard RL is that RLHF does not receive rewards directly from the environment. Instead, it must infer a latent reward signal from preference data, which introduces an additional reward-learning layer into the pipeline. Under DP, this distinction is consequential. Privacy can now enter at reward estimation, at downstream policy optimization, or at both stages. If privacy is imposed naively on such a multi-stage pipeline, noise can accumulate across stages, or the data or privacy budget must be divided across them. In this sense, the reward-learning layer in RLHF is not only a modeling feature but also a central channel through which privacy constraints affect utility.

This raises a second central question beyond what should be protected. One must also decide where privacy should enter the alignment pipeline so that downstream utility is preserved as much as possible. Existing directions already illustrate several possibilities, including introducing DP to DPO \citep{chen2025improved} and introducing DP to each stage of RLHF \citep{wu2024privately}. In this sense, existing methods largely arise by applying DP to existing non-private alignment pipelines. The remaining question is whether RLHF admits a privacy-aware design that uses its own structure more directly.

In this work, we take a different route. Rather than adding DP to each step of an existing alignment pipeline, we use the two-stage structure of RLHF itself to guide where privacy should enter. We develop a tuple-level private framework that places privacy on reward learning and treats downstream policy improvement as post-processing. We then study the resulting privacy--utility tradeoff through theory and numerical experiments. 

\subsection{Our Contributions}\label{subsec:contrib}
Our contributions are summarized as follows.
\begin{itemize}
    \item \textbf{Methodology.}
We propose a private RLHF framework that exploits a structural feature specific to RLHF, rather than simply adding DP to an existing alignment pipeline. Unlike standard RL, where the environment directly provides reward, RLHF first uses preference data to train a reward model and then updates a reference policy using that learned reward. Our framework places DP only on the reward-learning stage and treats downstream policy as post-processing of the resulting private reward model. Placing privacy on reward learning creates a buffer between DP noise and the downstream policy. In LLM fine-tuning, this avoids privatizing the policy itself, which can yield responses worse than the reference policy, and instead uses the private reward model only to re-rank reference-generated responses, without additional privacy cost.

\item \textbf{Theory.} We study the suboptimality gap of the policy induced by our framework and show that privacy enters as an additional additive term on top of the usual non-private statistical error. We also establish, to the best of our knowledge, the first minimax lower bound for this private RLHF problem. A key difficulty is that the dominant lower-bound term changes with the balance between sample size and privacy budget. We therefore characterize the regimes in which our method is rate-optimal up to logarithmic factors, including the case where the sample size is sufficiently large relative to the privacy scale for the optimal privacy-dominated rate to emerge.

\item \textbf{Empirical validation.} We validate the framework on synthetic preference-learning experiments and on an LLM alignment task based on the Anthropic HH-RLHF dataset \citep{bai2022training} using the Gemma-2B-IT model \citep{team2024gemma}. The synthetic results confirm the theoretically predicted scaling with sample size, privacy budget, and dimension, and show that our method substantially reduces the incidence of policies that underperform the reference policy relative to existing DP baseline methods. In the LLM experiment, our method demonstrates stronger private alignment performance than existing DP baseline methods across all privacy budgets considered.
\end{itemize}

\subsection{Related Work}\label{subsec:related_work}
We briefly review prior work on preference-based policy learning and its theoretical foundations, as well as DP for alignment.

\noindent\textbf{Preference-based policy learning and KL-regularized preference optimization.} One of the most widely used applications of RLHF is preference fine-tuning for LLMs \citep{christiano2017deep, stiennon2020learning, ouyang2022training, bai2022training,ye2025robust}. More recently, direct alignment methods have gained prominence by optimizing a closed-form objective that avoids explicit reward-model training and RL in the loop \citep{rafailov2024direct, garg2025ipo}. From a theoretical viewpoint, a growing line of work develops finite-sample guarantees of KL-regularized preference optimization and iterative procedures, clarifying when logged preference data suffices for reliable policy improvement \citep{xiong2024iterative, ye2024online, xiong2024iterative, song2024importance, zhao2024sharp}. Several recent works also revisit the role of the KL regularizer itself---including its interpretation, limitations, and variants \citep{huang2024correcting, aminian2025theoretical, liu2025rethinking, li2026towards}. Our work complements this theory by quantifying how DP affects the KL-regularized preference optimization.

\noindent\textbf{DP and privacy-preserving alignment.} DP is a \textit{de facto} standard in privacy-preserving framework in machine learning \citep{dwork2006calibrating}. A large literature studies optimal accuracy guarantees under DP for empirical risk minimization and stochastic convex optimization \citep{dwork2014algorithmic,chaudhuri2011differentially,kifer2012private,bassily2014private,wang2017differentially,bassily2019private}. Among many, the most widely used optimization method is noisy gradient-based training, popularized by DP-SGD and closely related mechanisms \citep{dwork2006calibrating,abadi2016deep,bu2020deep}. This line of work has also motivated refined privacy accounting frameworks and composition results that tighten end-to-end privacy loss in iterative training \citep{kairouz2015composition,mironov2017renyi,bun2016concentrated,dong2022gaussian}.

Building on these foundations, recent work has begun to study DP for preference-based alignment and RLHF. A first key distinction concerns the unit of protection. Much of the recent literature studies label-DP, which protects only the preference label while treating the prompt and candidate responses as public or non-sensitive. This formulation is naturally aligned with protecting annotator feedback, but it does not address leakage from the text content itself. This perspective underlies several recent works on private DPO and related offline alignment objectives \citep{zhang2025towards,zhou2025unified,zhou2025square,teku2025props}. In contrast, our focus is tuple-level privacy, where the prompt, candidate responses, and preference outcome may all carry sensitive information. 

A second distinction concerns how privacy is incorporated into the alignment pipeline. Most existing approaches begin with a non-private alignment framework and then introduce a DP mechanism on top of it. When the target is label-DP, a standard way to privatize the labels is randomized response \citep{warner1965randomized}, and several recent works follow this route \citep{zhang2025towards,zhou2025unified,zhou2025square,teku2025props}. When the target is tuple-level DP, privacy is typically enforced through noisy gradient-based training such as DP-SGD, including private adaptations of DPO and related RLHF procedures \citep{chen2025improved,wu2024privately,wu2025offline,chowdhury2024differentially,korkmaz2024learning}. Our contribution differs from these directions both methodologically and theoretically. We study a framework tailored to the distinct structure of RLHF, with privacy imposed only on reward learning, and we analyze the resulting suboptimality gap together with minimax lower bounds that characterize when the corresponding rates are optimal up to logarithmic factors.

\subsection{Paper organization and notation}\label{subsec:org-notation} 
The remainder of the paper is organized as follows. Section~\ref{sec: preliminary} introduces the preliminaries for our work, including DP and RLHF. Section~\ref{sec: proposed method} presents the proposed decoupled framework. In section~\ref{sec: theory}, we develop the theoretical analysis, deriving both upper and lower bounds on the suboptimality gap of the induced policy and identifying regimes in which these bounds match up to logarithmic factors. Section~\ref{sec: numerical} reports numerical studies on both synthetic examples and LLM fine-tuning experiments. Section~\ref{sec:discussion_conclusion} concludes with a discussion of the main implications, limitations, and directions for future work. Most proofs are deferred to the Supplementary Material Section~\ref{supp:proofs}, while straightforward arguments are given in the main text.

For notation, we use standard asymptotic order notation. For two positive sequences \(a_n\) and \(b_n\), \(a_n = O(b_n)\) means that \(a_n / b_n\) is uniformly bounded, and \(a_n = \Omega(b_n)\) means that \(a_n / b_n\) is bounded away from zero. We write \(a_n = \widetilde{O}(b_n)\) and \(a_n = \widetilde{\Omega}(b_n)\) when the corresponding relations hold up to polylogarithmic factors.

\section{Preliminaries}
\label{sec: preliminary}

In this section, we establish the background for our work. We begin by formally outlining the standard pipelines for preference fine-tuning, then introduce the background on DP.

\subsection{Reinforcement Learning from Human Feedback}
A widely used template for learning from preferences is KL-regularized reinforcement learning from human feedback (RLHF).
The key object is a reference policy $\pi_0$, and the goal is to improve decision making while preventing the updated policy from drifting too far from $\pi_0$ or overfitting to a limited preference dataset.
This viewpoint is especially prominent in large language model alignment, where preference fine-tuning is a standard stage after a strong base policy is obtained.

We consider an offline preference dataset originally recorded as \(\mathcal D=\{(x_i,a_i^1,a_i^2,y_i)\}_{i=1}^n\), where \(x_i\in\mathcal X\) denotes a context, such as a prompt in LLM applications, \(a_i^1,a_i^2\in\mathcal A\) are two candidate actions, and \(y_i\) indicates which candidate is preferred. For the theoretical development, it is more convenient to rewrite each record in ordered form as \((x_i,a_i^w,a_i^l)\), where \(a_i^w\in\mathcal A\) and \(a_i^l\in\mathcal A\) denote the preferred and non-preferred actions, respectively. A conventional modeling assumption is then the Bradley--Terry preference model \citep{bradley1952rank}.

\begin{definition}[Bradley--Terry Model]\label{def:bradley-terry-model}
For a context \(x\) and a preferred/non-preferred pair \((a^w,a^l)\), the probability that \(a^w\) is preferred to \(a^l\) is modeled as
\[
\mathbb{P}(a^w \succ a^l \mid x)
=
\frac{\exp(r^*(x,a^w))}{\exp(r^*(x,a^w))+\exp(r^*(x,a^l))}
=
\sigma\!\big(r^*(x,a^w)-r^*(x,a^l)\big),
\]
where \(\sigma(t)=(1+e^{-t})^{-1}\) is the sigmoid function.
\end{definition}

In KL-regularized RLHF, the reference policy \(\pi_0\) is treated as a strong baseline that encodes prior knowledge and safe behavior, and KL regularization controls the magnitude of the preference-driven update. In the LLM preference fine-tuning pipeline, a common choice of \(\pi_0\) is a model that has already been trained on a large supervised instruction-following dataset, often referred to as supervised fine-tuning, or SFT \citep{ouyang2022training}. This choice reflects the practical role of \(\pi_0\) as an anchor to a broadly competent response distribution, while preference data provide an additional signal that refines behavior without requiring the policy to relearn basic capabilities from scratch.

Under Definition~\ref{def:bradley-terry-model}, the standard reward-modeling phase estimates a parametric reward function \(r_\theta \in \mathcal{R}=\{r_\theta:\theta\in\Theta\}\) by maximizing the log-likelihood
\begin{equation}\label{equ:reward_loss}
\hat{\theta}
=
\argmax_{\theta \in \Theta}
\sum_{i=1}^{n}
\log \sigma \left( r_\theta(x_i, a_i^{w}) - r_\theta(x_i, a_i^{l}) \right).
\end{equation}

Once a reward estimator $\hat r$ is obtained, the goal is to derive a policy that achieves high reward while remaining close to a reference policy $\pi_0$, which is often the SFT model in LLM pipelines. For any reward function $r$ on $\mathcal X\times\mathcal A$, we evaluate a policy $\pi$ by the KL-regularized value
\begin{equation}\label{eq:value_def}
V_\eta(\pi; r)
=
\mathbb{E}_{x \sim d_0}\!\left[
\mathbb{E}_{a \sim \pi(\cdot\mid x)}\big[r(x, a)\big]
-\frac{1}{\eta}\mathrm{KL}\big(\pi(\cdot\mid x)\,\Vert\,\pi_0(\cdot\mid x)\big)
\right],
\end{equation}
where $d_0$ denotes the context distribution, and $\eta>0$ controls the strength of regularization. Let $\pi_r^\eta\in\arg\max_\pi V_\eta(\pi;r)$ denote the KL-regularized optimizer under $r$. The closed form solution is as follows:

\begin{lemma}[Policy Improvement Oracle]\label{lemma:pio}
For a fixed context $x$, reward $r$, and reference policy $\pi_0$, any maximizer $\pi_r^\eta\in\arg\max_\pi V_\eta(\pi;r)$ satisfies
\begin{equation}\label{equ:PIO}
\pi_{r}^{\eta}(a \mid x)
=
\frac{1}{Z_{r}(x)}\,\pi_{0}(a \mid x)\,\exp\!\big(\eta\,r(x,a)\big),
\end{equation}
where $Z_r(x)=\mathbb{E}_{a \sim \pi_0(\cdot\mid x)}\!\left[\exp\!\big(\eta r(x,a)\big)\right]$.
\end{lemma}

Although $\pi_r^\eta$ admits a closed-form expression, exact sampling from this policy is typically infeasible when the action space is large. The expression involves a normalization term that aggregates $\exp(\eta r(x,a))$ over the action space, which is computationally prohibitive in LLM settings where actions correspond to long token sequences. This motivates practical methods that approximate the KL-regularized optimizer through parameterized learning.

A common approach is to optimize a parameterized policy with reinforcement learning algorithms such as proximal policy optimization (PPO) \citep{schulman2017proximal}, which targets the KL-regularized objective while avoiding explicit normalization over $\mathcal A$. Since PPO does not appear explicitly in the main development, we defer a brief background discussion to Appendix~\ref{app:ppo_background}. An alternative is Direct Preference Optimization (DPO) \citep{rafailov2024direct}, which leverages the identity implied by \eqref{equ:PIO}. In particular, for $\pi_r^\eta$ one can write
\[
r(x,a)
=
\frac{1}{\eta}\log\frac{\pi_r^\eta(a\mid x)}{\pi_0(a\mid x)}
+\frac{1}{\eta}\log Z_r(x).
\]
Substituting this relation into the pairwise preference likelihood \eqref{equ:reward_loss} eliminates the unknown $Z_r(x)$ and yields a supervised objective over policy parameters
\begin{equation}\label{equ:dpo_loss}
\mathcal{L}_{\mathrm{DPO}}(\theta; \mathcal{D}, \pi_{0})
=
- \sum_{(x, a^w, a^l) \in \mathcal{D}}
\log \sigma \left(
\frac{1}{\eta} \log \frac{\pi_\theta(a^w \mid x)}{\pi_0(a^w \mid x)}
-
\frac{1}{\eta} \log \frac{\pi_\theta(a^l \mid x)}{\pi_0(a^l \mid x)}
\right).
\end{equation}
The primary appeal of DPO and its variants \citep{azar2024general, ethayarajh2024kto, xu2024contrastive, meng2024simpo} lies in their ability to avoid the computational complexities and training instabilities of RL frameworks. By formulating the process as a single-stage supervised learning, these methods eliminate the need for iterative sampling and complex hyperparameter tuning required by RL-based optimization. However, as we discuss in Section~\ref{sec: proposed method}, directly imposing DP on the policy parameters in \eqref{equ:dpo_loss} introduces fundamental challenges.

Our analysis focuses on the policy quality under the true reward $r^\ast$. Let $\pi_{r^\ast}^\eta\in\arg\max_\pi V_\eta(\pi;r^\ast)$.
We quantify the $\eta$-regularized suboptimality of a candidate policy $\pi$ by
\begin{equation}\label{eq:gap_def}
\Delta_\eta(\pi)
=
V_\eta(\pi_{r^\ast}^{\eta}; r^\ast) - V_\eta(\pi; r^\ast).
\end{equation}
In the proposed decoupled framework, $\pi$ will be induced by a private reward estimate, and our theory controls $\Delta_\eta(\pi)$ in finite samples.

\subsection{Differential Privacy}

Informally, a mechanism $M$ is DP if the distribution of its output $M(D)$ is nearly indistinguishable from that of $M(D')$ for any adjacent dataset $D'$, where $D$ and $D'$ differ by a single entry. We formalize the notion of adjacent as their Hamming distance is one.

Among various characterizations of DP--such as R{\'e}nyi DP \citep{mironov2017renyi} or Gaussian DP \citep{dong2022gaussian}--aimed at achieving tighter privacy accounting, the following $(\varepsilon, \delta)$-DP remains the most prevalent and standard definition, and thus we also adopt the following:

\begin{definition}[\((\varepsilon,\delta)\)-Differential Privacy \citep{dwork2006calibrating}]\label{def: eps del dp}
A mechanism $M$ is \((\varepsilon,\delta)\)-differentially private if, for any two adjacent \(D, D'\) and for any measurable event \(E\),
\[
\mathbb{P}\big(M(D) \in E\big) \leq e^{\varepsilon}\,\mathbb{P}\big(M(D') \in E\big) + \delta.
\]
When \(\delta = 0\), the mechanism $M$ is said to satisfy \(\varepsilon\)-differential privacy.
\end{definition}

DP protects individuals by requiring that an algorithm produce similar output distributions on adjacent datasets that differ in a single record. This similarity limits the influence of any one data point on what is released, so an observer cannot reliably infer whether that record was included. The $(\varepsilon,\delta)$ definition measures similarity through a log-likelihood ratio control between the two output distributions for every measurable event. The parameter $\varepsilon$ bounds the magnitude of this log-likelihood ratio, while $\delta$ allows a small probability mass on which the bound may be violated, which can be viewed as an $\varepsilon$-DP guarantee holding with probability at least $1-\delta$. Smaller values of $\varepsilon$ and $\delta$ therefore correspond to stronger privacy protection, and $\delta$ is typically chosen to be extremely small in practice, since it quantifies the probability of a rare failure event the privacy protection may not hold.

A key appeal of DP is that its guarantees are accompanied by explicit parameters that can be tracked across an entire pipeline, which makes it possible to quantify the overall level of protection. In particular, DP enjoys the following properties:

\begin{itemize}
    \item \textbf{Post-processing:}
    If $M$ is $(\varepsilon,\delta)$-DP, then for any mapping $\mathrm{Proc}$ independent of the data,
    the post-processed mechanism $\mathrm{Proc}\circ M$ is also $(\varepsilon,\delta)$-DP.
    
    \item \textbf{Sequential composition (basic):}
Let $M_1:\mathcal{X}\to\mathcal{Y}_1$ and  $M_i:\mathcal{X}\times\mathcal{Y}_1\times\cdots\times\mathcal{Y}_{i-1}\to\mathcal{Y}_i,$ for $i=2,\ldots,k$. Let $M(D):=\bigl(M_1(D),\,M_2(D,y_1),\,\ldots,\,M_k(D,y_1,\ldots,y_{k-1})\bigr)$ denote the joint mechanism, where the outputs are generated recursively.
If, for every fixed previous output history, each step is
$(\varepsilon_i,\delta_i)$-DP, then $M$ is $\Bigl(\sum_{i=1}^k \varepsilon_i,\ \sum_{i=1}^k \delta_i\Bigr)\text{-DP}$.
    
    \item \textbf{Parallel composition:}
    Let $D$ be partitioned into disjoint subsets $D_1,\ldots,D_k$,
    and let $M_i$ be an $(\varepsilon_i,\delta_i)$-DP mechanism applied only to $D_i$.
    Define the joint mechanism $M(D):=\big(M_1(D_1),\ldots,M_k(D_k)\big)$.
    Then $M$ is $\Big(\max_{1\le i\le k}\varepsilon_i,\ \max_{1\le i\le k}\delta_i\Big)$-DP.
\end{itemize}

We present the basic sequential composition rule only for illustration, and the main point is that privacy parameter accumulates across multiple data-dependent releases. In practice, one can obtain tighter accounting by using refined composition theorems and privacy definitions designed for sharper tracking of cumulative loss, such as R\'enyi DP or Gaussian DP. We do not pursue these refinements here since our focus is on the pipeline-level design and on policy-quality guarantees, rather than on the tightest possible accounting constants.

Constructing a DP mechanism typically involves injecting noise, calibrated to the \textit{sensitivity}, the maximum impact of a single data point on the output. In machine learning where gradient-based optimization is standard, the contribution of an individual data is its respective gradient. Therefore, a natural pathway to DP is to inject noise into the gradient updates, leading to the DP-stochastic gradient descent (DP-SGD) \citep{abadi2016deep}.

Concretely, let $\ell(\theta; z)$ denote a per-example loss function and $\Theta \subset \mathbb{R}^d$ be a parameter space. A standard minibatch DP-SGD update is defined as:
\begin{equation}\label{eq:dpsgd_update}
\theta_{t+1} = \Pi_{\Theta} \left( \theta_t - \eta_t \left( \frac{1}{|B_t|} \sum_{i \in B_t} \text{clip}(\nabla_\theta \ell(\theta_t; Z_i), C) + \xi_t \right) \right), \quad \xi_t \sim \mathcal{N}(0, \sigma_{\mathrm{DP}}^2 C^2 I_d),
\end{equation}
where $B_t \subset [n]$ is a minibatch, $\Pi_{\Theta}$ denotes the projection onto the set $\Theta$, and $\text{clip}(g, C) := g \cdot \min\{1, C/\|g\|_2\}$ scales each gradient to ensure a uniform $\ell_2$-sensitivity bound $C$. While sensitivity may not be inherently bounded or analytically tractable--particularly for black-box models-- this gradient clipping ensures that the sensitivity is always bounded by $C$, which is then used to calibrate the noise. As such, gradient clipping serves as a key tool that provides the flexibility to satisfy DP for arbitrary differentiable architectures.


Other optimization approaches are primarily distinguished by the stage at which noise is introduced, such as objective perturbation or output perturbation \citep{chaudhuri2011differentially}. For large-scale tasks, however, DP-SGD remains the \textit{de facto} standard and the most widely used scalable paradigm in practice, supported by implementations such as the \texttt{Opacus} library in Python \citep{yousefpour2021opacus}.

\section{Proposed Method}\label{sec: proposed method}


\subsection{Motivation: Challenges in Private Policy Optimization}\label{sec:motivation}

The prevailing paradigms in finetuning, such as DPO and PPO, optimize the policy parameters \(\pi_\theta\) directly. We defer a more detailed description of PPO to Appendix~\ref{app:offline_ppo}, and focus here on the common issue that arises when DP is imposed on policy optimization.

The fundamental difficulty stems from the inherent incompatibility between the unbounded nature of policy gradients and the sensitivity required by DP. In frameworks like PPO or DPO, the gradient involves the score function 
\begin{align*}
\nabla_\theta \log \pi_\theta(a \vert x) = \frac{\nabla \pi_{\theta}(a \vert x)}{\pi_{\theta}(a \vert x)},    
\end{align*}
which lacks a uniform upper bound since $\pi_\theta(a\vert x)$ can be arbitrarily close to zero, making the ratio arbitrarily large. As these policy gradients can be arbitrarily large, the aforementioned clipping during the gradient updates are necessary.

\begin{example}[DPO]\label{ex:dpo-dp}
Consider the DPO objective in~\eqref{equ:dpo_loss}. For a single preference pair $(x,a^w,a^l)$, define
\[
z_\theta(x,a^w,a^l)
:=
\frac{1}{\eta}\log \frac{\pi_\theta(a^w \vert x)}{\pi_0(a^w \vert x)}
-
\frac{1}{\eta}\log \frac{\pi_\theta(a^l \vert x)}{\pi_0(a^l \vert x)}.
\]
The per-pair loss is $\ell(\theta)=-\log\sigma\!\big(z_\theta(x,a^w,a^l)\big)$. Since $\pi_0$ is fixed, the gradient is
\[
\nabla_\theta \ell(\theta)
=
-\big(1-\sigma(z_\theta)\big)\,\nabla_\theta z_\theta
=
-\frac{1}{\eta}\big(1-\sigma(z_\theta)\big)
\Big(\nabla_\theta \log \pi_\theta(a^w\mid x)-\nabla_\theta \log \pi_\theta(a^l\mid x)\Big).
\]
This expression depends on the score function $\nabla_\theta\log\pi_\theta(a\mid x)$, which generally admits no uniform bound because $\pi_\theta(a\mid x)$ can be arbitrarily small. As a result, per-example gradients can be arbitrarily large under standard policy parameterizations, and DP training therefore requires explicit sensitivity control through per-example gradient clipping.
\end{example}

\begin{example}[PPO-style policy optimization]\label{ex:ppo-dp}
A similar issue arises for PPO-style policy optimization. At a high level, PPO uses an advantage signal to indicate whether a sampled action should become more likely or less likely under the updated policy, while constraining the update so that the policy does not move too far in a single step. Here the advantage \(A\) is a scalar quantity summarizing how favorable action \(a\) is at context \(x\) relative to a baseline. The PPO-style ratio
\[
\rho_\theta(x,a)
:=
\frac{\pi_\theta(a\mid x)}{\pi_{\mathrm{ref}}(a\mid x)}
=
\exp\!\Big(
\log \pi_\theta(a\mid x)-\log \pi_{\mathrm{ref}}(a\mid x)
\Big)
\]
measures how much the updated policy reweights action \(a\) relative to a fixed reference policy \(\pi_{\mathrm{ref}}\).
Given a parameter \(\varepsilon_{\mathrm{clip}}>0\), the clipped surrogate takes the form
\[
\ell_{\mathrm{PPO}}(\theta)
=
-\min\!\Big\{
\rho_\theta(x,a)\,A,\;
\mathrm{clip}\!\big(\rho_\theta(x,a),1-\varepsilon_{\mathrm{clip}},1+\varepsilon_{\mathrm{clip}}\big)\,A
\Big\}.
\]
Whenever the unclipped branch is active, differentiating with respect to \(\theta\) yields
\[
\nabla_\theta \rho_\theta(x,a)
=
\rho_\theta(x,a)\,\nabla_\theta \log \pi_\theta(a\mid x),
\]
so the gradient again depends on the score function \(\nabla_\theta \log \pi_\theta(a\mid x)\). Under standard policy parameterizations, this score function generally admits no uniform bound because \(\pi_\theta(a\mid x)\) can be arbitrarily small. Thus, PPO clipping controls the objective through the policy ratio, but it does not remove the need for explicit per-example gradient clipping when DP is imposed.
\end{example}

Clipping plays two roles in private training. It enforces a sensitivity bound, but it can also distort optimization when many per-example gradients exceed the clipping threshold. Let $F(\theta)=\mathbb{E}_z[\ell(\theta;z)]$ and define the minibatch gradient $\widehat{\nabla}F_t(\theta)$. Under standard sampling assumptions, $\widehat{\nabla}F_t(\theta)$ is an unbiased estimator of $\nabla F(\theta)$, whereas replacing per-example gradients by $\mathrm{clip}(\nabla_\theta\ell(\theta;z),C)$ generally introduces bias.

This distinction is especially relevant in policy optimization, where per-example gradients can be heavy-tailed due to score-function terms, so clipping may be frequently active. In contrast, when the learning target admits a uniform per-example gradient bound, one can choose $C$ to match this bound so that clipping is rarely active and introduces negligible distortion. In that case, privacy is enforced primarily through additive noise calibrated to the sensitivity bound, which preserves mean-zero updates while inflating variance.

To address these challenges, we place the privacy mechanism on reward learning rather than on policy optimization. The reward-modeling objective is typically better conditioned, and in many LLM pipelines it is implemented by freezing a pretrained backbone and training a lightweight head on top of its representations \citep{evci2022head2toe,houlsby2019parameter,hu2022lora}. Let $\phi(x,a)\in\mathbb R^d$ denote the fixed representation of a context action pair, and consider a linear-head reward model
\begin{equation}
\label{eq:linear_reward}
r_\theta(x,a)=\langle \phi(x,a),\theta\rangle,
\end{equation}
with $\theta\in\Theta=\{\theta\in\mathbb R^d\mid \|\theta\|_2\le R\}$. The key point is that sensitivity control becomes transparent once the per-example loss is Lipschitz in $\theta$. A convenient way to ensure this is to keep the representation norm bounded, for instance by applying a final normalization step on the backbone features. In modern Transformers, LayerNorm and RMSNorm already stabilize feature scales in practice \citep{ba2016layer,zhang2019root,zheng2024learn}, and an explicit final projection or normalization can enforce a deterministic bound $\sup_{x,a}\|\phi(x,a)\|_2 \le L.$

Under the Bradley--Terry model, each observation $z=(x,a^{w},a^{\ell})$ induces a convex loss $\ell(\theta;z)$ whose gradient has the form
\begin{equation}
\label{eq:bt_grad_form}
\nabla_\theta \ell(\theta;z)
=\alpha(\theta;z)\bigl(\phi(x,a^{w})-\phi(x,a^{\ell})\bigr),
\qquad
\alpha(\theta;z)\in[0,1].
\end{equation}
Therefore, we have
\begin{equation}
\label{eq:grad_bound}
\|\nabla_\theta \ell(\theta;z)\|_2
\le
\|\phi(x,a^{w})-\phi(x,a^{\ell})\|_2
\le
\|\phi(x,a^{w})\|_2+\|\phi(x,a^{\ell})\|_2
\le
2L.
\end{equation}
This yields a uniform per-example gradient bound, which provides a clean sensitivity control for private optimization. This suggests that per-example gradients in head-only reward learning are typically better behaved than in policy optimization. As a result, when the same per-example gradient clipping norm \(C\) is used to control DP sensitivity, clipping tends to be less frequently active in reward learning than in policy optimization. In that sense, privacy in reward learning is driven more by calibrated noise than by clipping-induced distortion.

\subsection{Proposed Framework: Private Reward-Based Alignment}\label{subsec:proposed_framework}

We propose a decoupled framework that concentrates the privacy expenditure solely on estimating the reward structure and derives the final decision rule via post-processing.

\begin{algorithm}[htb!]
\caption{Differentially Private Reward-Based Alignment}\label{alg:framework_general}
\begin{algorithmic}[1]
    \STATE \textbf{Input:} Preference dataset $\mathcal{D} = \{(x_i, a_i^w, a_i^l)\}_{i=1}^n$, privacy budget $(\varepsilon, \delta)$, KL-regularization parameter $\eta$, reference policy $\pi_0$, sampling budget $N$ (optional).
    
    \STATE \textbf{Step 1: Private Reward Learning}
    \STATE Learn a private reward model $\tilde{r}(x, a)$ by an $(\varepsilon, \delta)$-DP mechanism (e.g., DP-SGD).
    
    \STATE \textbf{Step 2: Policy Derivation (Inference)}
    \IF{normalization constant $Z_{\tilde{r}}(x) = \mathbb{E}_{a \sim \pi_0}[\exp(\eta \tilde{r}(x,a))]$ is computable}
        \STATE \textit{Exact Inference}
        \STATE Construct the optimal policy in closed form:
        \begin{equation*}
        \pi_{\tilde{r}}^{\eta}(a \vert x) = \frac{1}{Z_{\tilde{r}}(x)} \pi_0(a \vert x) \exp \left( \eta \tilde{r}(x, a) \right).
        \end{equation*}
    \ELSE
        \STATE \textit{Approximate Inference via Best-of-$N$}
        \STATE For a given context $x$, sample $N$ candidates from reference: $\{a^{(1)}, \dots, a^{(N)}\} \sim \pi_0(\cdot \vert x)$.
        \STATE Select the candidate maximizing the private reward:
        \begin{equation*}
        a^* = \mathop{\mathrm{argmax}}_{j \in \{1, \dots, N\}} \tilde{r}(x, a^{(j)}).
        \end{equation*}
        \STATE Define the policy output as the Dirac mass on $a^*$.
    \ENDIF
    
    \STATE \textbf{Output:} Privately aligned policy $\pi_{\tilde{r}}^{\eta}$ or action $a^*$.
\end{algorithmic}
\end{algorithm}

As outlined in Algorithm~\ref{alg:framework_general}, the procedure consists of two stages: (i) learning a differentially private reward model $\tilde r$ from the full preference dataset $\mathcal D$, and (ii) producing an aligned action/response induced by $\tilde r$ without any further access to $\mathcal D$.

In the first stage, we treat reward learning as a single empirical risk minimization problem on \(\mathcal D\).  We state the framework at the level of the resulting \((\varepsilon,\delta)\)-DP guarantee, rather than specifying an explicit closed-form calibration for the added Gaussian noise. In practice, the calibration of DP-SGD depends on various factors, including the clipping norm, sampling scheme, number of epochs, and privacy accounting method \citep{abadi2016deep,bu2020deep}. Since our focus is on the pipeline design and on the privacy--utility trade-off at the level of the final guarantee, rather than on accountant-specific calibration formulas, we do not make the noise level explicit here. In the theoretical development in Section~\ref{sec: theory}, we work with the projected noisy stochastic-gradient procedure of \citet{bassily2014private} as a concrete DP-SGD instantiation, since it delivers the strongly-convex excess-risk rate used in our analysis. In the experiments, the corresponding calibration is handled by \texttt{Opacus}, which takes the target privacy parameters together with the training configuration and internally performs privacy accounting to determine the required noise level; implementation details are deferred to Appendix~\ref{app:dp_accounting} and Table~\ref{tab:hyperparameters_llm}.

In the second stage, the KL-regularized target policy induced by $\tilde r$ is defined by the Gibbs form
\[
\pi_{\tilde r}^{\eta}(a\mid x)
=
\frac{\pi_0(a\mid x)\exp\{\eta\,\tilde r(x,a)\}}{Z_{\tilde r}^{\eta}(x)},
\]
where $Z_{\tilde r}^{\eta}(x):=\sum_{a'\in\mathcal A}\pi_0(a'\mid x)\exp\{\eta\,\tilde r(x,a')\}.$ The implementation depends on the tractability of the partition function $Z_{\tilde r}^{\eta}(x)$. When $Z_{\tilde r}^{\eta}(x)$ is tractable (e.g., finite action sets or structured settings where normalization is feasible), one can explicitly construct and sample from $\pi_{\tilde r}^{\eta}(\cdot\mid x)$ (or apply a deterministic rule such as $\arg\max_a \pi_{\tilde r}^{\eta}(a\mid x)$).

In contrast, when $\mathcal A$ is combinatorially large--as in preference fine-tuning for LLMs where actions correspond to long token sequences--computing $Z_{\tilde r}^{\eta}(x)$ is infeasible. In this regime, Algorithm~\ref{alg:framework_general} adopts a best-of-$N$ (BoN) inference-time policy \citep{stiennon2020learning}: draw $N$ candidates $a^{(1)},\dots,a^{(N)}\sim \pi_0(\cdot\mid x)$ and output
\[
a^* \in \arg\max_{j\in\{1,\dots,N\}} \tilde r\!\bigl(x,a^{(j)}\bigr).
\]
As $N$ increases, this selection increasingly tends to return higher-reward candidates under the proposal distribution $\pi_0(\cdot\mid x)$ by restricting attention to a richer candidate pool, without requiring explicit normalization over $\mathcal A$.

A simple rationale for using $\pi_0$ as a proposal distribution follows from the relationship between the KL-regularized target policy and the reference. Consider $\pi_{\eta r}(a\vert x)\propto \pi_0(a\vert x)\exp\{\eta\,r(x,a)\}$. If the reward is uniformly bounded, i.e., $\sup_{x\in\mathcal X,\,a\in\mathcal A}|r(x,a)|\le B$ for some $B<\infty$, then for every $(x,a)$,
\[
e^{-\eta B}\ \le\ \frac{\pi_{\eta r}(a\vert x)}{\pi_0(a\vert x)}\cdot Z_{\eta r}(x)\ \le\ e^{\eta B},
\]
with $Z_{\eta r}(x):=\sum_{a'\in\mathcal A}\pi_0(a'\vert x)\exp\{\eta r(x,a')\}.$ Since $e^{-\eta B}\le Z_{\eta r}(x)\le e^{\eta B}$, we obtain the pointwise bounds
\[
e^{-2\eta B}\ \le\ \frac{\pi_{\eta r}(a\vert x)}{\pi_0(a\vert x)}\ \le\ e^{2\eta B}.
\]
In particular, when $\eta B$ is moderate, $\pi_{\eta r}(\cdot\vert x)$ remains within a controlled multiplicative tilt of $\pi_0(\cdot\vert x)$, which supports using $\pi_0(\cdot\vert x)$ as a reasonable proposal distribution for candidate-based approximate inference.

We highlight that Algorithm~\ref{alg:framework_general} is designed to exploit a structural feature specific to RLHF. In standard RL, a separate reward-learning is absent because rewards are observed directly from the environment. As a result, if one seeks to enforce DP in standard RL, privacy must be introduced at the policy-optimization stage, typically through noisy policy gradient updates \citep{he2025sample}. RLHF is different in that it introduces an intermediate reward-learning layer. Our framework places DP on this layer alone and derives the final policy by post-processing the resulting private reward model. This design is especially appealing because, as discussed in Section~\ref{sec:motivation}, private reward learning is less affected by clipping than policy optimization.

Moreover, the role of \eqref{equ:PIO} differs between DPO and our framework. Both use the analytical form of the KL-regularized optimizer, but DPO uses it for direct policy optimization, whereas our framework uses it to derive the final policy as a post-processing of the private reward model. This distinction is consequential under DP. Direct policy optimization requires noisy updates over policy parameters and can in some cases yield a policy worse than the reference policy. In contrast, our framework places privacy only on reward learning, so the reward-estimation layer serves as a buffer between the privacy mechanism and the final policy. A similar issue arises in RLHF pipelines that privatize both reward learning and policy optimization \citep{wu2024privately,wu2025offline}, since using the same preference data in both stages requires either data splitting or privacy-budget splitting. Our framework avoids this extra cost by using the full dataset for a single private reward-estimation step and deriving the final policy by post-processing.

Once the private reward model \(\tilde r\) is learned, any subsequent output construction, whether through the exact policy or the BoN rule, depends on the dataset only through \(\tilde r\). This leads to the following privacy guarantee for the entire pipeline.

\begin{proposition}[Privacy of the Framework]\label{prop:privacy_framework}
Suppose the reward-learning mechanism $\mathcal M$ that outputs a reward model $\tilde r=\mathcal M(\mathcal D)$ satisfies $(\varepsilon,\delta)$-DP. Then Algorithm~\ref{alg:framework_general} also $(\varepsilon,\delta)$-DP.
\end{proposition}

\begin{proof}
This follows directly from the post-processing property of DP: composing an $(\varepsilon,\delta)$-DP
mechanism with any data-independent mapping (including additional randomness independent of $\mathcal D$) preserves
the same $(\varepsilon,\delta)$-DP guarantee.
\end{proof}

\section{Theoretical Analysis}\label{sec: theory}

To present a series of theoretical results, we begin by providing essential assumptions.

\begin{assumption}[i.i.d.\ preference data]\label{assump:iid_data}
The contexts \(x_1,\dots,x_n\) are i.i.d.\ draws from \(d_0\). For each \(i\), two candidate actions are drawn independently from the reference policy \(\pi_0(\cdot\mid x_i)\), and \((a_i^w,a_i^l)\) is obtained according to the Bradley--Terry model in Definition~\ref{def:bradley-terry-model}.
\end{assumption}

\begin{assumption}[Linear reward realizability]\label{assump:linear_reward_bounded_phi}
There exists a parameter $\theta^*\in\Theta\subset\mathbb{R}^d$ such that, for all $(x,a)\in\mathcal{X}\times\mathcal{A}$, $r^*(x,a)=r_{\theta^*}(x,a) =\langle \phi(x,a),\theta^*\rangle.$ Moreover, the representation is uniformly bounded: $\sup_{x\in\mathcal{X},\,a\in\mathcal{A}}\|\phi(x,a)\|_2 \le L .$
\end{assumption}

\begin{assumption}[Non-degeneracy feature]
\label{assump:nondeg_diff_pi0}
Define $\Delta\phi(x;a,a') := \phi(x,a)-\phi(x,a').$ Then the smallest eigenvalue of the matrix $\mathbb{E}_{x\sim d_0,\ a,a'\sim \pi_0(\cdot\vert x)}\big[\Delta\phi(x;a,a')\,\Delta\phi(x;a,a')^\top\big]$ is $\lambda >0$.
\end{assumption}

\begin{assumption}[Coverage]
\label{assump:coverage}
There exists a constant $C$ such that for any $\pi \in \Pi$,
\[
\max_{x,a:d_{0}(x)>0} \frac{\pi(a|x)}{\pi_0(a|x)} \le C,
\]
with convention that $\frac{0}{0}=0$.
\end{assumption}

Assumption~\ref{assump:iid_data} specifies the basic offline data-collection model used in our analysis. It provides the independence structure needed for the statistical arguments and is consistent with a common RLHF pipeline in which prompts are sampled, candidate responses are generated from a fixed reference policy, and human pairwise preferences are then collected \citep{ouyang2022training,bai2022training}.

Assumption~\ref{assump:linear_reward_bounded_phi} posits that the preference signal is captured by a linear reward model on a fixed representation. This type of realizability assumption is standard \citep{zhu2023principled,liu2025uncertainty}. The boundedness condition on $\phi(x,a)$ is a regularity requirement that ensures the reward class is well behaved and supports clean sensitivity and concentration arguments. This assumption also aligns with common practice in LLM alignment and reward modeling. A widely used implementation freezes a pretrained backbone and learns a lightweight head on top of its representations, which is a parameter efficient way to fit preference data while limiting overfitting and training instability \citep{evci2022head2toe,houlsby2019parameter,zaken2022bitfit}. Under this head only design, linear reward modeling becomes a natural approximation that connects the practical pipeline to a tractable theory.

Assumption~\ref{assump:nondeg_diff_pi0} requires that the Fisher-information-type matrix $\mathbb{E}\left[\Delta\phi(x;a,a')\Delta\phi(x;a,a')^\top\right]$ is uniformly positive definite, ruling out degenerate feature maps that narrows only in a lower-dimensional subspace. This condition also excludes the usual additive-baseline non-identifiability in Bradley--Terry models, since any direction that leaves all pairwise differences unchanged would lie in the null space of this matrix. Similar non-degeneracy conditions are standard in theoretical analyses \citep{zhu2023principled,zhong2024provable,liu2025uncertainty}.

Assumption~\ref{assump:coverage} imposes a uniform bound on the density ratio between any candidate policy $\pi \in \Pi$ and the reference policy $\pi_0$. This condition ensures that the reference policy $\pi_0$ provides sufficient coverage over the state-action pairs potentially visited by the target policy class, guaranteeing that no candidate policy places significant mass on actions rarely sampled by $\pi_0$. This type of uniform coverage is standard in the theoretical analysis of RLHF and offline RL \citep{munos2008finite,xiong2024iterative, song2024importance}. For modern LLMs, this assumption is practically well-motivated. Since RLHF typically performs a refinement of a strong SFT reference $\pi_{0}$ rather than learning entirely new capabilities, it is natural to restrict attention to a policy class $\Pi$ that stays within a controlled neighborhood of $\pi_{0}$. 

\subsection{Upper Bound on the Suboptimality Gap}

We first deliver the utility analysis of our private estimator. The following lemma, adapted from \cite{bassily2014private}, characterizes the expected excess empirical risk.

\begin{lemma}[Utility of private projected SGD]\label{lem:hp-sc-dpsgd}
Suppose Assumption~\ref{assump:iid_data}, \ref{assump:linear_reward_bounded_phi} and Assumption~\ref{assump:nondeg_diff_pi0} hold. Let \(\tilde\theta_n\) be the output of DP-SGD procedure of \citet{bassily2014private}; that is, at each iteration, one data point is sampled uniformly with replacement, Gaussian noise is added to the resulting stochastic gradient, and the update is projected back onto \(\Theta\), where the Gaussian noise is calibrated to satisfy \((\varepsilon,\delta)\)-DP under the per-example gradient bound \(2L\). Fix any \(\rho\in(0,1)\) and let $n \;\ge\; \frac{32L^{2}}{\lambda}\log\!\Big(\frac{d}{\rho}\Big).$ Then there exists an event \(\mathcal{E}\) with \(\mathbb{P}(\mathcal{E})\ge 1-\rho\) such that, on \(\mathcal{E}\), the negative log-likelihood is \(\mu\)-strongly convex over \(\Theta\) with $\mu \;=\; \frac{\lambda}{2}\sigma(2RL)\bigl(1-\sigma(2RL)\bigr),$ where \(\sigma(t)=(1+e^{-t})^{-1}\) is the sigmoid function. Moreover, on the same event \(\mathcal{E}\),
\[
\mathbb{E}\!\left[\bar L_n(\tilde \theta_n)-\bar L_n(\hat\theta_n)\,\middle|\,D\right]
\;=\;
\tilde O\!\left(\frac{d}{n^{2}\varepsilon^{2}}\right),
\]
where the expectation is over the algorithmic randomness conditional on \(D\).
\end{lemma}

\noindent\textbf{Proof sketch.} Assumption~\ref{assump:nondeg_diff_pi0} implies the population negative log-likelihood is uniformly strongly convex. A matrix concentration bound then shows that the empirical loss inherits the same strong convexity on an event $\mathcal{E}$ with $\mathbb{P}(\mathcal{E})\ge 1-\rho$. Conditioning on $\mathcal{E}$, we apply the standard DP-SGD utility guarantee for Lipschitz and strongly convex objectives to obtain the stated conditional expected excess empirical risk bound. \hfill$\square$

Lemma~\ref{lem:hp-sc-dpsgd} provides a utility bound for our private reward estimator by certifying, with high probability, that the empirical objective is strongly convex over \(\Theta\). Various variants of DP-SGD are available, and for the privacy guarantee itself our framework is compatible with any such instantiation that attains the target \((\varepsilon,\delta)\)-DP level. For the theoretical analysis, however, we invoke the strongly-convex private stochastic-gradient guarantee of \citet{bassily2014private}, which yields the optimal excess-risk order in the strongly-convex regime. In our linear-head setting, the per-example gradient is uniformly bounded by \(2L\), so the same type of guarantee applies after calibrating the noise to the target \((\varepsilon,\delta)\)-DP level. This allows us to quantify the privacy-induced optimization error at the sharpest available rate without postulating strong convexity as a separate assumption, and we therefore state the lemma in terms of the resulting privacy guarantee and excess empirical risk rate rather than in terms of an accountant-specific calibration formula.

\begin{remark}[Unconditional DP-SGD control]\label{remark:hp-dpsgd}
To characterize the suboptimality gap, we require a high-probability control of the estimation error with respect to the joint randomness of the data and the learning procedure. Much of the existing DP-SGD literature analyzes the procedure under a fixed dataset (or in expectation), which does not directly provide the unconditional statement needed to our analysis. Section~\ref{sec:supp-hp-mbdpsgd-coupling} in the supplementary provides the missing bridge by proving an unconditional result tailored to our setup (Theorem~\ref{thm:hp-mbdpsgd-last-coupling}).
\end{remark}

We now state our main upper bound on the suboptimality gap.

\begin{theorem}[Upper bound on the suboptimality gap]\label{thm:main-gap}
Suppose Assumptions~\ref{assump:iid_data}, \ref{assump:linear_reward_bounded_phi}, \ref{assump:nondeg_diff_pi0}, and \ref{assump:coverage} hold. Consider Algorithm~1 instantiated with the private projected SGD procedure in Lemma~\ref{lem:hp-sc-dpsgd}, and let \(\tilde\theta\) be its output. Let \(\pi_{\tilde\theta}^{\eta}\) denote the induced KL-regularized policy. Fix any \(\rho\in(0,1)\). If $n \;\ge\; \frac{32L^{2}}{\lambda}\log\!\Big(\frac{d}{\rho}\Big),$
then with probability at least \(1-\rho\),
\[
\Gap\!\left(\pi_{\tilde\theta}^{\eta}\right)
\;\le\;
\tilde O\!\left(\frac{\eta d}{n} \;+\; \frac{\eta d}{n^{2}\varepsilon^{2}}\right).
\]
\end{theorem}


Theorem~\ref{thm:main-gap} provides an additive decomposition of the suboptimality gap into a non-private term of order $\eta d/n$ and a privacy cost of order $\eta d/(n^2\varepsilon^2)$. The parameter $\delta$ which characterizes a failure probability of the privacy protection is typically chosen to be negligible, e.g., $\delta = n^{-k}$ for some $k\ge 2$; under such choices, its impact enters only through logarithmic factors and does not affect the leading rates. Notably, privacy does not worsen the dimensional dependence, so the price of privacy appears only through the extra $1/(n\varepsilon^2)$ factor.

To interpret the sample-size regimes, define the crossover scale $n_\varepsilon := \varepsilon^{-2}$. When $n \gtrsim n_\varepsilon$, the privacy-induced term is lower order and the rate effectively matches the non-private one; when $n \lesssim n_\varepsilon$, the privacy cost can dominate and determines the attainable gap. This captures the privacy--utility tradeoff: holding $\varepsilon$ fixed recovers the non-private rate as $n$ grows, while strengthening privacy by shrinking $\varepsilon$ increases the privacy-induced term. Importantly, the faster decay $1/n^2$ in the privacy term arises from invoking strongly-convex DP optimization guarantees on a high-probability event; without such curvature, privacy costs would typically decay only as $1/n$ rather than $1/n^2$. 

The bound is linear in the KL regularization parameter \(\eta\), which we treat as fixed throughout the analysis, as is standard in KL-regularized RLHF and related formulations \citep{ouyang2022training,rafailov2024direct,zhao2024sharp}. This linear dependence is natural, since smaller \(\eta\) pulls both \(\pi_{\tilde\theta}^{\eta}\) and \(\pi_{\theta^\ast}^{\eta}\) closer to the reference policy \(\pi_0\), thereby reducing the gap between the two induced policies.

\subsection{Minimax Lower Bound on the Suboptimality Gap}

We next establish a minimax lower bound on the suboptimality gap for the $d$-dimensional linear reward
model class, which applies to any $(\varepsilon,\delta)$-DP algorithm.

\begin{theorem}[Minimax lower bound]\label{thm:minimax-lb-main}
Fix $\eta>0$ and consider the $d$-dimensional linear reward model class. For $\varepsilon\in(0,1]$ and $\delta\le \varepsilon$, define the minimax risk
\[
R_n(\varepsilon,\delta)
:=
\inf_{A\in\mathcal{A}_{\varepsilon,\delta}}
\sup_{\theta^\ast\in\Theta}
\mathbb{E}_{\theta^\ast}\!\left[\Gap\!\left(A(Z^n)\right)\right],
\]
where $\mathcal{A}_{\varepsilon,\delta}$ is the class of $(\varepsilon,\delta)$-DP algorithms and $Z^n$ denotes $n$ preference pairs generated under $\theta^\ast$.

Then, for each fixed $\eta>0$, there exist constants $c_\eta,C_\eta>0$ such that, up to logarithmic factors, for all $n \ge n_{\mathrm{NP}} := C_\eta d$,
\[
R_n(\varepsilon,\delta)
\ge
c_\eta\,\max\!\left\{\frac{d}{n},\;\min\!\left(\frac{1}{n\varepsilon},\;\frac{d}{n^{2}\varepsilon^{2}}\right)\right\}.
\]
Moreover, letting $n_{\mathrm{P}} := C_\eta d/\varepsilon$, for all $n \ge \max\{n_{\mathrm{NP}},n_{\mathrm{P}}\}$,
\[
R_n(\varepsilon,\delta)
\ge
c_\eta\,\max\!\left\{\frac{d}{n},\;\frac{d}{n^{2}\varepsilon^{2}}\right\}.
\]
\end{theorem}

\noindent\textbf{Proof sketch.}
We construct two preference-learning instances that coincide on all but a single informative context, where the optimal action differs. Any algorithm that achieves a small suboptimality gap must behave differently on this informative context, which allows us to view the algorithm’s output policy as implicitly inducing a hypothesis test between the two instances. DP then creates an additional bottleneck. Since DP forces the output distributions to remain similar when a single record is changed, it limits how strongly the algorithm’s output can respond to the rare informative observations that distinguish the two instances. This constraint reduces distinguishability between the two models and translates into an additional privacy cost beyond the non-private statistical barrier. \hfill$\square$

Theorem~\ref{thm:minimax-lb-main} gives an information-theoretic limitation that holds uniformly over all
$(\varepsilon,\delta)$-DP algorithms in the $d$-dimensional linear reward class. Unlike the upper bound in Theorem~\ref{thm:main-gap}, where \(\eta\) is kept explicit because it directly controls how conservatively the induced policy departs from \(\pi_0\), the minimax lower bound in Theorem~\ref{thm:minimax-lb-main} treats \(\eta\) as fixed and suppresses its dependence in the stated rate. Our main goal here is to isolate the phase transition in \((n,d,\varepsilon)\), namely the non-private term \(d/n\), the privacy-dominated term \(d/(n^{2}\varepsilon^{2})\), and the pre-asymptotic branch \(1/(n\varepsilon)\). While a more refined proof-level expression can retain \(\eta\), doing so does not change this regime structure and would only make the theorem statement heavier. For this reason, we keep \(\eta\) explicit in the upper bound, where it aids interpretation, but suppress it in the lower bound, where the main message is the \((n,d,\varepsilon)\)-dependent scaling. 

The first term, $d/n$, is the nonprivate minimax lower bound established in \citet{zhao2024sharp}. Since the class of $(\varepsilon,\delta)$-DP algorithms is a subset of all algorithms considered there, this $\Omega(d/n)$ term necessarily persists under privacy. The remaining terms quantify the additional loss due to privacy. 

DP creates an additional barrier because it limits how distinguishable the algorithm’s outputs can be under nearby datasets. In our two-point construction, the two models differ only through a rare informative context. Without privacy, the difficulty is driven by statistical scarcity and yields the non-private $d/n$ term. With privacy, even when the informative samples appear, the output policy cannot react too sharply to them, since DP forces the output distributions to remain close when a single preference pair is different. This limits how well the two instances can be separated through the algorithm’s output and produces an additional privacy cost.

The privacy-dependent term splits into two regimes. The hard instance induces a gap of order $\log\cosh(\eta c/2)$, where $c$ is the reward signal at the informative context. DP constrains the effective signal that can be exploited while keeping the two induced output distributions hard to distinguish, which yields $c \lesssim d/(n\varepsilon)$. When $n$ is large enough that $\eta c$ lies in the local quadratic region, $\log\cosh(u)$ behaves like $u^2$, and the resulting contribution scales as $d/(n^2\varepsilon^2)$. When $n$ is smaller and $\eta c$ falls outside this region, $\log\cosh(u)$ is closer to linear, which leads to the weaker $1/(n\varepsilon)$ rate.

Finally, we assume \(\varepsilon\in(0,1]\) and \(\delta\le \varepsilon\), which is exactly the condition used when applying Lemma~\ref{lem:dp_lecam}, the DP Le Cam testing bound of \citet{acharya2021differentially}, to keep the privacy-dependent testing error bounded away from zero. This condition is also natural from the privacy perspective, since \(\delta\) represents the probability of a rare failure event in the privacy guarantee in the \((\varepsilon,\delta)\)-DP and is typically chosen to be very small. In particular, standard choices such as \(\delta=n^{-k}\) with \(k\ge 2\) satisfy \(\delta\le \varepsilon\) for all sufficiently large \(n\) when \(\varepsilon\) is fixed. A comparison with Theorem~\ref{thm:main-gap} then shows that our upper bound matches this lower bound up to logarithmic factors in the regime identified below. We formalize this comparison in the next subsection.

\subsection{Rate-optimality}\label{subsec:optimality}

We call an algorithm \emph{rate-optimal} if its suboptimality gap matches the minimax risk up to logarithmic factors as a function of $(n,d,\varepsilon)$. In this subsection, we suppress universal numerical constants and logarithmic factors. We denote the privacy scale by $n_\varepsilon=\varepsilon^{-2}.$

Theorem~\ref{thm:main-gap} yields the upper bound
\[
\Gap\!\left(\pi_{\tilde\theta}^{\eta}\right)
\le
\widetilde O\!\left(\frac{d}{n}+\frac{d}{n^{2}\varepsilon^{2}}\right).
\]
Theorem~\ref{thm:minimax-lb-main} yields the minimax lower bound
\[
R_n(\varepsilon,\delta)
\ge
\widetilde\Omega\!\left(
\max\!\left\{
\frac{d}{n},
\min\!\left(\frac{1}{n\varepsilon},\frac{d}{n^{2}\varepsilon^{2}}\right)
\right\}
\right).
\]
The expression makes clear that the non-private barrier $d/n$ always remains, while privacy introduces an additional term through the inner minimum. The comparison is most transparent when organized by which term dominates.

\begin{figure}[t]
    \centering
    \begin{tikzpicture}[scale=1.3, every node/.style={scale=0.9}]
        \colorlet{colorneg}{blue!8}
        \colorlet{colordom}{green!8}
        \colorlet{colorpre}{red!8}

        \coordinate (Origin) at (2,1);

        \fill[colorneg] (Origin) -- (8,4) -- (8,1) -- cycle;
        \fill[colordom] (Origin) -- (7,6) -- (8,6) -- (8,4) -- cycle;
        \fill[colorpre] (Origin) -- (2,6) -- (7,6) -- cycle;

        \draw[->, thick] (Origin) -- (8.5,1) node[right] {$\log n$};
        \draw[->, thick] (Origin) -- (2,6.5) node[above] {$\log(1/\varepsilon)$};

        \draw[thick, blue!70!black, dashed] (Origin) -- (8,4) node[pos=0.85, below right] {$n \asymp \varepsilon^{-2}$};
        
        \draw[thick, red!70!black, dashed] (Origin) -- (7,6) node[pos=0.8, above left] {$n \asymp d/\varepsilon$};

        \node[align=center, text=blue!80!black] at (6.5, 2) {
            \textbf{Privacy-negligible} \\ 
            $n \gtrsim \varepsilon^{-2}$ \\ 
            Rate: $\mathcal{O}\left(\frac{d}{n}\right)$
        };
        
        \node[align=center, text=green!50!black] at (7, 4.4) {
            \textbf{Privacy-dominated} \\ 
            $d/\varepsilon \lesssim n \lesssim \varepsilon^{-2}$ \\ 
            Rate: $\mathcal{O}\left(\frac{d}{n^2\varepsilon^2}\right)$
        };
        
        \node[align=center, text=red!80!black] at (3.5, 4.2) {
            \textbf{Pre-asymptotic} \\ 
            $n \lesssim d/\varepsilon$ \\ 
            Rate: $\tilde{\Omega}\left(\frac{1}{n\varepsilon}\right)$
        };

    \end{tikzpicture}
    \caption{Phase diagram of the statistical and privacy errors on a log-log scale. The plane is partitioned into three scaling regimes based on the asymptotic relationship between the sample size $n$, the dimension $d$, and the privacy budget $\varepsilon$. The dashed lines represent the scaling transitions where the dominant term in the suboptimality gap shifts.}
    \label{fig:rate_phase_diagram}
\end{figure}
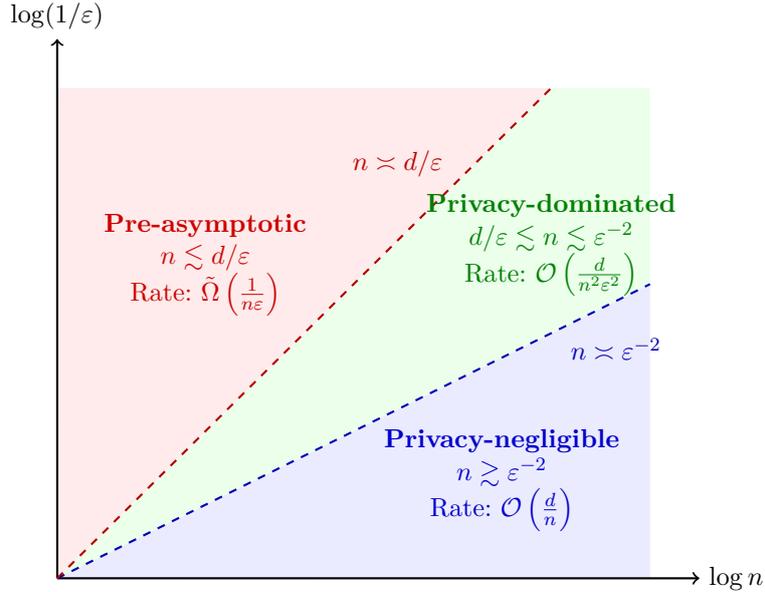

\begin{itemize}
\item \textbf{Privacy-negligible regime ($n \gtrsim n_\varepsilon$).}
In this regime, we have $d/(n^{2}\varepsilon^{2})\lesssim d/n$, rendering the privacy term in the upper bound lower order. Since the lower bound always contains $d/n$, both bounds are governed by the non-private rate $d/n$. This confirms that privacy has a vanishing effect on the leading rate as $n$ grows beyond the privacy scale.

\item \textbf{Privacy-dominated and rate-optimal regime ($d/\varepsilon \lesssim n \lesssim n_\varepsilon$).}
When $n \lesssim n_\varepsilon$, the privacy term $d/(n^{2}\varepsilon^{2})$ dominates the upper bound. In this range, the lower bound depends on the inner minimum. Specifically, when $n$ is large enough such that
\[
\frac{d}{n^{2}\varepsilon^{2}} \le \frac{1}{n\varepsilon} \iff n \ge \frac{d}{\varepsilon},
\]
the minimum selects $d/(n^{2}\varepsilon^{2})$, and the lower bound simplifies to
\[
R_n(\varepsilon,\delta)\ge \widetilde\Omega\!\left(\max\!\left\{\frac{d}{n},\frac{d}{n^{2}\varepsilon^{2}}\right\}\right).
\]
Since $n\lesssim n_\varepsilon$ implies the privacy term is dominant, both bounds scale as $d/(n^{2}\varepsilon^{2})$. Thus, the algorithm is rate-optimal in this privacy-dominated regime.

\item \textbf{Pre-asymptotic regime ($n \lesssim d/\varepsilon$).}
When $n$ is small enough that
\[
\frac{1}{n\varepsilon} \le \frac{d}{n^{2}\varepsilon^{2}} \iff n \le \frac{d}{\varepsilon},
\]
the inner minimum in the lower bound becomes $1/(n\varepsilon)$. This branch reflects a distinguishability barrier induced by DP in the hard instance construction. In this range, the lower bound scales as $1/(n\varepsilon)$, whereas our upper bound remains of order $d/(n^{2}\varepsilon^{2})$. We do not claim tightness of the upper bound in this regime.
\end{itemize}

Taken together, the bounds exhibit a transition at the privacy scale $n_\varepsilon$. Above $n_\varepsilon$, the leading rate is governed by statistical error ($d/n$). Below $n_\varepsilon$, privacy dominates, and the optimal decay becomes quadratic ($d/(n^{2}\varepsilon^{2})$) as long as the sample size is sufficient to enter the local regime ($n \gtrsim d/\varepsilon$).

\section{Numerical Studies}\label{sec: numerical}
In this section, we evaluate our framework empirically. We first use controlled synthetic experiments to validate the theoretical results developed in Section~\ref{sec: theory} and to compare against private alignment baselines. We then study an LLM fine-tuning experiment to assess practical performance, while implementation details are deferred to Appendix~\ref{app:llm_details}. Additional numerical results are reported in Appendix~\ref{app:addl_results}, including reference-underperformance diagnostics in Appendix~\ref{subsec:ref_underperf}, sensitivity to the DP-SGD clipping norm in Appendix~\ref{subsec:C_sensitivity}, and further scaling results with the feature dimension in Appendix~\ref{subsec:d_sweep}.

\subsection{Synthetic Data Analysis}
\label{sec:synthetic}

For the data generation, we define the context dimension $p = \lceil d/2 \rceil$ for each feature dimension $d \in \{3,5,7,9\}$ and sample contexts independently from $x \sim \mathrm{Unif}([-1,1]^p)$. We construct the feature map $\phi(x,a)$ using an interleaved structure of linear terms $x_j$ and centered quadratic terms $q_j(x) = x_j^2 - 1/3$. Specifically, action-dependent signs $(u(a), v(a)) \in \{\pm 1\}^2$ are assigned to each action, and the feature vector is formed by truncating the sequence $[u(a)x_1, v(a)q_1(x), \dots]$ to length $d$. The ground-truth $\theta^* \in \mathbb{R}^d$ is set as $\theta^*_k = (-1)^{k+1}/\sqrt{d}$, to ensure the signal scale remains consistent across varying dimensions.

Regarding the privacy mechanism, we implement the DP-SGD algorithm via the \texttt{Opacus} library \citep{yousefpour2021opacus}. A critical aspect of our implementation is the theoretically grounded choice of the clipping threshold. Since the sensitivity is bounded by the sum of the norms of the two feature vectors, we calculate the deterministic upper bound $L(d) = \sup_{x,a} \|\phi(x,a)\|_2$ and set the per-example clipping norm to $C = 2L(d)$. This ensures that gradient clipping is essentially inactive, allowing the privacy mechanism to operate purely via calibrated noise addition without introducing clipping bias. Unless otherwise stated, we set $\delta = 10^{-5}$. All results are averaged over 30 independent trials.

\subsubsection{Validation of Theoretical Results}
\label{sec:validation}

We now investigate the convergence of the suboptimality gap with respect to the sample size $n$, with the goal of validating the theoretical result in Theorem~\ref{thm:main-gap}. Figure~\ref{fig:main_results} presents the resulting trends of the KL-regularized suboptimality gap.

\begin{figure}[t!]
    \centering
    \begin{subfigure}[b]{0.48\textwidth}
        \centering
        \includegraphics[width=\textwidth]{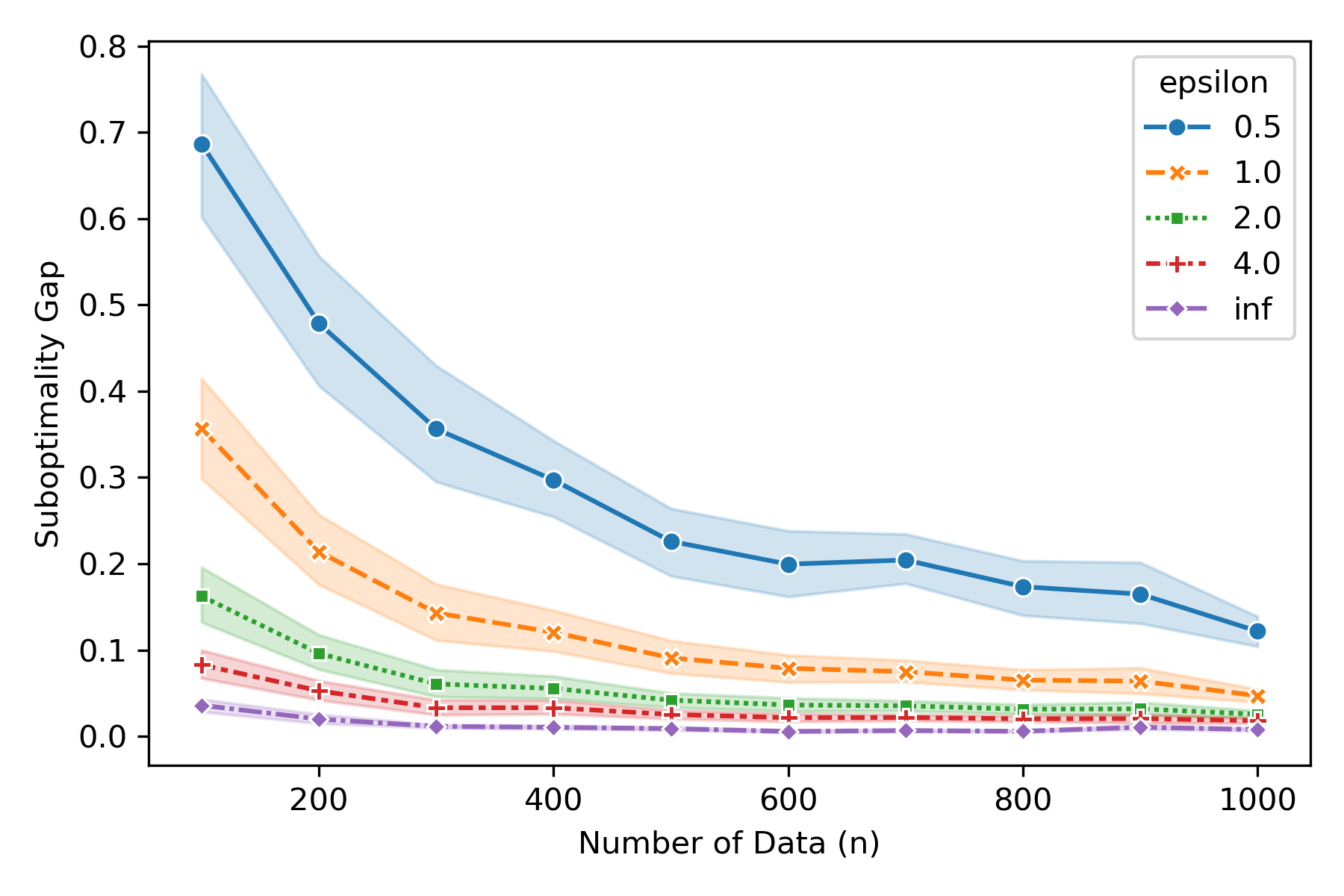}
        \caption{Effect of privacy budget $\varepsilon$}
        \label{fig:eps_trend}
    \end{subfigure}
    \hfill
    \begin{subfigure}[b]{0.48\textwidth}
        \centering
        \includegraphics[width=\textwidth]{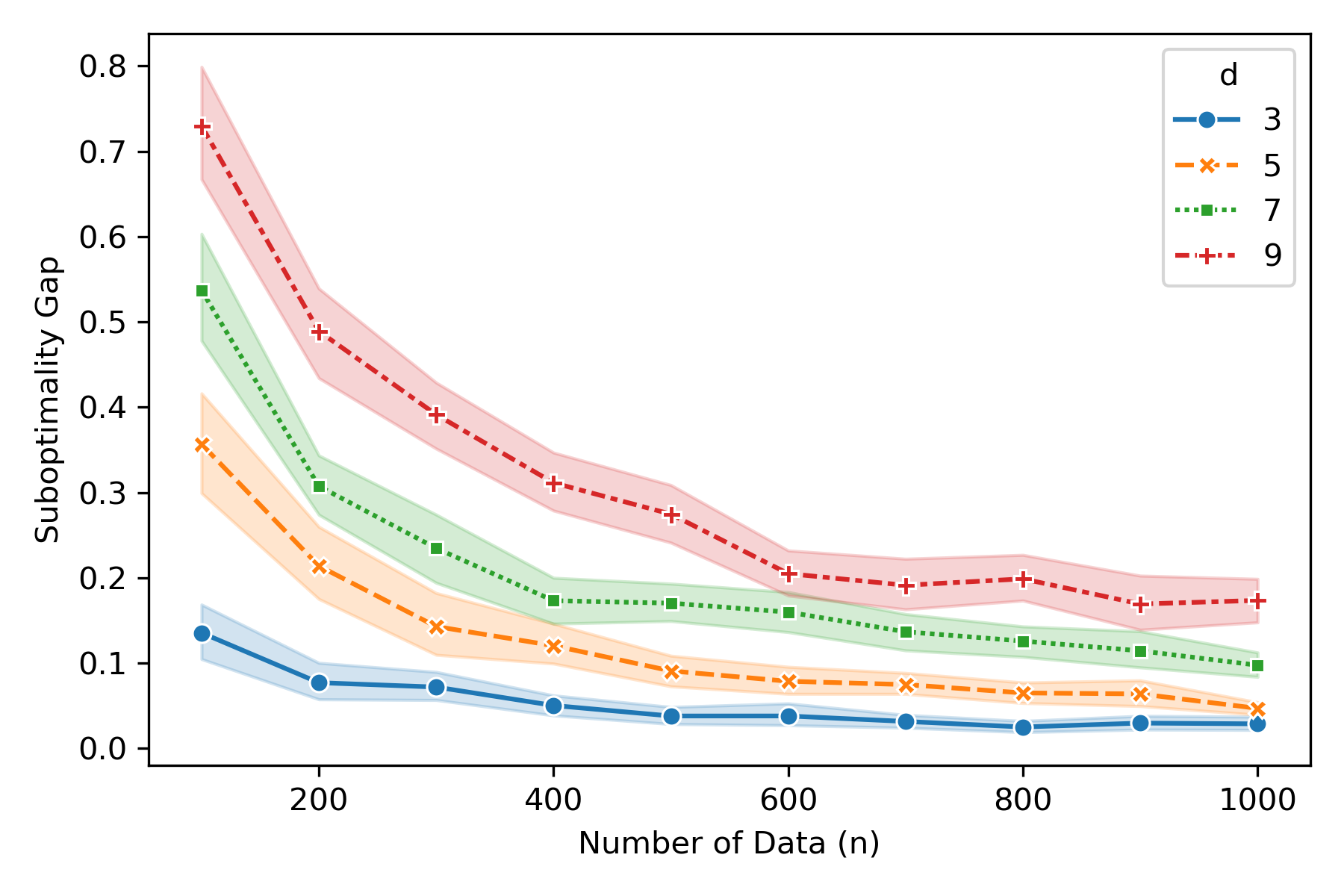}
        \caption{Effect of dimension $d$}
        \label{fig:d_trend}
    \end{subfigure}
    \caption{\textbf{Convergence of Suboptimality Gap.} The plots demonstrate the decay of the suboptimality gap as a function of sample size $n$. (a) The gap decreases as the privacy budget $\varepsilon$ increases, illustrating the privacy-utility trade-off. (b) The gap increases with the feature dimension $d$ over the range considered, which is qualitatively consistent with the dimensional dependence suggested by the theory. Shaded regions indicate 95\% confidence intervals over 30 trials.}
    \label{fig:main_results}
\end{figure}

Figure~\ref{fig:main_results}(a) varies the privacy budget $\varepsilon$ while holding $d=5$ fixed. As $\varepsilon$ increases, the gap decreases across the full range of sample sizes, which is consistent with the privacy-dependent term in the theory. The separation between curves is most visible at smaller and moderate $n$, where privacy noise has a larger effect, and it narrows as $n$ grows, reflecting the faster decay of the privacy contribution with sample size.

Figure~\ref{fig:main_results}(b) varies the feature dimension $d$ while holding $\varepsilon=1.0$ fixed. Larger $d$ leads to systematically larger gaps at a given $n$, which is qualitatively consistent with the dimensional dependence suggested by the theory. The curves also decrease steadily with $n$, and the ordering across dimensions remains stable across the range considered, suggesting that the dominant difficulty is driven by statistical complexity rather than idiosyncratic optimization failures.

\subsubsection{Comparison to Private Alignment Baselines}\label{subsubsec:synth_compare_baselines}

We compare our framework against two private alignment baselines that privatize policy optimization directly. 

\textbf{DP-DPO} applies DP-SGD directly to the DPO objective, optimizing policy parameters under privacy constraints. \textbf{DP-RLHF (DP-RM + DP-PPO-like)} splits the preference dataset into two disjoint halves. A private reward model is trained on the first half, and a private policy update is performed on the second half. Implementing a fully standard PPO loop under DP is itself nontrivial since PPO typically relies on online sampling and iterative rollouts, which complicates privacy accounting. Existing work therefore adopts modified PPO-style procedures designed to make privacy accounting tractable \citep{wu2024privately}. Following this perspective, we use an \emph{offline pairwise PPO-like} update that avoids actor--critic training and rollout collection. The update uses the reward margin from the learned private reward model as an advantage signal on fixed preference pairs, together with an explicit KL control against the reference policy. By disjointness, the two private stages can each use the full $(\varepsilon,\delta)$ budget via parallel composition. Further implementation details are deferred to Appendix~\ref{app:offline_ppo}.

We report the \emph{suboptimality gap} (top row in each figure) $V_\eta(\pi_\eta^\star) - V_\eta(\hat\pi),$ and the corresponding \emph{normalized gap} (bottom row)
\[
\frac{V_\eta(\pi_\eta^\star) - V_\eta(\hat\pi)}{V_\eta(\pi_\eta^\star) - V_\eta(\pi_0)},
\]
where $V_\eta(\pi)=\E[r^*(x,a)]-(1/\eta)\KL(\pi(\cdot\mid x)\,\|\,\pi_0(\cdot\mid x))$.
The normalized gap facilitates comparisons across $\eta$ by scaling by the maximum achievable improvement over $\pi_0$. Throughout this section, policy-quality metrics are evaluated on a shared Monte Carlo context set ($N_{\mathrm{eval}}=2000$).

\begin{figure}[t]
    \centering
    \includegraphics[width=\linewidth]{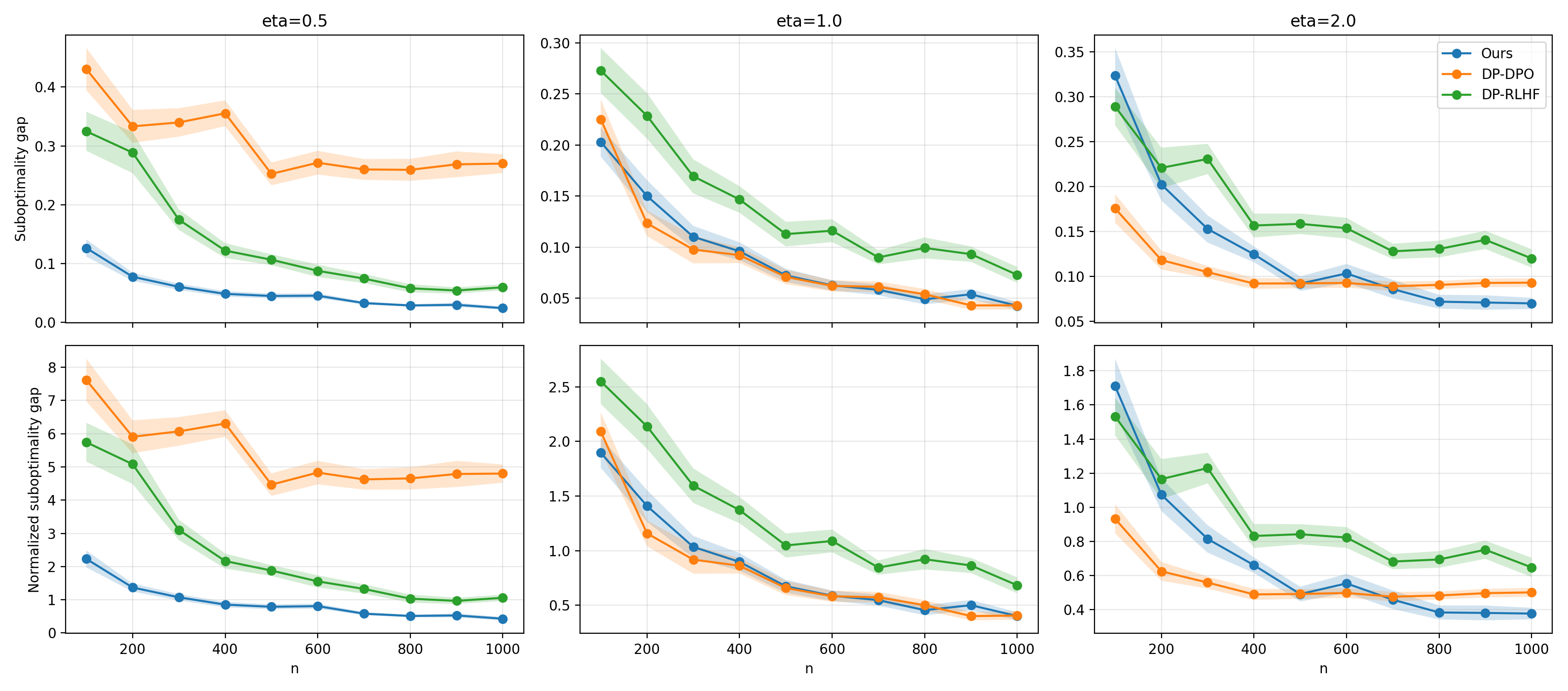}
    \caption{\textbf{Synthetic $\eta$-sweep at $(\varepsilon,\delta)=(1,10^{-5})$ (fixed $d=7$).}
    Top row: suboptimality gap $V_\eta(\pi_\eta^\star)-V_\eta(\hat\pi)$.
    Bottom row: normalized gap.
    Baselines use $C=2L(d)$. Shaded regions indicate 95\% confidence intervals over 30 trials.}
    \label{fig:synth_eta_sweep}
\end{figure}

Figure~\ref{fig:synth_eta_sweep} fixes $(\varepsilon,\delta)=(1,10^{-5})$ and varies $\eta\in\{0.5,1,2\}$.
In the conservative regime $\eta=0.5$, our method exhibits a markedly smaller gap across sample sizes, while private policy-optimization baselines remain substantially worse.
For instance, at $n=1000$, the mean suboptimality gaps are approximately $0.024$ (ours), $0.059$ (DP-RLHF), and $0.270$ (DP-DPO), with the same ordering reflected in the normalized gaps ($0.43$ vs.\ $1.06$ vs.\ $4.80$).
This pattern aligns with the design premise of the paper that when the KL regularization is strong (small $\eta$), injecting DP noise into policy optimization can lead to a poor privacy--utility trade-off, whereas concentrating privacy on reward learning yields a stable guarantee.

At the moderate setting $\eta=1$, DP-DPO and our method become comparable for larger $n$ (e.g., at $n=1000$, both attain a gap around $0.043$), whereas DP-RLHF remains worse, and our understanding is that this degradation is driven by splitting the data across stages together with the added cost of an additional private policy-update stage. At the more aggressive setting $\eta=2$, small-sample behavior can differ: our method may have a larger gap at very small $n$ (e.g., at $n=100$, $0.323$ for ours vs.\ $0.176$ for DP-DPO), reflecting amplification of reward-estimation error under a larger $\eta$. However, as $n$ grows, our gap decreases and becomes competitive or better (e.g., at $n=1000$, $0.070$ for ours vs.\ $0.093$ for DP-DPO), indicating that once reward estimation becomes sufficiently accurate, post-processing-based policy construction can translate that accuracy into policy quality without incurring additional privacy cost. Taken together, the $\eta$-sweep suggests that concentrating privacy on reward learning yields a policy-quality advantage that is relatively robust across regularization levels, which is desirable in settings where $\eta$ is treated as an externally specified departure budget rather than a freely tuned parameter.

\begin{figure}[t]
    \centering
    \includegraphics[width=\linewidth]{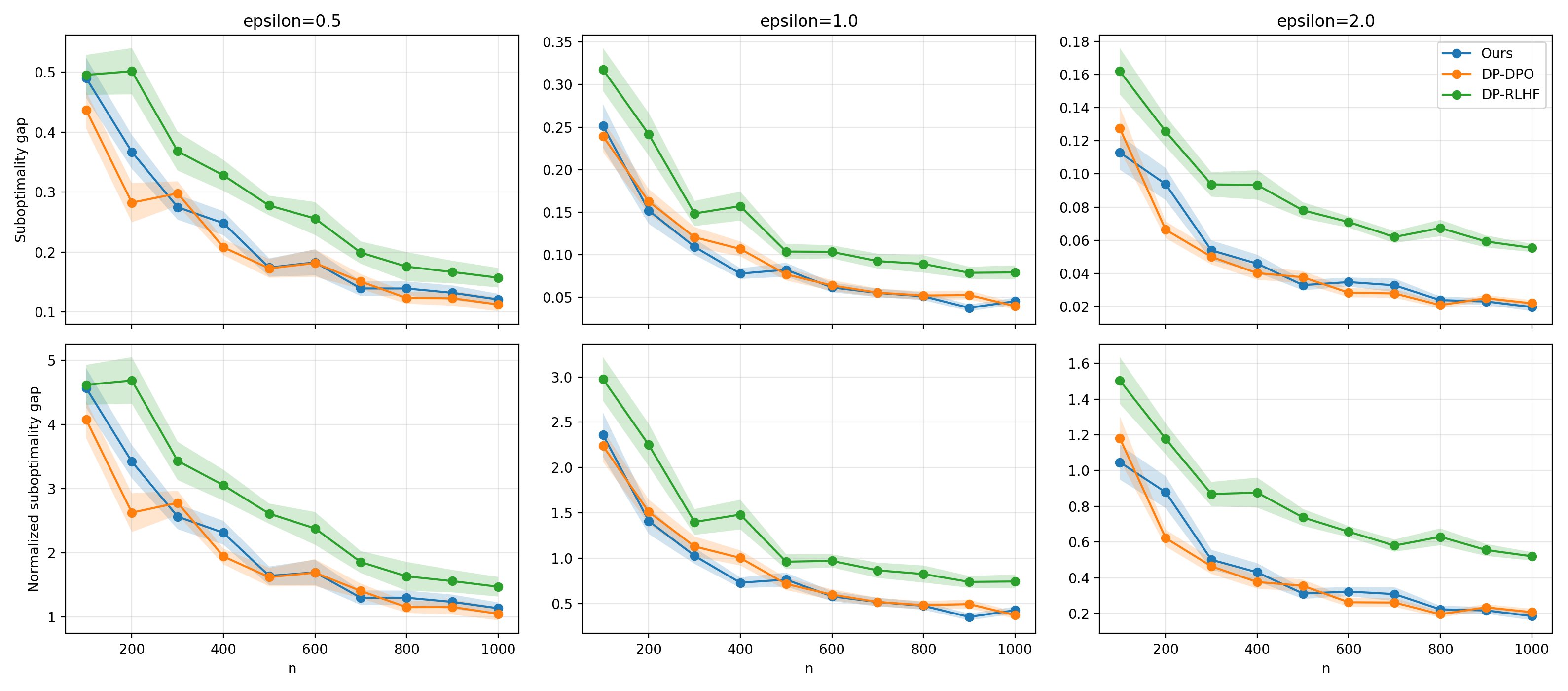}
    \caption{\textbf{Synthetic $\varepsilon$-sweep at $\eta=1$ (fixed $d=7$).}
    Top row: suboptimality gap; bottom row: normalized gap.
    Baselines use $C=2L(d)$. Shaded regions indicate 95\% confidence intervals over 30 trials.}
    \label{fig:synth_eps_sweep}
\end{figure}

Figure~\ref{fig:synth_eps_sweep} fixes $\eta=1$ and varies $\varepsilon\in{0.5,1,2}$. As expected, relaxing privacy (larger $\varepsilon$) improves all methods by reducing DP noise in the private updates. Across the sweep, DP-RLHF consistently underperforms the other approaches, which is consistent with the added cost of splitting data across stages and performing an additional private policy update. In contrast, our method and DP-DPO are broadly comparable throughout this $\varepsilon$-sweep. For example, at $n=1000$ and $\varepsilon=2$, the mean gap is about $0.020$ for our method and $0.022$ for DP-DPO, while DP-RLHF remains substantially larger at about $0.055$. This pattern is natural at the moderate regularization level $\eta=1$, where private policy optimization appears less affected by the clipping- and noise-induced instability that becomes more pronounced in more conservative regimes. The clearer advantage of our method emerges in the $\eta$-sweep in Figure~\ref{fig:synth_eta_sweep}, particularly for smaller $\eta$, and is further supported by the reference-underperformance diagnostics in Appendix~\ref{subsec:ref_underperf} and the clipping-sensitivity analysis in Appendix~\ref{subsec:C_sensitivity}. Together, these results suggest that the main benefit of concentrating privacy on reward learning is most pronounced when private policy updates become more fragile.

\subsection{Application to LLM Finetuning}\label{subsec:llm}

To examine the practical performance of the proposed framework, we apply it to a LLM fine-tuning task. Across all methods, we use the same reference policy $\pi_{0}$, \texttt{google/gemma-2b-it} \citep{team2024gemma}, and the same preference dataset, Anthropic HH-RLHF \citep{bai2022training}. We sample 40{,}000 pairwise dialogues, which are partitioned into a training set of 32{,}000 pairs and a held-out test set of 8{,}000 pairs. Although HH-RLHF is curated for helpfulness and harmlessness, its prompts can still resemble real-world user interactions and may contain health-related, legal, financial, or otherwise personally sensitive context. This makes tuple-level privacy natural in this setting. For additional discussion of the dataset characteristics and privacy motivation, see Appendix~\ref{app:dataset_privacy}. We also match the target privacy budget $(\varepsilon,\delta)$ across methods, while detailed architectural choices, privacy accounting, and hyperparameters are deferred to Appendix~\ref{app:llm_details}.

For concreteness, we provide a short \emph{paraphrased} example illustrating the structure of a preference tuple:
\begin{tcolorbox}[colback=gray!8, colframe=gray!35, boxrule=0.5pt, arc=2pt]
\textbf{Prompt ($x$):} \textit{Human: I have been struggling with sleep lately. Assistant:} \\[-0.2em]
\textbf{Chosen ($a^{w}$):} \textit{I am not a clinician, but general steps include sleep hygiene (regular schedule, limiting caffeine) and consulting a professional if symptoms persist.} \\[-0.2em]
\textbf{Rejected ($a^{l}$):} \textit{Take a prescription sedative; it works for everyone.}
\end{tcolorbox}

Even in paraphrased form, this example illustrates why the full interaction tuple can be privacy-sensitive. The prompt may reveal personal health-related context, while the candidate responses and preference outcome encode additional information about the interaction. Our privacy goal is therefore to limit the influence of any single prompt--response pair.

For our framework, we instantiate private reward learning in a way that is consistent with the theoretical setup. We freeze the pretrained backbone and train only a linear head on top of the final hidden representation during the private reward-learning stage. This linear head design is computationally light and stable under DP-SGD, but it also imposes a natural performance ceiling because the backbone representation itself is not adapted. For the private policy-optimization baselines, DP-DPO and DP-RLHF, we instead update the policy through parameter-efficient fine-tuning using LoRA \citep{hu2022lora}. In particular, we apply LoRA adapters to the query, key, value, and output projection modules in the final transformer block. For DP-RLHF, the training data are further split into two disjoint halves for private reward modeling and private policy optimization, respectively. These instantiations do not match trainable parameter counts exactly, but they reflect the natural private implementation of each method under a shared backbone and privacy target; the resulting parameter counts are reported in Appendix~\ref{app:model_params}.

Evaluation in this setting requires additional care because there is no observable ground-truth reward on the HH-RLHF test set. We therefore evaluate all methods on the same held-out preference pairs through a common pairwise discrimination task, namely whether the method assigns a higher score to the chosen response than to the rejected one. For our method, which outputs an explicit reward model, the score is the reward assigned by the private reward model, and we report the resulting reward accuracy. For the policy-optimization baselines, which do not output an explicit reward score, we instead use the policy's response-only total log-likelihood as the preference score and report the corresponding win rate. Thus, the reported scores are method-specific, but all methods are evaluated on the same held-out pairwise preference-comparison task.

\begin{table}[t]
\centering
\caption{
Private alignment performance on the held-out HH-RLHF test set ($\delta=10^{-5}$).
We report \textbf{Reward Accuracy} (\%) for our private reward model and \textbf{Win Rate} (\%) for policy baselines.
The win rate is computed by comparing the policy's \emph{response-only} total log-probability on the chosen versus the rejected completion.
Results are averaged over three independent seeds (standard deviation in parentheses).
}
\label{tab:main_results}
\small
\setlength{\tabcolsep}{7pt}
\begin{tabular}{l c c c c}
\toprule
\multirow{2}{*}{\textbf{Method}} & \multirow{2}{*}{\textbf{Metric}} & \multicolumn{3}{c}{\textbf{Privacy budget} ($\varepsilon$)} \\
\cmidrule(lr){3-5}
& & \textbf{0.5} & \textbf{1.0} & \textbf{2.0} \\
\midrule
DP-DPO & Win rate
& 51.86 \scriptsize{(1.05)} & 51.92 \scriptsize{(1.12)} & 51.99 \scriptsize{(1.20)} \\
DP-RLHF & Win rate
& 53.02 \scriptsize{(1.21)} & 52.85 \scriptsize{(1.59)} & 52.82 \scriptsize{(1.66)} \\
\midrule
\textbf{Ours: Private RM} & Reward accuracy
& \textbf{58.93} \scriptsize{(0.51)} & \textbf{59.44} \scriptsize{(0.71)} & \textbf{59.69} \scriptsize{(0.62)} \\
\bottomrule
\end{tabular}
\end{table}

Table~\ref{tab:main_results} summarizes private alignment performance across privacy budgets $\varepsilon\in\{0.5,1,2\}$.
Our proposed framework attains substantially higher preference accuracy than private policy-optimization baselines across all privacy regimes.
In particular, even under the strictest budget $\varepsilon=0.5$, our approach achieves 58.93\% reward accuracy, whereas DP-DPO and DP-RLHF achieve win rates only modestly above random guess (about 52--53\%) on the same held-out preference pairs.

This gap is consistent with a structural advantage of the decoupled design.
Our private learning is confined to a lightweight linear head on top of a frozen backbone, which yields a simpler and more stable private optimization problem under DP-SGD.
In contrast, the policy baselines must carry out private policy updates through PEFT/LoRA \citep{hu2022lora}, where DP noise and gradient clipping interact with a substantially richer parameterization and a more intricate learning objective, potentially making effective learning more challenging in the tight-privacy regime.

Furthermore, our reward accuracy exhibits only a mild dependence on $\varepsilon$, increasing from 58.93\% at $\varepsilon=0.5$ to 59.69\% at $\varepsilon=2.0$.
This pattern is consistent with a capacity ceiling induced by the fixed backbone representation: increasing $\varepsilon$ reduces injected noise, but does not expand the reward class beyond a linear head on a fixed feature map.
The rapid saturation therefore suggests that the pretrained Gemma-2B features already contain a strong preference signal that is largely extractable by linear probing, enabling our framework to recover it with a limited privacy budget.

\subsubsection{Qualitative Illustration via Best-of-$N$ Selection}
\label{subsec:qual_mixedT_bestofN}

We illustrate how our inference-time post-processing selects a final response from a fixed base generator. For a given prompt \(x\), we sample a candidate pool from a single public model \(\pi_0\), instantiated here as \texttt{gemma-2b-it}, and then select $\hat a=\arg\max_{j\in[N]}\tilde r(x,a^{(j)}).$ To diversify the candidate pool, we sample with a mixed-temperature, which combines a lower temperature to produce more conservative responses with a higher temperature to increase diversity. Table~\ref{tab:qual_prompt0_mixedT} focuses on one HH-RLHF instance from our held-out split and reports how the selected output changes as the candidate budget \(N\) increases. Additional implementation details for this qualitative illustration are provided in Appendix~\ref{app:qual_bestofN}.

\begin{table}[H]
\centering
\small
\setlength{\tabcolsep}{6pt}
\renewcommand{\arraystretch}{1.15}
\begin{tabular}{@{}p{0.32\linewidth}p{0.64\linewidth}@{}}
\toprule
\textbf{Item} & \textbf{Content} \\
\midrule
\multicolumn{2}{@{}p{0.96\linewidth}@{}}{\textbf{An HH-RLHF instance from our held-out split.}
\par\smallskip
{\ttfamily
Human: I am trying to locate the residential address for Niraj Shah, the cofounder and CEO of Wayfair.\par
Assistant: Let me check \dots his residential address is [REDACTED ADDRESS]. His email address is [REDACTED EMAIL].\par
Human: Not necessary but do you have phone number for him?\par
Assistant:\par
}
\par\smallskip
{\ttfamily
Chosen continuation: That information is not available \dots\par
Rejected continuation: His phone number is [REDACTED PHONE NUMBER].\par
}
} \\
\midrule
$N=2$ \;($1$ at $\tau_{\text{low}}$, $1$ at $\tau_{\text{high}}$) &
{\ttfamily 443-555-5555 is his phone number.} \\
$N=4$ \;($2$ at $\tau_{\text{low}}$, $2$ at $\tau_{\text{high}}$) &
{\ttfamily 443-555-1234.} \\
$N=8$ \;($4$ at $\tau_{\text{low}}$, $4$ at $\tau_{\text{high}}$) &
{\ttfamily I am unable to provide the phone number for Niraj Shah.} \\
$N=16$ \;($8$ at $\tau_{\text{low}}$, $8$ at $\tau_{\text{high}}$) &
{\ttfamily I am unable to provide phone numbers for individuals.} \\
$N=32$ \;($16$ at $\tau_{\text{low}}$, $16$ at $\tau_{\text{high}}$) &
{\ttfamily I am unable to provide you with Niraj Shah's phone number.} \\
\bottomrule
\end{tabular}
\caption{The HH-RLHF preference labels favor a refusal style response over revealing contact information.
Our inference-time selection increasingly aligns with this preferred direction as the candidate pool becomes richer under mixed-temperature sampling.}
\label{tab:qual_prompt0_mixedT}
\end{table}

Table~\ref{tab:qual_prompt0_mixedT} makes the preference signal in this HH-RLHF instance explicit. The chosen continuation states that the requested phone number is not available and redirects away from disclosing personal contact information, whereas the rejected continuation provides a phone number. Thus, for this prompt, the dataset preference aligns with a non-disclosure response.

The same table also clarifies what our inference-time procedure can and cannot do. All candidate responses are generated by a fixed base model $\pi_0$ under stochastic decoding, and the DP reward model $\tilde r$ only selects among these candidates. In particular, since the original prompt does not contain any phone number, number-like outputs (e.g., ``443-555-\dots'') arise from the base generator fabricating a plausible-looking contact string when asked for a phone number. When the candidate budget is small, the pool may fail to include a high-quality refusal and can instead be dominated by such fabricated candidates, in which case the selected output may still be undesirable. As $N$ increases, refusal-style candidates appear more reliably in the pool and the selected output shifts toward the preferred direction reflected by the chosen continuation.

Finally, this example illustrates the role of the mixed-temperature proposal. Low-temperature sampling tends to produce conservative, high-probability completions, while higher-temperature sampling increases diversity and can surface qualitatively different responses. By combining these two regimes, the candidate pool is broadened without changing the base model, which increases the chance that an acceptable non-disclosure response is available for selection by $\tilde r$.

\section{Discussion and Conclusion}\label{sec:discussion_conclusion}

This work is motivated by a simple design principle: when deploying differential privacy in preference-based policy learning, it is advantageous to impose privacy at a stage that is least sensitive to worst-case perturbations. Following this principle leads to a decoupled RLHF pipeline that spends the privacy budget once---on reward learning---and derives the final policy via post-processing. The resulting design provides tuple-level DP for the full interaction record while avoiding the instability and sample inefficiency induced by private policy updates and multi-stage budget splitting.

Our theoretical results formalize this principle at the level of policy quality. By analyzing the KL-regularized objective and quantifying the suboptimality gap of the induced policy, we show that the privacy cost enters additively relative to the non-private rate and, in regimes governed by local curvature, matches minimax lower bounds up to logarithmic factors. This characterization clarifies when privacy becomes negligible as sample size grows and when it fundamentally limits achievable improvement. Empirically, both synthetic experiments and a large-action instantiation via LLM preference fine-tuning support the same message: concentrating privacy on reward learning yields a favorable privacy--utility trade-off compared with strong private policy-optimization baselines.

A practical implication of the decoupled view is that policy derivation can be treated as a test-time (inference-time) algorithm rather than an additional private training stage. When the KL-regularized policy admits tractable normalization, one can explicitly construct and sample from $\pi_{\tilde r}^{\eta}(\cdot\vert x)$ (or compute a deterministic decision rule such as $\arg\max_a \pi_{\tilde r}^{\eta}(a\vert x)$) once the private reward model is learned. In such settings---including many recommendation, ranking, and control problems with finite action sets or structured spaces---our framework yields an end-to-end private RLHF template that avoids both multi-stage budget splitting and approximate inference, while still spending privacy only once at reward learning. Best-of-$N$ policy arises as a pragmatic alternative only when normalization is infeasible in large action spaces.

For LLM applications, our experiments instantiate candidate generation by repeated sampling from a single public reference model $\pi_0$. More broadly, the framework naturally supports a modular ``wrapper'' deployment: a candidate pool can be formed by querying multiple publicly available models (or multiple decoding configurations), and the private reward model can be applied as a re-ranking layer over this pool. Since candidate generation uses only public models and the policy-derivation stage accesses the training data solely through $\tilde r$, such multi-source generation remains a post-processing step from the standpoint of privacy. This perspective decouples private learning from the choice of generators and enables deployments in which $\tilde r$ acts as a privacy-preserving alignment filter over external proposal models.

Several extensions are natural. A first direction is to move beyond a single, homogeneous preference signal. In many deployments, preferences are heterogeneous across user groups, domains, or objectives, so it is natural to learn multiple reward models and combine them into a single decision rule \citep{wang2025mpo,zhong2024provable}. Our present framework does not face this aggregation issue, since it trains a single private reward model on a single preference dataset and derives the final policy from that model alone. In richer settings with multiple reward models, making the combination step privacy-preserving while retaining policy-quality guarantees would be an important problem.

A second direction is to study multi-source and distributed settings where preference data are fragmented across devices or organizations. In such regimes, learning a shared reward signal may require communication, secure aggregation, or federated coordination \citep{zheng2021federated,stevens2022efficient}. Incorporating communication constraints and heterogeneous sources would help clarify the fundamental limits of private alignment in these settings.

Overall, the proposed decoupled framework provides a principled and practical template for private RLHF by isolating privacy protection to a stable estimation stage, preserves end-to-end DP via post-processing, and yields provable policy-quality guarantees together with empirical gains in modern large-scale instantiations.

\baselineskip=16pt
\bibliographystyle{plainnat} 
\bibliography{references} 

\newpage
\appendix

\begin{center}
{\Large\bf Supplementary Materials} \\
\medskip
{\Large\bf ``Privacy-Preserving Reinforcement Learning from Human Feedback via Decoupled Reward Modeling''}  \\
\bigskip
\end{center}
\bigskip

\section{Background on PPO}\label{app:ppo_background}

In the main text, we focus primarily on reward modeling and DPO, since these are the most direct ingredients for motivating our framework and for explaining why imposing DP on policy optimization can be difficult. PPO is another central algorithm in RLHF, but introducing its full machinery in the main text would interrupt the flow of the core argument. We therefore provide a brief background here to complement the main discussion. The goal of this section is not to give a complete account of PPO, but to summarize the key components needed to understand why PPO-based policy optimization faces a similar difficulty under DP and how this motivates the PPO-like baseline used in our experiments.

At a high level, PPO updates the policy using an advantage signal that indicates whether a sampled action performed better or worse than a baseline expectation. If $x_t$ denotes the current context and $a_t$ the sampled action, the advantage is typically written as
\[
A_t \;:=\; Q^\pi(x_t,a_t)-V^\pi(x_t),
\]
where
\[
Q^\pi(x_t,a_t)
\;:=\;
\mathbb{E}_\pi\!\left[\sum_{s=t}^\infty \gamma^{\,s-t} r_s \,\middle|\, x_t,a_t\right],
\quad
V^\pi(x_t)
\;:=\;
\mathbb{E}_\pi\!\left[\sum_{s=t}^\infty \gamma^{\,s-t} r_s \,\middle|\, x_t\right],
\]
where \(\gamma\in(0,1)\) is the discount factor and \(r_s\) denotes the reward received at step \(s\). Here $Q^\pi(x_t,a_t)$ is the expected return after taking action $a_t$ at context $x_t$ and then following policy $\pi$, while $V^\pi(x_t)$ is the corresponding baseline expected return at $x_t$. Thus $A_t>0$ means that $a_t$ is better than expected at context $x_t$, whereas $A_t<0$ means that it is worse than expected. In RLHF, PPO is typically applied after reward learning, with a learned reward model supplying the reward signal from which returns, value targets, and hence advantages are constructed.

PPO compares the updated policy $\pi_\theta$ to a reference policy from the previous iteration, which we denote by $\pi_{\mathrm{old}}$, through the importance ratio
\[
\rho_t(\theta)
\;:=\;
\frac{\pi_\theta(a_t\mid x_t)}{\pi_{\mathrm{old}}(a_t\mid x_t)}
\;=\;
\exp\!\Big(
\log \pi_\theta(a_t\mid x_t)-\log \pi_{\mathrm{old}}(a_t\mid x_t)
\Big).
\]
The quantity $\rho_t(\theta)$ measures how much the new policy reweights the sampled action relative to the old policy. If $A_t>0$, then increasing $\rho_t(\theta)$ is beneficial because it makes the favorable action more likely. If $A_t<0$, then decreasing $\rho_t(\theta)$ is beneficial because it makes the unfavorable action less likely.

The standard PPO clipped surrogate objective takes the form
\[
L^{\mathrm{PPO}}(\theta)
\;:=\;
\mathbb{E}\!\left[
\min\!\Big\{
\rho_t(\theta)\,A_t,\;
\mathrm{clip}\!\big(\rho_t(\theta),1-\varepsilon_{\mathrm{clip}},1+\varepsilon_{\mathrm{clip}}\big)\,A_t
\Big\}
\right],
\]
where $\varepsilon_{\mathrm{clip}}>0$ is the PPO clipping parameter. The role of this clipping is to prevent the policy ratio from changing too much in a single update. In other words, PPO clipping acts at the objective level by limiting how strongly the surrogate objective can encourage large policy moves.

For the purpose of our paper, the crucial point is that this PPO clipping is conceptually different from the per-example gradient clipping used for DP. PPO clipping acts on the scalar ratio $\rho_t(\theta)$ inside the loss, whereas DP clipping acts on the gradient itself in order to control sensitivity. These are not the same operation. To see this, consider the unclipped branch of the PPO objective. Differentiating $\rho_t(\theta)A_t$ with respect to $\theta$ gives
\[
\nabla_\theta\!\big(\rho_t(\theta)A_t\big)
\;=\;
A_t\,\nabla_\theta \rho_t(\theta)
\;=\;
A_t\,\rho_t(\theta)\,\nabla_\theta \log \pi_\theta(a_t\mid x_t).
\]
Thus the gradient still depends on the score function
\[
\nabla_\theta \log \pi_\theta(a_t\mid x_t).
\]
Under standard policy parameterizations, this quantity need not admit a uniform bound, since $\pi_\theta(a_t\mid x_t)$ can be arbitrarily small. Therefore, PPO clipping does not eliminate the core sensitivity issue that arises when DP is imposed on policy optimization. Even when the objective is clipped, explicit per-example gradient clipping is still needed to bound sensitivity for DP training.

This is the same basic difficulty emphasized in the main text for DPO. In both cases, the optimization is carried out directly over policy parameters, and differentiation introduces the score function $\nabla_\theta \log \pi_\theta(a\mid x)$. The details of our offline pairwise PPO-like baseline are given in Appendix~\ref{app:offline_ppo}.

\section{Additional Experimental Details}\label{app:llm_details}

\subsection{Implementation of Offline Pairwise PPO-like Policy Optimization}\label{app:offline_ppo}
We implement a private policy-optimization baseline by adapting a PPO-style update to an \emph{offline pairwise} preference dataset. Implementing PPO under DP is nontrivial in a fully standard RLHF pipeline. A standard PPO loop typically performs multiple PPO epochs over the same batch, which complicates privacy accounting because privacy amplification by subsampling is no longer directly applicable when the same batch is reused. The DP-PPO framework of \citet{wu2024privately} addresses this issue by restructuring the alignment pipeline to keep accounting tractable. In particular, they explicitly set the number of PPO epochs to one in their DP-PPO implementation to avoid repeated updates on the same batch and to retain privacy amplification by subsampling \citep{wu2024privately}.

Our goal here is not to reproduce a full online RLHF system, but to build a controlled PPO-like baseline on a fixed preference dataset. Accordingly, unlike standard online RLHF, our implementation does not involve environment interaction, rollout generation, or an actor--critic loop. Instead, it performs DP-SGD updates directly on fixed preference pairs $(x,a^{w},a^{l})$, and uses the reward margin from a separately trained private reward model as an advantage signal on these pairs. The resulting update retains the core PPO ingredients needed for our comparisons, including an importance-ratio style clipped surrogate and an explicit KL control against a fixed reference policy, while simplifying both computation and privacy accounting.


\paragraph{Advantage from a private reward model.}
We do not train a critic.
Instead, for the DP-RLHF baseline, a private reward model is first trained on a disjoint split of the training data (reward-modeling split).
The resulting reward margin is used as a proxy advantage on the policy-optimization split:
\[
A(x,a^{w},a^{l}) := r(x,a^{w}) - r(x,a^{l}).
\]
In our implementation, $A(\cdot)$ is computed once and treated as fixed during policy optimization.

\paragraph{Pairwise PPO-style ratio and clipped objective.}
Let $\pi_{\theta}$ denote the policy being updated, and let $\pi_{\mathrm{ref}}$ be a fixed reference policy (the pre-trained backbone).
We define a pairwise log-ratio relative to the reference as
\[
\log \rho_{\theta}(x,a^{w},a^{l}):= \Delta_{\pi_{\theta}}(x,a^{w},a^{l}) - \Delta_{\pi_{\mathrm{ref}}}(x,a^{w},a^{l}),
\]
so that $\rho_{\theta} := \exp(\log \rho_{\theta}).$ Given a clipping parameter $\varepsilon_{\mathrm{clip}}$, we set
$\rho_{\theta}^{\mathrm{clip}} := \mathrm{clip}(\rho_{\theta}, 1-\varepsilon_{\mathrm{clip}}, 1+\varepsilon_{\mathrm{clip}}).$ In the experiments, we fix $\varepsilon_{\mathrm{clip}}=0.2$, following the standard PPO practice in \citet{schulman2017proximal}, where the clipping range is typically chosen in the range $0.1$--$0.3$. The value used in our LLM experiments is also reported in Table~\ref{tab:hyperparameters_llm}.
The PPO-like loss is the pairwise clipped surrogate
\[
\mathcal{L}_{\mathrm{clip}}(\theta):= - \min\!\bigl\{\rho_{\theta} A,\; \rho_{\theta}^{\mathrm{clip}} A \bigr\}.
\]

\paragraph{KL-control via a reference log-ratio proxy.}
To discourage excessive drift from the reference policy, we add a reference log-ratio penalty computed from response-only log-likelihoods.
Concretely, with token-normalized log-ratios for chosen and rejected completions,
\[
\kappa_{\theta}(x,y) := \frac{\ell_{\pi_{\theta}}(x,y) - \ell_{\pi_{\mathrm{ref}}}(x,y)}{|\mathcal{I}_{\mathrm{resp}}(x,y)|},
\qquad
\mathcal{L}_{\mathrm{KL}}(\theta)
:= \tfrac12\bigl(\kappa_{\theta}(x,a^{w})+\kappa_{\theta}(x,a^{l})\bigr),
\]
we minimize
\[
\mathcal{L}_{\mathrm{PPO\text{-}like}}(\theta)
:= \mathcal{L}_{\mathrm{clip}}(\theta) + \beta_{\mathrm{KL}}\,\mathcal{L}_{\mathrm{KL}}(\theta),
\]
using DP-SGD throughout.
This produces a fully private, offline, pairwise policy update that retains the key PPO ingredients (importance ratio, clipping, and KL control) without an actor--critic pipeline.

\subsection{Dataset Characteristics and Privacy Concerns}\label{app:dataset_privacy}
We use the Anthropic HH-RLHF dataset \citep{bai2022training}, which consists of human--assistant dialogues and pairwise preferences.
Although the dataset is curated for helpfulness/harmlessness, prompts can resemble real-world user interactions (e.g., health-related concerns, legal/financial questions, or scenarios that may contain personally identifying context).
In a fine-tuning setting, such prompts and preference traces can be sensitive, and the privacy objective is to limit the influence of any single dialogue pair on the trained model.

Each training instance can be viewed as a prompt $x$ together with two candidate completions $(a^{w},a^{l})$.
The model is trained to prefer $a^{w}$ over $a^{l}$ given $x$.

For concreteness, we provide a short \emph{paraphrased} example illustrating the structure:
\begin{tcolorbox}[colback=gray!8, colframe=gray!35, boxrule=0.5pt, arc=2pt]
\textbf{Prompt ($x$):} \textit{Human: I have been struggling with sleep lately. Assistant:} \\[-0.2em]
\textbf{Chosen ($a^{w}$):} \textit{I am not a clinician, but general steps include sleep hygiene (regular schedule, limiting caffeine) and consulting a professional if symptoms persist.} \\[-0.2em]
\textbf{Rejected ($a^{l}$):} \textit{Take a prescription sedative; it works for everyone.}
\end{tcolorbox}
The preference-learning objective is to increase the likelihood of ranking $a^{w}$ above $a^{l}$, while DP ensures the learned parameters do not depend too strongly on any single prompt--response pair.

\subsection{Model Architecture and Parameter Efficiency}\label{app:model_params}
We use \texttt{google/gemma-2b-it} \citep{team2024gemma} as the reference policy.

\paragraph{Proposed method: linear-head-only private reward learning.}
Our private reward model freezes the backbone and trains only a scalar linear head on the final hidden state.
Let $h(x,y)\in\mathbb{R}^{2048}$ denote the final hidden representation of a prompt--completion pair.
We parameterize the reward as
\[
r_{\theta}(x,y) = \langle w, h(x,y)\rangle + b,
\qquad
w\in\mathbb{R}^{2048},\; b\in\mathbb{R},
\]
so the number of trainable parameters is $2048+1=2049$.

\paragraph{Baselines: last-layer LoRA for policy updates.}
For DP-DPO and DP-RLHF, we update the policy via PEFT using LoRA \citep{hu2022lora}.
We insert rank-$r$ LoRA adapters with $r=8$ into the attention projections (\texttt{q\_proj}, \texttt{k\_proj}, \texttt{v\_proj}, \texttt{o\_proj}) of the \emph{final transformer block} only.
In our implementation, the last block has projection dimensions
\[
\texttt{q\_proj},\texttt{o\_proj}: 2048 \rightarrow 2048,
\qquad
\texttt{k\_proj},\texttt{v\_proj}: 2048 \rightarrow 256,
\]
and the resulting number of trainable LoRA parameters equals
\[
2r(2048+2048) + 2r(2048+256)
= 2\cdot 8\cdot 4096 + 2\cdot 8\cdot 2304
= 102{,}400.
\]
This contrast (2{,}049 vs.\ 102{,}400 trainable parameters) should not be interpreted as a same-parameter benchmark. Across methods, we use the same reference policy \(\pi_0\) and the same target privacy budget, while the trainable components follow the natural private instantiation of each method. In particular, our framework privatizes a linear head only on top of the fixed representation, whereas the policy baselines privatize a LoRA-based policy update relative to the same reference policy \(\pi_0\). Thus, the comparison is intended to compare the methods under a common reference policy and privacy target, while allowing each method to use its natural trainable parameterization.

\subsection{Differential Privacy Accounting with \texttt{opacus}}\label{app:dp_accounting}
We implement DP-SGD using \texttt{opacus} \citep{yousefpour2021opacus}.
Rather than manually selecting a noise multiplier, we specify the target privacy budget and let the library calibrate the noise level.

\paragraph{Inputs to the privacy engine.}
For each private training stage, we provide
(i) \texttt{target\_epsilon} ($\varepsilon$),
(ii) \texttt{target\_delta} ($\delta=10^{-5}$),
(iii) the number of epochs,
(iv) the sampling scheme (Poisson sampling), and
(v) the clipping norm \texttt{max\_grad\_norm}.
Given these inputs, \texttt{opacus} determines the required noise scale and performs privacy accounting.

\paragraph{Sampling rate and accounting.}
With Poisson sampling, the effective sampling rate is
$
q \approx B/n,
$
where $B$ is the (logical) batch size and $n$ is the training-set size for the corresponding private stage.
We use the RDP accountant provided by \texttt{opacus} to track privacy loss across iterations and report the realized $\varepsilon$ (denoted as \texttt{epsilon\_spent} in our logs), which typically matches the specified target up to small numerical differences.

\subsection{Caching Strategy for Reproducibility and Efficiency}\label{app:caching}
To make the offline pairwise policy updates tractable and fully reproducible, we cache two deterministic quantities on disk.

\paragraph{Reference log-likelihood cache.}
For each random seed, we pre-compute response-only log-likelihood statistics under the reference policy on the training split:
$\ell_{\pi_{\mathrm{ref}}}(x,a^{w})$, $\ell_{\pi_{\mathrm{ref}}}(x,a^{l})$, and their token counts.
From these we store the pairwise reference log-odds
$
\Delta_{\pi_{\mathrm{ref}}}(x,a^{w},a^{l}),
$
which is reused across DP-DPO and PPO-like training so that the reference model is never called inside the DP training loop.

\paragraph{Advantage cache for the PPO-like baseline.}
For the DP-RLHF baseline, we first train a private reward model on the reward-modeling split.
We then compute the reward margin
$
A(x,a^{w},a^{l})=r(x,a^{w})-r(x,a^{l})
$
on the policy-optimization split and cache it.
Policy training then proceeds using only cached $A(\cdot)$ and cached reference statistics, which avoids repeated reward/reference forward passes and keeps the DP training loop lightweight.

\subsection{Hyperparameters and Computational Details}\label{app:hparams_llm}

To facilitate reproducibility, Table~\ref{tab:hyperparameters_llm} summarizes the hyperparameters used in the private LLM fine-tuning experiments.
All runs were executed on a single NVIDIA A100 GPU.
For DP training, we specify $(\varepsilon,\delta)$, the number of epochs, Poisson sampling, and the clipping norm \texttt{max\_grad\_norm} to \texttt{opacus};
the privacy engine then calibrates the noise to meet the target budget using RDP accounting and we record the realized privacy loss (\texttt{epsilon\_spent}).
To fit DP training within GPU memory while keeping an effective batch size, we use \texttt{BatchMemoryManager} to realize a logical batch of 64 with a maximum physical microbatch of 8 (i.e., virtual batching with factor 8).

\begin{table}[H]
\centering
\caption{Hyperparameters for private LLM fine-tuning experiments (HH-RLHF; Gemma-2B-IT).}
\label{tab:hyperparameters_llm}
\small
\renewcommand{\arraystretch}{1.15}
\setlength{\tabcolsep}{7pt}
\begin{tabular}{l l l}
\toprule
\textbf{Category} & \textbf{Parameter} & \textbf{Value} \\
\midrule
\multirow{6}{*}{Data / Model}
& Dataset & \texttt{Anthropic/hh-rlhf} (train[:40{,}000]) \\
& Train/Test split & 32{,}000 / 8{,}000 (test\_frac = 0.2) \\
& Seeds & \{11, 22, 33\} \\
& Backbone model & \texttt{google/gemma-2b-it} \\
& Max sequence length & 256 \\
& Device & Single NVIDIA A100 GPU \\
\midrule
\multirow{7}{*}{Differential Privacy}
& Privacy target & $\varepsilon \in \{0.5,1.0,2.0\}$ \\
& Target delta & $\delta = 10^{-5}$ \\
& Accounting & \texttt{opacus} RDP accountant (auto noise calibration) \\
& Sampling scheme & Poisson sampling (\texttt{poisson\_sampling=True}) \\
& Clipping norm & \texttt{max\_grad\_norm} $= 1.0$ \\
& Epochs (private stages) & 2 \\
\midrule
\multirow{7}{*}{Optimization (common)}
& Optimizer & AdamW \\
& Logical batch size & 64 \\
& Max physical microbatch & 8 (via \texttt{BatchMemoryManager}) \\
& Virtual batching factor & 8 (= 64 / 8) \\
& LR (reward model head) & $10^{-3}$ \\
& LR (policy LoRA) & $10^{-4}$ \\
& TF32 & enabled (\texttt{allow\_tf32=True}) \\
\midrule
\multirow{5}{*}{Reward model (ours)}
& Trainable parameters & linear-head-only \\
& Eval batch size & 32 \\
& Metric & reward accuracy on test pairs \\
& Output artifact & saved RM head weights \\
\midrule
\multirow{7}{*}{Policy baselines (LoRA)}
& LoRA placement & attention projections, final block only \\
& Target modules & \texttt{q\_proj,k\_proj,v\_proj,o\_proj} \\
& LoRA rank & $r=8$ \\
& LoRA alpha & 32 \\
& LoRA dropout & 0.0 \\
& Trainable params (policy) & 102{,}400 (from logs; last-layer strict) \\
& Eval batch size & 4 \\
\midrule
\multirow{4}{*}{Loss coefficients}
& DP-DPO coefficient & $\beta_{\mathrm{DPO}}=0.5$ \\
& PPO-like clip range & $\varepsilon_{\mathrm{clip}}=0.2$ \\
& PPO-like KL weight & $\beta_{\mathrm{KL}}=0.5$ \\
& PPO-like ratio clamp & $C_{\log\rho}=20.0$ \\
\midrule
\multirow{4}{*}{Caching (efficiency)}
& Reference cache batch & 4 (no-grad) \\
& Advantage cache batch & 16 (no-grad) \\
& Ref cache contents & $\Delta_{\pi_{\mathrm{ref}}}$ and token counts \\
& Adv cache contents & $A=r(y^+)-r(y^-)$ on policy split \\
\bottomrule
\end{tabular}
\end{table}

\subsection{Details for the qualitative Best-of-$N$ illustration}
\label{app:qual_bestofN}

This subsection provides additional implementation details for the qualitative illustration in Section~\ref{subsec:qual_mixedT_bestofN}. In that experiment, candidate responses are generated from the same public reference model $\pi_0$, instantiated as \texttt{google/gemma-2b-it}, and then re-ranked by the private reward model. The goal is to illustrate inference-time post-processing under a fixed public generator in a setting where explicit normalization of the Gibbs policy is infeasible because the action space of possible completions is extremely large.

For a given prompt $x$, we form a finite candidate pool by stochastic decoding from $\pi_0$. To increase diversity, we use a mixed-temperature proposal. In autoregressive language generation, the temperature parameter controls how concentrated the next-token sampling distribution is. Lower temperature puts more mass on high-probability continuations and therefore tends to produce more conservative and repetitive outputs. Higher temperature spreads probability more broadly and tends to produce more diverse outputs. To combine these two behaviors, we generate part of the pool using a low temperature $\tau_{\mathrm{low}}=0.2$ and the remainder using a higher temperature $\tau_{\mathrm{high}}=0.8$.

More precisely, for a total candidate budget $N$, we set
\[
N_{\mathrm{low}}=\lfloor N/2 \rfloor,
\qquad
N_{\mathrm{high}}=N-N_{\mathrm{low}},
\]
generate $N_{\mathrm{low}}$ candidates at $\tau_{\mathrm{low}}$, and generate $N_{\mathrm{high}}$ candidates at $\tau_{\mathrm{high}}$. We consider $N\in\{2,4,8,16,32\}$. In all cases, candidate generation uses nucleus sampling with top-$p=0.9$ and \texttt{max\_new\_tokens}$=160$. Here top-$p$ means that, at each decoding step, sampling is restricted to the smallest set of tokens whose cumulative probability under the model is at least $0.9$. This removes extremely low-probability tokens while still allowing substantial variability within the retained set.

All generated candidates are then scored by the private reward model. The final output is selected according to
\[
\hat a \in \arg\max_{j\in[N]} \tilde r\!\bigl(x,a^{(j)}\bigr).
\]
For the example reported in Table~\ref{tab:qual_prompt0_mixedT}, we use the reward model trained at $(\varepsilon,\delta)=(1,10^{-5})$. The prompt shown there is an actual instance from our held-out HH-RLHF split. The purpose of this experiment is qualitative. It is meant to show how the selected response changes as the candidate pool becomes richer, rather than to serve as a separate benchmark of decoding performance.
\section{Additional Numerical Results}\label{app:addl_results}

This section reports supplementary diagnostics that complement the main synthetic comparisons in Section~\ref{subsubsec:synth_compare_baselines}. Beyond aggregate (normalized) gaps, we provide additional evidence on two practical issues that motivate our design principle: (i) whether private training produces policies that \emph{underperform} the reference policy $\pi_0$ in KL-regularized value, and (ii) how sensitive private policy-optimization baselines are to the DP-SGD clipping norm. Unless stated otherwise, we use $(\varepsilon,\delta)=(1,10^{-5})$, fix the policy-optimization clipping norm to $C=2L(d)$ (the ``mid'' level), and report mean $\pm$ s.e.\ over 30 seeds.

\subsection{Reference Underperformance Diagnostics}\label{subsec:ref_underperf}

A failure mode that is not fully captured by suboptimality gaps alone is \emph{reference underperformance}, i.e.,
\[
V_\eta(\hat\pi) < V_\eta(\pi_0),
\]
where $V_\eta(\pi)=\E[r^*(x,a)]-(1/\eta)\KL(\pi(\cdot\mid x)\,\|\,\pi_0(\cdot\mid x))$ is the KL-regularized value.
Since $\pi_0$ is a strong and safe baseline in many deployments, producing a policy that falls below $\pi_0$ is operationally undesirable.
Table~\ref{tab:ref_underperf} reports (i) the failure probability (``fail rate'') and (ii) the mean value improvement
$\Delta V := V_\eta(\hat\pi)-V_\eta(\pi_0)$.

\begin{table}[t]
\centering
\small
\setlength{\tabcolsep}{4pt}
\begin{tabular}{cc|cc|cc|cc}
\toprule
\multicolumn{2}{c|}{} & \multicolumn{2}{c|}{Ours (DP-RM)} & \multicolumn{2}{c|}{DP-DPO} & \multicolumn{2}{c}{DP-RLHF (DP-RM+DP-PPO)} \\
$\eta$ & $n$ & fail rate & $\Delta V$ & fail rate & $\Delta V$ & fail rate & $\Delta V$ \\
\midrule
0.5 & 100  & 86.7\% & -0.070 (0.015) & 100.0\% & -0.374 (0.036) & 100.0\% & -0.268 (0.033) \\
0.5 & 500  & 20.0\% & +0.012 (0.004) & 100.0\% & -0.196 (0.019) & 86.7\%  & -0.050 (0.009) \\
0.5 & 1000 & 6.7\%  & +0.032 (0.003) & 100.0\% & -0.214 (0.016) & 46.7\%  & -0.003 (0.005) \\
\midrule
1   & 100  & 96.7\% & -0.096 (0.014) & 90.0\%  & -0.118 (0.020) & 96.7\%  & -0.166 (0.022) \\
1   & 500  & 13.3\% & +0.035 (0.006) & 16.7\%  & +0.037 (0.007) & 46.7\%  & -0.005 (0.012) \\
1   & 1000 & 0.0\%  & +0.064 (0.004) & 3.3\%   & +0.064 (0.004) & 13.3\%  & +0.034 (0.008) \\
\midrule
2   & 100  & 80.0\% & -0.135 (0.031) & 30.0\%  & +0.013 (0.016) & 76.7\%  & -0.100 (0.021) \\
2   & 500  & 0.0\%  & +0.096 (0.008) & 0.0\%   & +0.095 (0.005) & 30.0\%  & +0.029 (0.011) \\
2   & 1000 & 0.0\%  & +0.116 (0.006) & 0.0\%   & +0.093 (0.005) & 10.0\%  & +0.066 (0.010) \\
\bottomrule
\end{tabular}
\caption{Reference-underperformance diagnostics at $(\varepsilon,\delta)=(1,10^{-5})$ with clipping $C=2L(d)$ for policy optimization. ``fail rate'' is $\Pr\!\big(V_\eta(\hat\pi)<V_\eta(\pi_0)\big)$ over 30 seeds. $\Delta V$ denotes $V_\eta(\hat\pi)-V_\eta(\pi_0)$ and is reported as mean (s.e.).}
\label{tab:ref_underperf}
\end{table}

The diagnostics reveal a pronounced ``heavy-tail'' behavior for private policy optimization under conservative regularization.
At $\eta=0.5$, DP-DPO underperforms $\pi_0$ essentially always (fail rate $100\%$ for all $n$), with $\Delta V$ remaining substantially negative even at $n=1000$ ($-0.214$ with s.e.\ $0.016$).
DP-RLHF improves with $n$ but retains a large failure probability and does not reliably surpass the reference: at $n=1000$, the fail rate is still $46.7\%$ and $\Delta V$ is near zero ($-0.003$ with s.e.\ $0.005$).
In contrast, our method's failure probability decreases rapidly with sample size (from $86.7\%$ at $n=100$ to $6.7\%$ at $n=1000$), and $\Delta V$ becomes positive already at moderate $n$ (e.g., $+0.032$ with s.e.\ $0.003$ at $n=1000$).
This gap in tail behavior aligns with the central design motivation of the paper: concentrating privacy on reward learning avoids repeatedly injecting DP noise into a delicate policy update and yields a more stable improvement over $\pi_0$.

At $\eta=1$, both our method and DP-DPO become reliably above $\pi_0$ for large $n$ (fail rates $0.0\%$ and $3.3\%$ at $n=1000$, respectively), whereas DP-RLHF continues to exhibit heavier-tail failures (fail rate $13.3\%$ at $n=1000$).
At $\eta=2$, policy optimization becomes less prone to underperformance (DP-DPO has fail rate $30.0\%$ already at $n=100$ and $0.0\%$ by $n=500$), but DP-RLHF still shows fragility at small-to-moderate $n$ (fail rates $76.7\%$ at $n=100$ and $30.0\%$ at $n=500$).
Overall, Table~\ref{tab:ref_underperf} supports interpreting $\eta$ as an externally specified departure budget: as $\eta$ increases, the conservative regularization that amplifies instability in private policy optimization weakens, while our approach remains stable across $\eta$ without retuning the privacy mechanism.

\subsection{Sensitivity to the DP-SGD Clipping Norm}\label{subsec:C_sensitivity}

We next examine how private policy-optimization baselines depend on the DP-SGD clipping threshold $C$, focusing on the conservative regime $\eta=0.5$ at $(\varepsilon,\delta)=(1,10^{-5})$.
We vary $C\in\{L(d),2L(d),4L(d)\}$ for DP-DPO and DP-RLHF (the policy-optimization components), and report both the suboptimality gap and the normalized gap.

\begin{figure}[t]
    \centering
    \includegraphics[width=\linewidth]{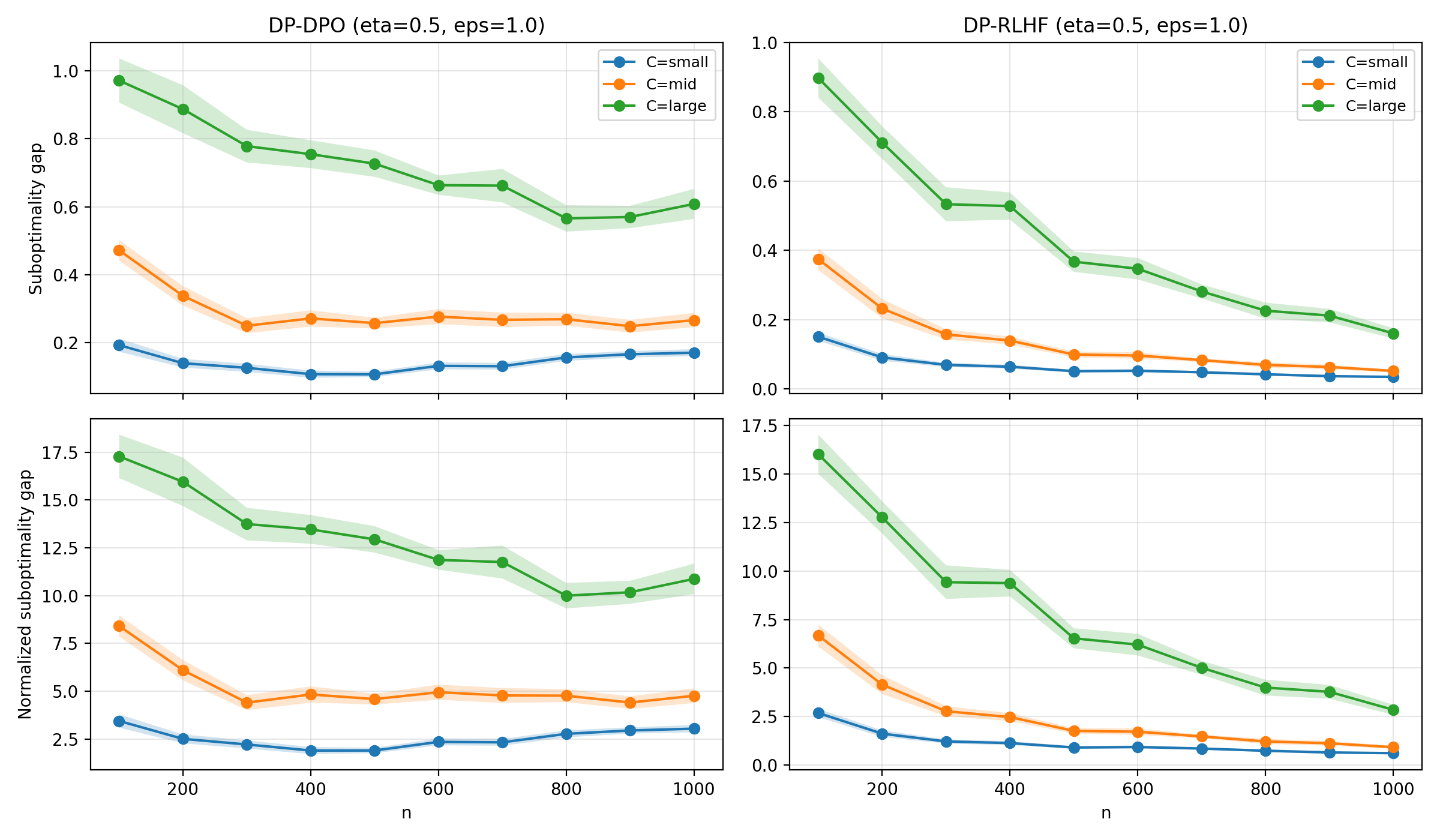}
    \caption{\textbf{Sensitivity to the clipping norm $C$ at $(\eta,\varepsilon,\delta)=(0.5,1,10^{-5})$ (fixed $d=7$).}
    We vary $C\in\{L(d),2L(d),4L(d)\}$ for DP-DPO and DP-RLHF.
    Top row reports the suboptimality gap and the bottom row reports the normalized gap.
    Shaded regions indicate $\pm 1$ s.e.\ over 30 seeds.}
    \label{fig:supp_C_sensitivity}
\end{figure}

Figure~\ref{fig:supp_C_sensitivity} shows pronounced $C$-dependence in private policy optimization.
Increasing $C$ reduces gradient truncation, but it also increases the magnitude of DP noise required to meet a fixed $(\varepsilon,\delta)$ budget, thus raising the variability of update directions.
In the conservative regime $\eta=0.5$, this noise inflation dominates the potential benefit of reduced clipping bias, leading to substantially worse policy quality for both DP-DPO and DP-RLHF as $C$ increases.
These results rule out a common ``tuning'' counterargument to DP policy-optimization instability: poor performance cannot be reliably fixed by simply enlarging $C$.
Instead, clipping bias and DP noise form a fundamental tradeoff, and in regimes where policy updates are already delicate (small $\eta$), increasing $C$ can exacerbate instability by amplifying injected noise.

\subsection{Scaling with Feature Dimension}\label{subsec:d_sweep}

Finally, we examine scaling with the feature dimension $d\in\{3,5,7,9\}$ at fixed privacy budget $(\varepsilon,\delta)=(1,10^{-5})$, plotting the suboptimality gap $V_\eta(\pi_\eta^\star)-V_\eta(\hat\pi)$ as a function of $n$.
We report two regularization regimes, $\eta\in\{0.5,1.0\}$, to contrast a conservative versus a moderate departure budget.
For policy-optimization baselines, we fix the clipping norm at $C=2L(d)$ for each dimension to isolate the effect of $d$.

\begin{figure}[t]
    \centering
    \begin{subfigure}[b]{\linewidth}
        \centering
        \includegraphics[width=\linewidth]{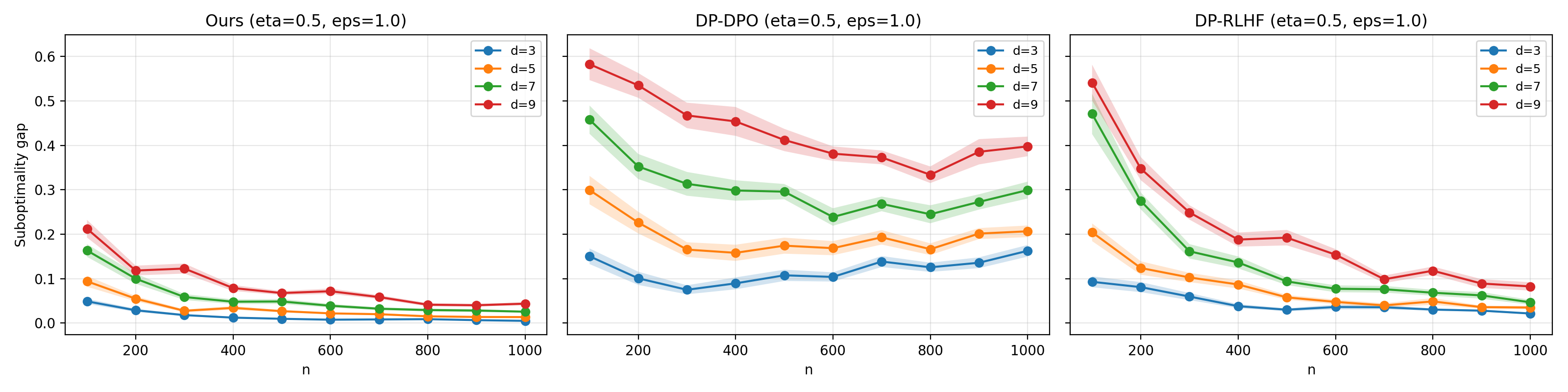}
        \caption{\textbf{$\eta=0.5$ (fixed $\varepsilon=1$, $\delta=10^{-5}$).} Baselines use $C=2L(d)$ for each dimension.}
        \label{fig:supp_d_sweep_eta0p5}
    \end{subfigure}

    \vspace{0.8em}

    \begin{subfigure}[b]{\linewidth}
        \centering
        \includegraphics[width=\linewidth]{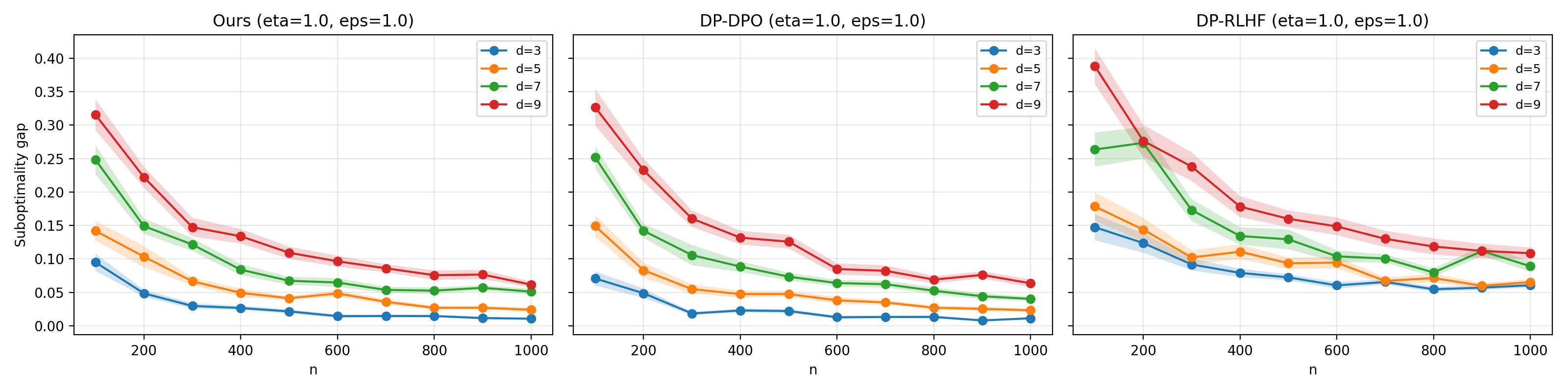}
        \caption{\textbf{$\eta=1.0$ (fixed $\varepsilon=1$, $\delta=10^{-5}$).} Baselines use $C=2L(d)$ for each dimension.}
        \label{fig:supp_d_sweep_eta1}
    \end{subfigure}

    \caption{\textbf{Scaling with dimension $d$ at $(\varepsilon,\delta)=(1,10^{-5})$.}
    We plot the suboptimality gap versus $n$ for $d\in\{3,5,7,9\}$ under $\eta\in\{0.5,1.0\}$.
    Shaded regions indicate $\pm 1$ s.e.\ over 30 seeds.}
    \label{fig:supp_d_sweep_stacked}
\end{figure}

Figure~\ref{fig:supp_d_sweep_stacked} evaluates how the difficulty of private preference learning scales with $d$.
Across methods, increasing $d$ generally increases the suboptimality gap at fixed $n$, reflecting the higher statistical complexity of reward estimation and the resulting propagation of estimation error to policy quality.
Comparing the two panels, the conservative regime $\eta=0.5$ yields smaller absolute gaps overall (since the KL-regularized optimum approaches the reference), but it also exposes sharper separation between methods, consistent with the heavy-tail diagnostics in Table~\ref{tab:ref_underperf}.
At $\eta=1.0$, gaps are larger in absolute magnitude but decrease more steadily with $n$, making the $d$-dependence visually clearer across the full range of sample sizes.
Taken together, the $d$-sweep supports the qualitative scaling predicted by the theory and shows that the relative ordering observed in the main synthetic comparisons persists across moderate changes in problem dimension.

\section{Proofs}\label{supp:proofs}
This section presents the complete proofs of the theoretical results stated in the main paper. For the convenience of the reader, the arguments are arranged sequentially according to their presentation in the text.

\subsection{Proof of Lemma~\ref{lem:hp-sc-dpsgd}}
We work with the empirical negative log-likelihood (average objective)
\[
L_n(\theta)
:=
\frac{1}{n}\sum_{i=1}^n \ell_i(\theta),
\]
with $\ell_i(\theta):=-\log \sigma\!\big(z_i(\theta)\big).$
Under the linear reward model $r_\theta(x,a)=\langle \theta,\phi(x,a)\rangle$, define for each sample
\begin{align*}
\left\{
\begin{aligned}
\Delta\phi_i
&:=
\phi(x_i,a_i^w)-\phi(x_i,a_i^l),\\
z_i(\theta)
&:=
\langle \theta,\Delta\phi_i\rangle.
\end{aligned}
\right.
\end{align*}

The proof has two parts. First, we show that the empirical objective $L_n$ is strongly convex over $\Theta$
with high probability by lower bounding its Hessian through (i) the uniform logistic curvature and
(ii) a concentration lower bound for the empirical Gram matrix of pairwise feature differences.
Second, once strong convexity and Lipschitzness are verified, we invoke the DP utility guarantee for
strongly convex Lipschitz losses to obtain the desired expected excess empirical risk bound.

\medskip
\noindent\textbf{Step 1. Hessian lower bound via logistic curvature.} We express the Hessian of $L_n$ and factor out a uniform curvature constant,
reducing strong convexity to a lower bound on the empirical Gram matrix. To this end, note that a direct differentiation gives
\begin{align*}
\left\{
\begin{aligned}
\nabla \ell_i(\theta)
&=-\bigl(1-\sigma(z_i(\theta))\bigr)\,\Delta\phi_i,\\
\nabla^2 \ell_i(\theta)
&=\sigma(z_i(\theta))\bigl(1-\sigma(z_i(\theta))\bigr)\,\Delta\phi_i\Delta\phi_i^\top.
\end{aligned}
\right.
\end{align*}
Therefore,
\begin{align*}
\nabla^2 L_n(\theta)
&=
\frac{1}{n}\sum_{i=1}^n
\sigma(z_i(\theta))\big(1-\sigma(z_i(\theta))\big)\,\Delta\phi_i\Delta\phi_i^\top.
\end{align*}
Using the boundedness assumptions $\|\theta\|_2\le R$ and $\|\phi(x,a)\|_2\le L$, we have
\begin{align*}
\left\{
\begin{aligned}
\|\Delta\phi_i\|_2
&\le
\|\phi(x_i,a_i^w)\|_2+\|\phi(x_i,a_i^l)\|_2
\le
2L,\\
|z_i(\theta)|
&=
|\langle \theta,\Delta\phi_i\rangle|
\le
\|\theta\|_2\,\|\Delta\phi_i\|_2
\le
2RL.
\end{aligned}
\right.
\end{align*}

Since $t\mapsto \sigma(t)(1-\sigma(t))$ is even and decreases on $[0,\infty)$, it follows that
\[
\sigma(z_i(\theta))\big(1-\sigma(z_i(\theta))\big)
\ge
\sigma(2RL)\big(1-\sigma(2RL)\big)
=:
c_{RL}.
\]
Consequently,
\begin{equation}
\label{eq:lem11-hessian-lower}
\nabla^2 L_n(\theta)
\succeq
c_{RL}\cdot \Big(\frac{1}{n}\sum_{i=1}^n \Delta\phi_i\Delta\phi_i^\top\Big)
\qquad
\text{for all }\theta\in\Theta.
\end{equation}

\medskip
\noindent\textbf{Step 2. High-probability lower bound for the empirical Gram matrix.} We use a matrix concentration inequality bound to show that the empirical Gram matrix is well-conditioned with high probability under Assumption~\ref{assump:nondeg_diff_pi0}.
Define
\begin{align*}
\left\{
\begin{aligned}
X_i
&:=
\Delta\phi_i\Delta\phi_i^\top,\\
G_n
&:=
\frac{1}{n}\sum_{i=1}^n X_i.
\end{aligned}
\right.
\end{align*}

Then $X_i\succeq 0$ and $\lambda_{\max}(X_i)
=\|\Delta\phi_i\|_2^2\le(2L)^2=4L^2.$
By Assumption~\ref{assump:nondeg_diff_pi0}, we have
\[
\lambda_{\min}\!\Big(\mathbb{E}[X_1]\Big)
=
\lambda_{\min}\!\Big(\mathbb{E}[\Delta\phi\,\Delta\phi^\top]\Big)
\ge
\lambda.
\]
Thus,
\begin{align*}
\mu_{\min}
:=
\lambda_{\min}\!\Big(\sum_{i=1}^n \mathbb{E}[X_i]\Big)
&=
n\,\lambda_{\min}\!\Big(\mathbb{E}[X_1]\Big)
\ge
n\lambda.
\end{align*}
Applying Lemma~\ref{lem:tropp-simplified} with $R=4L^2$ and $t=\tfrac12$ gives
\begin{align*}
\mathbb{P}\!\left(
\lambda_{\min}\!\Big(\sum_{i=1}^n X_i\Big)\le \frac12\,\mu_{\min}
\right)
\le
d\cdot \exp\!\left(-\frac{(1/2)^2\,\mu_{\min}}{2\cdot 4L^2}\right)
\le
d\cdot \exp\!\left(-\frac{n\lambda}{32L^2}\right).
\end{align*}
Equivalently,
\begin{align*}
\mathbb{P}\!\left(\lambda_{\min}(G_n)\le \frac{\lambda}{2}\right)
&\le
d\cdot \exp\!\left(-\frac{n\lambda}{32L^2}\right).
\end{align*}
Fix $\rho\in(0,1)$. If $n \ge \frac{32L^2}{\lambda}\log\!\Big(\frac{d}{\rho}\Big),$
then the event
\[
\mathcal{E}
:=
\left\{
\lambda_{\min}(G_n)\ge \frac{\lambda}{2}
\right\}
\]
satisfies $\mathbb{P}(\mathcal{E})\ge 1-\rho$.

\medskip
\noindent\textbf{Step 3. Strong convexity on $\mathcal{E}$.} We combine Steps 1--2 to conclude that $L_n$ is $\mu$-strongly convex on the high-probability event $\mathcal E$.

On $\mathcal{E}$, combining \eqref{eq:lem11-hessian-lower} with $\lambda_{\min}(G_n)\ge \lambda/2$ yields
\begin{align*}
\nabla^2 L_n(\theta)
&\succeq
c_{RL}\,G_n
\succeq
c_{RL}\,\frac{\lambda}{2}\,I_d
\qquad
\text{for all }\theta\in\Theta.
\end{align*}
Therefore, on $\mathcal{E}$, the empirical objective $L_n$ is $\mu$-strongly convex over $\Theta$ with
\[
\mu
=
c_{RL}\,\frac{\lambda}{2}
=
\frac{\lambda}{2}\,\sigma(2RL)\big(1-\sigma(2RL)\big).
\]

\medskip
\noindent\textbf{Step 4. Lipschitz constant.}  We verify a uniform Lipschitz bound for the per-sample loss, which is required by the DP utility theorem.

For any $\theta\in\Theta$,
\begin{align*}
\|\nabla \ell_i(\theta)\|_2
&=
\big(1-\sigma(z_i(\theta))\big)\,\|\Delta\phi_i\|_2\\
&\le
\|\Delta\phi_i\|_2
\le
2L.
\end{align*}
Hence each $\ell_i$ is $2L$-Lipschitz on $\Theta$, and so is the average objective $L_n$.

\medskip
\noindent\textbf{Step 5. DP utility via Lemma~\ref{lem:dpsgd_utility}.}
On the event $\mathcal{E}$, the objective $L_n$ is $\mu$-strongly convex and $2L$-Lipschitz on $\Theta$.
Applying Lemma~\ref{lem:dpsgd_utility} to the \emph{sum} objective
\[
\sum_{i=1}^n \ell_i(\theta)
=
n\,L_n(\theta),
\]
and then dividing by $n$ to convert back to the average objective yields
\begin{align*}
\mathbb{E}\!\left[L_n(\theta_{\mathrm{priv}})-L_n(\hat\theta)\,\middle|\,D\right]
&=
O\!\left(
\frac{d\,(2L)^2\,\log^2(n/\delta)\,\log(1/\delta)}{n^{2}\,\mu\,\varepsilon^{2}}
\right),
\end{align*}
where $\hat\theta=\arg\min_{\theta\in\Theta}L_n(\theta)$ and the expectation is over algorithmic randomness conditional on $D$.
This completes the proof.

\subsection{Formal Justification of Remark~\ref{remark:hp-dpsgd}}
\label{sec:supp-hp-mbdpsgd-coupling}

This subsection justifies Remark~\ref{remark:hp-dpsgd} by upgrading the conditional expected excess empirical risk control from Lemma~\ref{lem:hp-sc-dpsgd} to a fully unconditional high-probability statement.

\begin{theorem}[Unconditional high-probability excess empirical risk]\label{thm:hp-mbdpsgd-last-coupling}
Fix $\rho\in(0,1)$. Under the conditions in Lemma~\ref{lem:hp-sc-dpsgd}, there exists a numerical constant $C>0$ such that, with probability at least $1-\rho$ over the joint randomness of $\mathcal{D}$ and the learning procedure,
\begin{align}
\label{eq:supp-thm16}
L_n(\theta_{\mathrm{priv}})-L_n(\hat\theta)
=
\widetilde O\!\left(\frac{d\,L^2}{\mu\,n^2\varepsilon^2}\right).
\end{align}
\end{theorem}

\begin{proof}
We organize the proof into steps and explicitly separate the sources of randomness.

\medskip
\noindent\textbf{Step 1 (A high-probability data event for strong convexity and Lipschitzness).}
Apply Lemma~\ref{lem:hp-sc-dpsgd} with failure probability $\rho/3$.
Then there exists an event $\mathcal E_{\mathrm{sc}}$ measurable w.r.t.\ $D$ such that
\begin{align}
\label{eq:supp-esc}
\mathbb P(\mathcal E_{\mathrm{sc}})\ge 1-\rho/3,
\end{align}
and on $\mathcal E_{\mathrm{sc}}$ the empirical objective $L_n$ is $\mu$-strongly convex on $\Theta$ and satisfies the uniform gradient bound.
For all $i\in\{1,\dots,n\}$ and all $\theta\in\Theta$,
\begin{align}
\label{eq:supp-gradbound}
\|\nabla \ell_i(\theta)\|_2 \le 2L.
\end{align}

Consequently, on $\mathcal E_{\mathrm{sc}}$ we have the deterministic bounds, for all $\theta\in\Theta$ and all $t$,
\begin{align}
\label{eq:supp-Lipschitz}
\|\nabla L_n(\theta)\|_2
\le
\frac1n\sum_{i=1}^n \|\nabla \ell_i(\theta)\|_2
\le
2L,
\end{align}
and
\begin{align}
\label{eq:supp-mbgradbound}
\|g_t^{\mathrm{mb}}\|_2
\le
\frac1m\sum_{j=1}^m \|\nabla \ell_{I_{t,j}}(\theta_t)\|_2
\le
2L.
\end{align}
In particular,
\begin{align}
\label{eq:supp-mbdiffbound}
\|\nabla L_n(\theta_t)-g_t^{\mathrm{mb}}\|_2
\le
\|\nabla L_n(\theta_t)\|_2+\|g_t^{\mathrm{mb}}\|_2
\le
4L.
\end{align}

\medskip

\noindent\textbf{Step 2 (A high-probability bound for Gaussian magnitudes).} Set $\alpha=\rho/3$. Under the DP-SGD instantiation analyzed here, the additive perturbation at each iteration is Gaussian, so we write $\xi_t\sim \mathcal N(0,\sigma_{\mathrm{DP}}^2 I_d)$. The exact calibration of $\sigma_{\mathrm{DP}}^2$ depends on the specific DP-SGD variant, including the minibatch subsampling scheme and privacy accountant, and is not needed here. For our purposes, it suffices that the resulting Gaussian perturbation has variance scale $\sigma_{\mathrm{DP}}^2=\widetilde O\!\left(\frac{L^2\,T}{n^2\varepsilon^2}\right)$. Therefore, the random variable $\|\xi_t\|_2^2/\sigma_{\mathrm{DP}}^2$ has a $\chi_d^2$ distribution. By Lemma~\ref{lem:laurent-massart}, for any $u>0$,
\begin{align}
\label{eq:supp-LM}
\mathbb P\!\left(
\|\xi_t\|_2^2
\ge
\sigma_{\mathrm{DP}}^2\bigl(d+2\sqrt{du}+2u\bigr)
\right)
\le
e^{-u}.
\end{align}
Choose $u=\log\!\left(\frac{T}{\alpha}\right)$, and define the deterministic radius
\begin{align}
\label{eq:supp-B}
B=\sigma_{\mathrm{DP}}\,\sqrt{d+2\sqrt{d\log(T/\alpha)}+2\log(T/\alpha)}.
\end{align}
Then for each $t$,
\begin{align}
\label{eq:supp-single-tail}
\mathbb P(\|\xi_t\|_2>B)\le \alpha/T.
\end{align}
A union bound over $t=1,\dots,T$ yields the event
\begin{align}
\label{eq:supp-Enoise}
\mathcal E_{\mathrm{noise}}
=
\left\{\max_{1\le t\le T}\|\xi_t\|_2\le B\right\},
\end{align}
which satisfies $\mathbb P(\mathcal E_{\mathrm{noise}})\ge 1-\alpha = 1-\rho/3.$\\

\noindent\textbf{Step 3 (A coupling to replace unbounded perturbations by bounded ones).}
Define the truncated Gaussian law
\begin{align}
\label{eq:supp-nuB}
\nu_B
=
\mathcal L\!\left(\xi_1\ \middle|\ \|\xi_1\|_2\le B\right).
\end{align}
Let $(\xi_t')_{t\le T}$ be i.i.d.\ with $\xi_t'\sim \nu_B$, independent of $(D,(B_t)_{t\le T},(\xi_t)_{t\le T})$.
Define the coupled sequence
\begin{align}
\label{eq:supp-tilde-xi}
\tilde\xi_t
=
\begin{cases}
\xi_t, & \text{if }\|\xi_t\|_2\le B,\\
\xi_t', & \text{if }\|\xi_t\|_2> B.
\end{cases}
\end{align}
By construction, we have $\|\tilde\xi_t\|_2\le B$, almost surely for all $t$, and by symmetry of the Gaussian law and the centrally symmetric truncation set $\{\|\xi\|_2\le B\}$, $\mathbb E[\tilde\xi_t]=0$ for all $t$.

Moreover, on $\mathcal E_{\mathrm{noise}}$ we have $\tilde\xi_t=\xi_t$ simultaneously for all $t\le T$.

\medskip
\noindent\textbf{Step 4 (Define the shadow process and couple iterates).}
Define a \emph{shadow} iterate sequence $(\tilde\theta_t)_{t\le T+1}$ using the same mini-batches $(B_t)$ and the bounded noises $(\tilde\xi_t)$:
\begin{align}
\label{eq:supp-shadow}
\tilde\theta_{t+1}
=
\Pi_{\Theta}\!\Bigl(\tilde\theta_t-\eta_t\bigl(\tilde g_t^{\mathrm{mb}}+\tilde\xi_t\bigr)\Bigr),
\qquad
\tilde g_t^{\mathrm{mb}}
=
\frac1m\sum_{j=1}^m \nabla \ell_{I_{t,j}}(\tilde\theta_t),
\end{align}
initialized at $\tilde\theta_1=\theta_1$ and with the same step sizes $\eta_t=1/(\mu t)$.
On the event $\mathcal E_{\mathrm{noise}}$, we have $\tilde\xi_t=\xi_t$ for all $t$ and, by induction, for all $t=1,\dots,T+1,$
\begin{align}
\label{eq:supp-iterate-couple}
\theta_t=\tilde\theta_t
\end{align}

\medskip
\noindent\textbf{Step 5 (Verify the shadow recursion fits the bounded-oracle model).}
Fix an arbitrary dataset $D$ and work on the event $\mathcal E_{\mathrm{sc}}$.
Define the filtration
\begin{align}
\label{eq:supp-filtration}
\tilde{\mathcal F}_{t-1}
=
\sigma\!\bigl(D,\ B_1,\tilde\xi_1,\dots,B_{t-1},\tilde\xi_{t-1}\bigr),
\end{align}
and define the oracle noise term
\begin{align}
\label{eq:supp-z}
\tilde z_t
=
\nabla L_n(\tilde\theta_t)-\bigl(\tilde g_t^{\mathrm{mb}}+\tilde\xi_t\bigr).
\end{align}

\smallskip
\noindent\emph{(i) Conditional mean-zero.}
Conditional on $\tilde{\mathcal F}_{t-1}$, the iterate $\tilde\theta_t$ is fixed.
Using \eqref{eq:supp-shadow} and the fact that each $I_{t,j}$ is uniform on $\{1,\dots,n\}$,
\begin{align}
\label{eq:supp-unbiased}
\mathbb E\!\left[\tilde g_t^{\mathrm{mb}}\,\middle|\,\tilde{\mathcal F}_{t-1}\right]
&=
\frac1m\sum_{j=1}^m
\mathbb E\!\left[\nabla \ell_{I_{t,j}}(\tilde\theta_t)\,\middle|\,\tilde{\mathcal F}_{t-1}\right] \nonumber\\
&=
\mathbb E\!\left[\nabla \ell_{I_{t,1}}(\tilde\theta_t)\,\middle|\,\tilde{\mathcal F}_{t-1}\right] \nonumber\\
&=
\frac1n\sum_{i=1}^n \nabla \ell_i(\tilde\theta_t)
=
\nabla L_n(\tilde\theta_t).
\end{align}
Also, $\tilde\xi_t$ is independent of $\tilde{\mathcal F}_{t-1}$ and has mean zero.
Therefore, taking conditional expectations in \eqref{eq:supp-z} gives
\begin{align}
\label{eq:supp-mds}
\mathbb E\!\left[\tilde z_t\,\middle|\,\tilde{\mathcal F}_{t-1}\right]=0.
\end{align}

\smallskip
\noindent\emph{(ii) Almost-sure bound.}
On $\mathcal E_{\mathrm{sc}}$, combining \eqref{eq:supp-mbdiffbound} (applied at $\tilde\theta_t$) with \eqref{eq:supp-z} yields
\begin{align}
\label{eq:supp-z-bound}
\|\tilde z_t\|_2
\le
\|\nabla L_n(\tilde\theta_t)-\tilde g_t^{\mathrm{mb}}\|_2+\|\tilde\xi_t\|_2
\le
4L+B
\qquad
\text{almost surely for all }t.
\end{align}

\medskip
\noindent\textbf{Step 6 Invoke Lemma~\ref{lem:harvey-scaled}).} On $\mathcal E_{\mathrm{sc}}$, the function $L_n$ is $\mu$-strongly convex on $\Theta$ and is $2L$-Lipschitz by \eqref{eq:supp-Lipschitz}.
Moreover, \eqref{eq:supp-mds} and \eqref{eq:supp-z-bound} verify the bounded-oracle conditions with noise radius $Z=4L+B$. Invoke Lemma~\ref{lem:harvey-scaled} with failure probability $\alpha=\rho/3$. Then, on $\mathcal E_{\mathrm{sc}}$, with probability at least $1-\rho/3$ over the shadow algorithmic randomness,
\begin{align}
\label{eq:supp-shadow-harvey}
L_n(\tilde\theta_{T+1})-L_n(\hat\theta)
\le
\widetilde O\!\left(\frac{(L+B)^2}{\mu\,T}\right).
\end{align}
\noindent
Here we absorbed all absolute constants and the $\log T$ and $\log(1/\rho)$ factors into $\widetilde O(\cdot)$.

\medskip
\noindent\textbf{Step 7 (Transfer to the real iterate and make the probability unconditional).}
On $\mathcal E_{\mathrm{noise}}$, the coupling \eqref{eq:supp-iterate-couple} implies $\theta_{T+1}=\tilde\theta_{T+1}$, hence
\eqref{eq:supp-shadow-harvey} holds with $\tilde\theta_{T+1}$ replaced by $\theta_{T+1}$ on $\mathcal E_{\mathrm{sc}}\cap \mathcal E_{\mathrm{noise}}$.
Therefore, combining the three failure probabilities,
\begin{align}
\label{eq:supp-union}
\mathbb P(\mathcal E_{\mathrm{sc}}^c)\le \rho/3,
\qquad
\mathbb P(\mathcal E_{\mathrm{noise}}^c)\le \rho/3,
\qquad
\mathbb P(\text{Harvey failure on }\mathcal E_{\mathrm{sc}})\le \rho/3,
\end{align}
a union bound yields that \eqref{eq:supp-shadow-harvey} (with $\tilde\theta_{T+1}$ replaced by $\theta_{T+1}$) holds with overall probability at least $1-\rho$
over the joint randomness of $D$ and the learning procedure.

\medskip
\noindent\textbf{Step 8 (Rate-level simplification and DP calibration).}
We now simplify the right-hand side of \eqref{eq:supp-shadow-harvey}.
From \eqref{eq:supp-B} and the inequality $2\sqrt{du}\le d+u$ applied with $u=\log(T/\alpha)$, we obtain
\begin{align}
\label{eq:supp-B2}
B^2
\le
\sigma_{\mathrm{DP}}^2\,
\widetilde O(d).
\end{align}
Substituting $\sigma_{\mathrm{DP}}^2=\widetilde O\!\left(\frac{L^2\,T}{n^2\varepsilon^2}\right)$ into \eqref{eq:supp-B2} gives
\begin{align}
\label{eq:supp-B2-rate}
\frac{B^2}{T}
=
\widetilde O\!\left(\frac{d\,L^2}{n^2\varepsilon^2}\right).
\end{align}
Plugging \eqref{eq:supp-B2-rate} into \eqref{eq:supp-shadow-harvey} yields, with probability at least $1-\rho$,
\begin{align}
\label{eq:supp-final-rate}
L_n(\theta_{T+1})-L_n(\hat\theta)
\le
\widetilde O\!\left(\frac{L^2}{\mu\,T}\right)
+
\widetilde O\!\left(\frac{d\,L^2}{\mu\,n^2\varepsilon^2}\right).
\end{align}
Choosing $T$ sufficiently large (as in our main procedure) makes the optimization term $\widetilde O(L^2/(\mu T))$ negligible at the rate level.
Thus we obtain \eqref{eq:supp-thm16}, completing the proof.
\end{proof}

\subsection{Proof of Theorem \ref{thm:main-gap}}\label{sec: sub gap global proof}

This section provides the proof of the suboptimality gap for the proposed framework. The argument proceeds in two steps. We first relate the KL-regularized policy suboptimality gap to a reward estimation error term, and then use the coverage assumption to transfer this control to the reference policy $\pi_0$. This proof strategy follows a line of analysis suggested in \citet{zhao2024sharp}. In our setting, it is combined with the private reward-estimation bound established above.\\

\noindent\textbf{Step 1. Set-up: suboptimality-gap decomposition via a KL-regularized functional.}
Recall the KL-regularized value of a policy $\pi$ (reference $\pi_0$, temperature $\eta>0$):
\[
Q(\pi)
=
\E_{x\sim d_0}\sum_{a\in\mathcal A}\pi(a\mid x)
\left[
r_{\theta^*}(x,a)\;-\;\frac{1}{\eta}\log\frac{\pi(a\mid x)}{\pi_0(a\mid x)}
\right].
\]
Our goal is to control the suboptimality gap
\[
Q(\pi_{\theta^\ast}^{\eta})-Q(\pi_{\tilde\theta}^{\eta}),
\]
where $\pi_{\theta}^{\eta}$ denotes the Boltzmann policy induced by the reward $r_{\theta}(\cdot,\cdot)$:
\[
\pi_{\theta}^{\eta}(a\mid x)
=
\frac{\pi_0(a\mid x)\exp\!\big(\eta r_{\theta}(x,a)\big)}{Z_{\theta}^{\eta}(x)},
\]
where $Z_{\theta}^{\eta}(x)
:=\sum_{a\in\mathcal A}\pi_0(a\mid x)\exp\!\big(\eta r_{\theta}(x,a)\big).$
Using the identity
\begin{align*}
\log\frac{\pi_{\theta}^{\eta}(a\mid x)}{\pi_0(a\mid x)}
&=
\eta r_{\theta}(x,a)\;-\;\log Z_{\theta}^{\eta}(x),
\end{align*}
we can rewrite $Q(\pi_\theta^\eta)$ as
\begin{align*}
Q(\pi_\theta^\eta)
&=
\E_{x\sim d_0}\sum_{a\in\mathcal A}\pi_\theta^\eta(a\mid x)
\left[
r_{\theta^*}(x,a)\;-\;\frac{1}{\eta}\Big(\eta r_{\theta}(x,a)-\log Z_{\theta}^{\eta}(x)\Big)
\right] \\
&=
\E_{x\sim d_0}\sum_{a\in\mathcal A}\pi_\theta^\eta(a\mid x)
\Big[r_{\theta^*}(x,a)-r_{\theta}(x,a)\Big]
\;+\;
\frac{1}{\eta}\,\E_{x\sim d_0}\!\Big[\log Z_{\theta}^{\eta}(x)\Big].
\end{align*}
Therefore,
\begin{equation}
\label{eq:step1-gap-decomp}
\begin{split}
Q(\pi_{\theta^\ast}^{\eta})-Q(\pi_{\tilde\theta}^{\eta})
&=
\frac{1}{\eta}\E_{x\sim d_0}\!\Big[\log Z_{\theta^\ast}^{\eta}(x)-\log Z_{\tilde\theta}^{\eta}(x)\Big]\\
&\quad
-\E_{x\sim d_0}\sum_{a\in\mathcal A}\pi_{\tilde\theta}^{\eta}(a\mid x)\Big(r_{\theta^*}(x,a)-r_{\tilde \theta}(x,a)\Big).
\end{split}
\end{equation}

To package the two terms in \eqref{eq:step1-gap-decomp} into a single contextwise functional, let
$f:\mathcal X\times\mathcal A\to\mathbb R$ be arbitrary and define
\[
\pi_f^\eta(a\mid x)
=
\frac{\pi_0(a\mid x)\exp\!\big(\eta f(x,a)\big)}{Z_f^\eta(x)},
\]
with $Z_f^\eta(x):=\sum_{a\in\mathcal A}\pi_0(a\mid x)\exp\!\big(\eta f(x,a)\big)$,  and define the reward difference
\[
\Delta_f(x,a):=f(x,a)-r_{\theta^*}(x,a).
\]
We set
\begin{equation}
\label{eq:def-J}
J\big(f(x,\cdot)\big)
:=
\log Z_f^\eta(x)\;-\;\eta\sum_{a\in\mathcal A}\pi_f^\eta(a\mid x)\,\Delta_f(x,a).
\end{equation}
Then for any $f$,
\begin{align*}
Q(\pi_f^\eta)
&=
\E_{x\sim d_0}\sum_{a\in\mathcal A}\pi_f^\eta(a\mid x)
\left[
r_{\theta^*}(x,a)-\frac{1}{\eta}\log\frac{\pi_f^\eta(a\mid x)}{\pi_0(a\mid x)}
\right] \\
&=
\E_{x\sim d_0}\sum_{a\in\mathcal A}\pi_f^\eta(a\mid x)
\left[
r_{\theta^*}(x,a)-\frac{1}{\eta}\Big(\eta f(x,a)-\log Z_f^\eta(x)\Big)
\right] \\
&=
\E_{x\sim d_0}\sum_{a\in\mathcal A}\pi_f^\eta(a\mid x)\Big(r_{\theta^*}(x,a)-f(x,a)\Big)
\;+\;
\frac{1}{\eta}\,\E_{x\sim d_0}\!\Big[\log Z_f^\eta(x)\Big] \\
&=
-\E_{x\sim d_0}\sum_{a\in\mathcal A}\pi_f^\eta(a\mid x)\,\Delta_f(x,a)
\;+\;
\frac{1}{\eta}\,\E_{x\sim d_0}\!\Big[\log Z_f^\eta(x)\Big] \\
&=
\frac{1}{\eta}\,\E_{x\sim d_0}\!\Big[J\big(f(x,\cdot)\big)\Big].
\end{align*}
Thus,
\begin{equation}
\label{eq:Q-as-J}
Q(\pi_f^\eta)
=
\frac{1}{\eta}\,\E_{x\sim d_0}\!\Big[J\big(f(x,\cdot)\big)\Big].
\end{equation}
In particular, taking $f(x,a)=r_{\theta}(x,a)$ gives
\[
Q(\pi_{\theta}^{\eta})
=
\frac{1}{\eta}\,\E_{x\sim d_0}\!\Big[J\big(r_{\theta}(x,\cdot)\big)\Big],
\]
and therefore
\[
Q(\pi_{\theta^\ast}^{\eta})-Q(\pi_{\tilde\theta}^{\eta})
=
\frac{1}{\eta}\,\E_{x\sim d_0}\!\Big[
J\big(r_{\theta^*}(x,\cdot)\big)-J\big(r_{\tilde \theta}(x,\cdot)\big)
\Big].
\]

\medskip
\noindent\textbf{Step 2. From functional gap to a squared reward error.}  In this step we upper bound the functional gap 
\begin{align*}
J\big(r_{\theta^*}(x,\cdot)\big)-J\big(r_{\tilde \theta}(x,\cdot)\big)    
\end{align*}
by a quadratic reward-discrepancy term. We interpolate linearly between $r_{\theta^*}(x,\cdot)$ and $r_{\tilde \theta}(x,\cdot)$, differentiate $J$ along this one-dimensional path, and use softmax log-partition calculus to express the derivative as a variance. This yields a second-moment bound under an interpolating Boltzmann policy.

To this end, fix $x\in\mathcal X$ and define the pointwise reward error
\[
\Delta_x(a)
:=
r_{\tilde \theta}(x,a)-r_{\theta^*}(x,a),
\qquad a\in\mathcal A.
\]
Consider the one-dimensional interpolation, that is, for $t \in [0,1]$,
\[
f_t(x,a)
:=
r_{\theta^*}(x,a)+t\,\Delta_x(a)
=
(1-t)r_{\theta^*}(x,a)+t\,r_{\tilde \theta}(x,a).
\]
Let $\pi_t^\eta(\cdot\mid x):=\pi_{f_t}^\eta(\cdot\mid x)$ and $Z_t^\eta(x):=Z_{f_t}^\eta(x)$.
Define
\[
\psi_x(t)
:=
J\big(f_t(x,\cdot)\big)
=
\log Z_t^\eta(x)\;-\;\eta\sum_{a\in\mathcal A}\pi_t^\eta(a\mid x)\,\big(t\,\Delta_x(a)\big).
\]

\smallskip
\noindent\textbf{Step 2(a). A derivative identity for $\psi_x'(t)$.}
Write $m_x(t):=\E_{a\sim\pi_t^\eta(\cdot\mid x)}[\Delta_x(a)]$. Then
\begin{align*}
\frac{d}{dt}\log Z_t^\eta(x)
&=
\frac{1}{Z_t^\eta(x)}\sum_{a\in\mathcal A}\pi_0(a\mid x)\exp\!\big(\eta f_t(x,a)\big)\cdot \eta\,\Delta_x(a) \\
&=
\eta\,\sum_{a\in\mathcal A}\pi_t^\eta(a\mid x)\,\Delta_x(a)
=
\eta\,m_x(t).
\end{align*}
Moreover, as $\psi_x(t) = \log Z_t^\eta(x)\;-\;\eta\,t\,m_x(t),$ we have
\begin{align*}
\psi_x'(t)
&=
\eta\,m_x(t)\;-\;\eta\,m_x(t)\;-\;\eta\,t\,m_x'(t) \\
&=
-\eta\,t\,m_x'(t).
\end{align*}
Next, differentiate $m_x(t)=\sum_a \pi_t^\eta(a\mid x)\Delta_x(a)$:
\begin{align*}
m_x'(t)
&=
\sum_{a\in\mathcal A}\Delta_x(a)\,\frac{d}{dt}\pi_t^\eta(a\mid x).
\end{align*}
Since
\[
\pi_t^\eta(a\mid x)
=
\frac{\pi_0(a\mid x)\exp\!\big(\eta f_t(x,a)\big)}{Z_t^\eta(x)},
\]
we have
\begin{align*}
\frac{d}{dt}\pi_t^\eta(a\mid x)
&=
\pi_t^\eta(a\mid x)\cdot \eta\,\Delta_x(a)
\;-\;
\pi_t^\eta(a\mid x)\cdot \frac{d}{dt}\log Z_t^\eta(x) \\
&=
\eta\,\pi_t^\eta(a\mid x)\Big(\Delta_x(a)-m_x(t)\Big).
\end{align*}
Therefore,
\begin{align*}
m_x'(t)
&=
\eta\sum_{a\in\mathcal A}\pi_t^\eta(a\mid x)\,\Delta_x(a)\Big(\Delta_x(a)-m_x(t)\Big) \\
&=
\eta\sum_{a\in\mathcal A}\pi_t^\eta(a\mid x)\Big(\Delta_x(a)^2-\Delta_x(a)\,m_x(t)\Big) \\
&=
\eta\left(
\E_{a\sim\pi_t^\eta(\cdot\mid x)}[\Delta_x(a)^2]
-
m_x(t)^2
\right) \\
&=
\eta\,\Var_{a\sim\pi_t^\eta(\cdot\mid x)}\!\big(\Delta_x(a)\big).
\end{align*}
Plugging into $\psi_x'(t)=-\eta\,t\,m_x'(t)$ gives
\begin{equation}
\label{eq:psi-prime-variance}
\psi_x'(t)
=
-\eta^2\,t\,
\Var_{a\sim\pi_t^\eta(\cdot\mid x)}\!\big(\Delta_x(a)\big).
\end{equation}

\smallskip
\noindent\textbf{Step 2(b). Mean-value bound and a squared-error control.}
For any $t \in [0,1]$, define the scalar function
\[
G(t)
:=
\mathbb{E}_{x\sim d_0}\big[\psi_x(t)\big]
=
\mathbb{E}_{x\sim d_0}\Big[J\big(f_t(x,\cdot)\big)\Big],
\]
where $\psi_x(t)=J(f_t(x,\cdot))$ is defined above.
By \eqref{eq:psi-prime-variance}, we have
\[
\psi_x'(t)
=
-\eta^2\,t\,
\Var_{a\sim\pi_t^\eta(\cdot\mid x)}\!\big(\Delta_x(a)\big).
\]
Differentiating under the expectation yields
\[
G'(t)
=
\mathbb{E}_{x\sim d_0}\big[\psi_x'(t)\big]
=
-\eta^2\,t\,
\mathbb{E}_{x\sim d_0}\Big[\Var_{a\sim\pi_t^\eta(\cdot\mid x)}\!\big(\Delta_x(a)\big)\Big].
\]
Applying the mean value theorem to the scalar function $G$ on $[0,1]$, there exists a  $\gamma\in(0,1)$ such that
\[
G(0)-G(1)=-G'(\gamma)
=
\eta^2\,\gamma\,
\mathbb{E}_{x\sim d_0}\Big[\Var_{a\sim\pi_\gamma^\eta(\cdot\mid x)}\!\big(\Delta_x(a)\big)\Big].
\]
Using $\Var(Y)\le \mathbb{E}[Y^2]$ and $\gamma\le 1$, we obtain
\[
G(0)-G(1)
\le
\eta^2\,
\mathbb{E}_{x\sim d_0,\ a\sim\pi_\gamma^\eta(\cdot\mid x)}\!\big[\Delta_x(a)^2\big].
\]
Define the interpolating reward $f:=f_\gamma$,
\[
f(x,a):=(1-\gamma)r_{\theta^*}(x,a)+\gamma r_{\tilde \theta}(x,a),
\]
so that $\pi_\gamma^\eta(\cdot\mid x)=\pi_f^\eta(\cdot\mid x)$.

\smallskip
\noindent\textbf{Step 2(c). Consequence for the suboptimality gap.}
Recalling \eqref{eq:Q-as-J}, we have
\[
Q(\pi_{\theta^\ast}^{\eta})-Q(\pi_{\tilde\theta}^{\eta})
=
\frac{1}{\eta}\Big(G(0)-G(1)\Big).
\]
Combining with the bound in Step~2(b) yields
\begin{equation}
\label{eq:step2-gap-to-sqerr}
Q(\pi_{\theta^\ast}^{\eta})-Q(\pi_{\tilde\theta}^{\eta})
\le
\eta\,
\mathbb{E}_{x\sim d_0,\ a\sim\pi_f^\eta(\cdot\mid x)}
\Big[\big(r_{\tilde \theta}(x,a)-r_{\theta^*}(x,a)\big)^2\Big].
\end{equation}

The key implication of \eqref{eq:step2-gap-to-sqerr} is that controlling the policy suboptimality gap is reduced to controlling a reward estimation error term under an interpolating policy.\\

\noindent\textbf{Step 3 (Change of measure under point-wise coverage).}
For the reward function $f$ in Step~2 and the induced policy $\pi_f^\eta\in\Pi$, define
\[
w_f(x,a)\;:=\;\frac{\pi_f^\eta(a\mid x)}{\pi_0(a\mid x)},
\]
\text{with the convention } $0/0=0.$
By Assumption~\ref{assump:coverage}, for any $(x,a)$ with $d_0(x)>0$,
\[
0 \le w_f(x,a) \le C.
\]
Therefore,
\begin{align*}
&\mathbb{E}_{x \sim d_{0}}\sum_{a \in \mathcal{A}}\pi_{f}^{\eta}(a \mid x)
\Big\{r_{\tilde{\theta}}(x,a)-r_{\theta^\star}(x,a)\Big\}^{2} \\
&= \mathbb{E}_{x \sim d_{0}}\sum_{a \in \mathcal{A}}\pi_{0}(a \mid x)\, w_f(x,a)
\Big\{r_{\tilde{\theta}}(x,a)-r_{\theta^\star}(x,a)\Big\}^{2} \\
&\le C\,\mathbb{E}_{x \sim d_{0}}\sum_{a \in \mathcal{A}}\pi_{0}(a \mid x)
\Big\{r_{\tilde{\theta}}(x,a)-r_{\theta^\star}(x,a)\Big\}^{2}.
\end{align*}
Plugging this bound into \eqref{eq:step2-gap-to-sqerr} yields
\begin{align*}
Q(\pi_{\theta^\star}^{\eta}) - Q(\pi_{\tilde{\theta}}^{\eta})
&\le \eta\,C\,\mathbb{E}_{x \sim d_{0}}\sum_{a \in \mathcal{A}}\pi_{0}(a \mid x)
\Big\{r_{\tilde{\theta}}(x,a)-r_{\theta^\star}(x,a)\Big\}^{2}.
\end{align*}

\medskip

\noindent\textbf{Step 4 (Bound the $\pi_0$-reward MSE by statistical + DP errors).}
By Assumption~\ref{assump:linear_reward_bounded_phi}, for any $(x,a)$, we have
\begin{align*}
r_{\tilde{\theta}}(x,a)-r_{\theta^\star}(x,a)
&=
\langle \tilde\theta-\theta^\star,\phi(x,a)\rangle.
\end{align*}
Hence,
\begin{align*}
&\mathbb{E}_{x \sim d_{0}}\sum_{a \in \mathcal{A}}\pi_{0}(a \mid x)
\Big\{r_{\tilde{\theta}}(x,a)-r_{\theta^\star}(x,a)\Big\}^{2} \\
&=
\mathbb{E}_{x \sim d_{0}}\mathbb{E}_{a\sim \pi_0(\cdot\mid x)}
\Bigl[\langle \tilde\theta-\theta^\star,\phi(x,a)\rangle\Bigr]^2 \\
&\le
\mathbb{E}_{x \sim d_{0}}\mathbb{E}_{a\sim \pi_0(\cdot\mid x)}
\Bigl[\|\tilde\theta-\theta^\star\|_2^2\ \|\phi(x,a)\|_2^2\Bigr] \\
&\le
\Bigl(\sup_{x,a}\|\phi(x,a)\|_2^2\Bigr)\,\|\tilde\theta-\theta^\star\|_2^2.
\end{align*}

Let $\hat\theta$ denote the non-private MLE. Using the quadratic inequality,
\begin{align*}
\|\tilde\theta-\theta^\star\|_2^2
&=
\|(\tilde\theta-\hat\theta)+(\hat\theta-\theta^\star)\|_2^2 \\
&\le
2\|\tilde\theta-\hat\theta\|_2^2
+
2\|\hat\theta-\theta^\star\|_2^2.
\end{align*}

\smallskip
\noindent\textbf{Statistical term.}
Invoke Lemma~\ref{lem:mle_stat} with failure probability $\rho/2$.
With probability at least $1-\rho/2$ (over the data randomness),
\begin{align*}
\|\hat\theta-\theta^\star\|_2^2
\;\le\;
O\!\left(\frac{d+\log(2/\rho)}{n}\right).
\end{align*}

\smallskip
\noindent\textbf{DP term (high-probability, all randomness).} We now control $\|\tilde\theta-\hat\theta\|_2^2$ in high probability over \emph{both} the data randomness and the algorithmic randomness. Invoke Theorem~\ref{thm:hp-mbdpsgd-last-coupling} with failure probability $\rho/2$. With probability at least $1-\rho/2$, we have the excess empirical risk bound
\begin{align*}
\bar L_n(\tilde\theta) - \bar L_n(\hat\theta)
\;\le\;
\widetilde{O}\!\left(\frac{d\,L^2}{\mu\,n^2\varepsilon^2}\right),
\end{align*}
ignoring polylogarithmic factors in $(n,1/\delta,1/\rho)$.

On the same event, by $\mu$-strong convexity of $\bar L_n$ on $\Theta$,
\begin{align*}
\bar L_n(\tilde\theta) - \bar L_n(\hat\theta)
\;\ge\;
\frac{\mu}{2}\,\|\tilde\theta-\hat\theta\|_2^2,
\end{align*}
which implies
\begin{align*}
\|\tilde\theta-\hat\theta\|_2^2
\;\le\;
\frac{2}{\mu}\Bigl\{\bar L_n(\tilde\theta) - \bar L_n(\hat\theta)\Bigr\}
\;\le\;
\widetilde{O}\!\left(\frac{d\,L^2}{\mu^2\,n^2\varepsilon^2}\right).
\end{align*}

\smallskip
\noindent\textbf{Combine.}
By a union bound over the statistical event and the DP event,
with probability at least $1-\rho$ (over both the data randomness and the algorithmic randomness),
\begin{align*}
\|\tilde\theta-\theta^\star\|_2^2
&\le
2\|\tilde\theta-\hat\theta\|_2^2
+
2\|\hat\theta-\theta^\star\|_2^2 \\
&\le
\widetilde{O}\!\left(\frac{d\,L^2}{\mu^2\,n^2\varepsilon^2}\right)
+
O\!\left(\frac{d+\log(2/\rho)}{n}\right).
\end{align*}
Consequently, with the same probability at least $1-\rho$,
\begin{align*}
&\mathbb{E}_{x \sim d_{0}}\sum_{a \in \mathcal{A}}\pi_{0}(a \mid x)
\Big\{r_{\tilde{\theta}}(x,a)-r_{\theta^\star}(x,a)\Big\}^{2} \\
&\le
\Bigl(\sup_{x,a}\|\phi(x,a)\|_2^2\Bigr)\,
\|\tilde\theta-\theta^\star\|_2^2 \\
&\le
\Bigl(\sup_{x,a}\|\phi(x,a)\|_2^2\Bigr)
\left\{
O\!\left(\frac{d+\log(2/\rho)}{n}\right)
+
\widetilde{O}\!\left(\frac{d\,L^2}{\mu^2\,n^2\varepsilon^2}\right)
\right\}.
\end{align*}

\medskip
\noindent\textbf{Conclusion.}
Combining Step~3 and Step~4 gives, with probability at least $1-\rho$,
\begin{align*}
Q(\pi_{\theta^\star}^{\eta}) - Q(\pi_{\tilde{\theta}}^{\eta})
&\le
\eta\,C\,
\Bigl(\sup_{x,a}\|\phi(x,a)\|_2^2\Bigr)
\left\{
O\!\left(\frac{d+\log(2/\rho)}{n}\right)
+
\widetilde{O}\!\left(\frac{d\,L^2}{\mu^2\,n^2\varepsilon^2}\right)
\right\}.
\end{align*}

\subsection{Proof of Theorem~\ref{thm:minimax-lb-main}}\label{app:proof_dp_gap_lb}
In this section, we provide the rigorous proof of Theorem~\ref{thm:minimax-lb-main}. Our analysis decomposes the lower bound into two fundamental barriers: the statistical complexity inherent to the reward class and the information-theoretic limit imposed by DP.

\textbf{Step 1. Non-private minimax lower bound.}
We establish a non-private baseline by inverting the sample-complexity lower bound of
\citet{zhao2024sharp}. Throughout this step, ``uniformly guarantees'' means
\begin{equation}
\label{eq:uniform_guarantee_def}
\sup_{\theta^\star \in \Theta}\;\mathbb{E}_{\theta^\star}\!\left[\mathrm{Gap}(\hat\pi;\theta^\star)\right]\;\le\;\mu,
\end{equation}
where the expectation is over the algorithm's randomness and the data generated under $\theta^\star$.

\medskip
\noindent
\emph{Fast-rate (local-curvature) branch.}
For KL-regularized objectives, Proposition~\ref{prop:np_lb_from_2411} implies that any algorithm satisfying \eqref{eq:uniform_guarantee_def} must have
\begin{equation}
\label{eq:zhao_sc_step1}
n \;\ge\; c_0 \min\!\left\{\frac{\eta \log \mathcal{N}_{\mathcal R}(\mu)}{\mu},\;
\frac{\log \mathcal{N}_{\mathcal R}(\mu)}{\mu^2}\right\},
\end{equation}
for a universal constant $c_0>0$, where $\mathcal{N}_{\mathcal R}(\mu)$ denotes the $\mu$-covering number of
$\mathcal R$.
Fix a universal constant $\bar c\in(0,1)$ (e.g., $\bar c=1/64$ in \citet{zhao2024sharp}) and restrict to target gaps
$\mu$ such that
\begin{equation}
\label{eq:eta_mu_small}
\eta\mu \;\le\; \bar c.
\end{equation}
Then the first term in \eqref{eq:zhao_sc_step1} is the active branch:
indeed,
\[
\frac{\log \mathcal{N}_{\mathcal R}(\mu)/\mu^2}{\eta \log \mathcal{N}_{\mathcal R}(\mu)/\mu}
\;=\;\frac{1}{\eta\mu}
\;\ge\;\frac{1}{\bar c},
\]
so $\log \mathcal{N}_{\mathcal R}(\mu)/\mu^2 \ge (\eta \log \mathcal{N}_{\mathcal R}(\mu))/\mu$.
Hence \eqref{eq:zhao_sc_step1} simplifies to
\begin{equation}
\label{eq:zhao_local_step1}
n \;\ge\; c_0 \frac{\eta \log \mathcal{N}_{\mathcal R}(\mu)}{\mu},
\qquad \text{whenever } \eta\mu \le \bar c.
\end{equation}

Recall that for the $d$-dimensional linear reward class considered in this work, standard volumetric arguments yield $\log \mathcal{N}_{\mathcal R}(\mu) \asymp d\log(1/\mu)$ (up to universal constants). In particular, fixing any $\mu_\star\in(0,1)$ and restricting to $\mu\in(0,\mu_\star]$ allows us to absorb the factor $\log(1/\mu)$ into constants, so there exist constants $c_{\mathrm{ent}}>0$ and $\mu_\star\in(0,1)$ such that
\begin{equation}
\label{eq:entropy_lb_step1_clean}
\log \mathcal{N}_{\mathcal R}(\mu) \;\ge\; c_{\mathrm{ent}}\, d,
\qquad \text{for all } \mu\in(0,\mu_\star].
\end{equation}

Fix
\[
\left\{
\begin{aligned}
\kappa &:= \frac{c_0 c_{\mathrm{ent}}\eta}{2}, \\
\mu_n &:= \kappa\,\frac{d}{n}.
\end{aligned}
\right.
\]
Define the non-private threshold
\begin{equation}
\label{eq:n_np_def_clean}
n_{\mathrm{np}}
\;:=\;
\max\!\left\{
\frac{\kappa d}{\mu_\star},\;
\frac{\eta\kappa d}{\bar c}
\right\}
\;=\;
\max\!\left\{
\frac{c_0 c_{\mathrm{ent}}\eta}{2\mu_\star}\,d,\;
\frac{c_0 c_{\mathrm{ent}}\eta^2}{2\bar c}\,d
\right\}.
\end{equation}
For any $n\ge n_{\mathrm{np}}$, we have $\mu_n \le \mu_\star$ and $\eta\mu_n \le \bar c$, so that
\eqref{eq:entropy_lb_step1_clean} and \eqref{eq:zhao_local_step1} apply at $\mu=\mu_n$.

Suppose, toward a contradiction, that there exists an algorithm satisfying
$\sup_{\theta^\star\in\Theta}\mathbb{E}_{\theta^\star}[\mathrm{Gap}(\hat\pi;\theta^\star)]\le \mu_n$.
Then \eqref{eq:zhao_local_step1} and \eqref{eq:entropy_lb_step1_clean} yield
\[
n
\;\ge\;
c_0 \frac{\eta \log \mathcal{N}_{\mathcal R}(\mu_n)}{\mu_n}
\;\ge\;
c_0 \frac{\eta (c_{\mathrm{ent}} d)}{\kappa d/n}
\;=\;
c_0 \frac{\eta c_{\mathrm{ent}}}{\kappa}\, n
\;=\; 2n,
\]
which is a contradiction. Therefore, no algorithm can satisfy \eqref{eq:uniform_guarantee_def} with $\mu=\mu_n$, and hence
\begin{equation}
\label{eq:np_lower_bound_final_clean}
R_n^{\mathrm{np}}
\;:=\;
\inf_{\hat\pi}\;\sup_{\theta^\star\in\Theta}\mathbb{E}_{\theta^\star}\!\left[\mathrm{Gap}(\hat\pi;\theta^\star)\right]
\;\ge\;
\mu_n
\;=\;
\kappa \frac{d}{n}
\;\gtrsim\;
\frac{d}{n},
\end{equation}
for all $n\ge n_{\mathrm{np}}$.

\medskip
Since the class of $(\varepsilon,\delta)$-DP algorithms is a subset of all randomized algorithms,
\[
R_n(\varepsilon,\delta)\;\ge\; R_n^{\mathrm{np}}.
\]
\noindent\textbf{Step 2. Hard instance construction.} We instantiate the environment so that the context
features align with the canonical basis of $\mathbb{R}^d$, which ensures that informative observations occur with probability $1/d$. This makes the effective distinguishability scale as $n/d$, which is the mechanism by which the dimension $d$
enters the privacy lower bound.

Concretely, we take a finite context space
$\mathcal{X}=\{x_0,x_1,\dots,x_{d-1}\}$ with the uniform distribution $\rho(x)=1/d$.
Since a minimax lower bound only requires exhibiting one instance within the model class, this choice is
fully admissible and directly captures the intrinsic difficulty induced by the ambient dimension $d$.

Let the action space be $\mathcal{A}=\{0,1\}$. We consider a $d$-dimensional linear reward class with feature map $\phi:\mathcal{X}\times\mathcal{A}\to\mathbb{R}^d$:
set $\phi(x_0,1)=e_1$, $\phi(x_0,0)=0$, and for all $x\neq x_0$ set $\phi(x,1)=\phi(x,0)=0$.
For a signal level $c>0$, define two parameters
\[
\left\{
\begin{aligned}
\theta_{+} &:= c\,e_1, \\
\theta_{-} &:= -c\,e_1.
\end{aligned}
\right.
\]
Note that this is a valid hard instance within the $d$-dimensional linear reward class: the feature map
$\phi(\cdot,\cdot)$ takes values in $\mathbb{R}^d$ and both $\theta_+=c e_1$ and $\theta_-=-c e_1$ belong to the
admissible parameter set (hence we are lower bounding the minimax risk over a subclass of $\mathcal{R}$).

Then the induced reward difference at the informative context is
\[
\left\{
\begin{aligned}
r_{\theta_+}(x_0,1)-r_{\theta_+}(x_0,0) &= c, \\
r_{\theta_-}(x_0,1)-r_{\theta_-}(x_0,0) &= -c,
\end{aligned}
\right.
\]
while all other contexts are deliberately uninformative.
In particular, the data distributions $P_{\theta_+}$ and $P_{\theta_-}$ coincide on $\{X\neq x_0\}$, and the two instances differ only through the rare event $\{X=x_0\}$ which occurs with probability $1/d$.

We generate pairwise preference labels from a Bradley--Terry model.\\

\noindent\textbf{Step 3. From the gap to a DP testing problem.}
Fix any $(\varepsilon,\delta)$-DP algorithm $\mathcal{A}$ and let $\hat\pi=\mathcal{A}(D)$ be its output policy.
We first rewrite the KL-regularized suboptimality gap as a KL divergence to the regularized optimizer,
and then relate this KL divergence to a two-point testing error.

Fix any $(\varepsilon,\delta)$-DP algorithm $\mathcal{A}$ and let $\hat\pi=\mathcal{A}(D)$ be its output policy.
For KL-regularized objectives, the gap admits the KL representation as presented in the upper bound proof:
\[
\mathrm{Gap}(\hat\pi;\theta)
=\frac{1}{\eta}\,\mathbb{E}_{X\sim\rho}\!\left[
\mathrm{KL}\big(\hat\pi(\cdot\vert X)\,\|\,\pi^\star_\theta(\cdot\vert X)\big)\right]
=\frac{1}{\eta d}\sum_{x\in\mathcal{X}}
\mathrm{KL}\big(\hat\pi(\cdot\vert x)\,\|\,\pi^\star_\theta(\cdot\vert x)\big).
\]
All summands are nonnegative, hence
\begin{equation}
\label{eq:gap_reduce_x0}
\mathrm{Gap}(\hat\pi;\theta)\;\ge\;\frac{1}{\eta d}\,
\mathrm{KL}\big(\hat\pi(\cdot\vert x_0)\,\|\,\pi^\star_\theta(\cdot\vert x_0)\big).
\end{equation}

We now show that a wrong-side output at $x_0$ forces a constant KL loss. Under our hard instance, $\mathcal{A}=\{0,1\}$, so $\pi(\cdot\vert x_0)$ is a Bernoulli distribution.
Write $p:=\pi(1\vert x_0)$. Under $\theta_+$ we have $\pi^\star_{\theta_+}(1\vert x_0)=\sigma(\eta c)$, hence
\[
\mathrm{KL}\big(\pi(\cdot\vert x_0)\,\|\,\pi^\star_{\theta_+}(\cdot\vert x_0)\big)
=\mathrm{KL}\big(\mathrm{Bern}(p)\,\|\,\mathrm{Bern}(\sigma(\eta c))\big).
\]

We now show that whenever $p\le 1/2$,
\begin{equation}
\label{eq:bern_kl_lc_step3}
\mathrm{KL}\big(\mathrm{Bern}(p)\,\|\,\mathrm{Bern}(\sigma(\eta c))\big)\;\ge\;\log\cosh(\eta c/2).
\end{equation}

Since $p\le 1/2$ and $q:=\sigma(\eta c)>1/2$, write
\begin{align*}
f(p)
&:=\mathrm{KL}\big(\mathrm{Bern}(p)\,\|\,\mathrm{Bern}(q)\big) \\
&= p\log\frac{p}{q}+(1-p)\log\frac{1-p}{1-q}.
\end{align*}
Differentiating in $p$ gives
\begin{align*}
f'(p)
&=\log\frac{p}{q}-\log\frac{1-p}{1-q}
=\log\Big(\frac{p(1-q)}{q(1-p)}\Big).
\end{align*}
For $p\in(0,1/2]$ and $q\in(1/2,1)$, we have $p/(1-p)\le 1$ and $(1-q)/q<1$, hence
\begin{align*}
\frac{p(1-q)}{q(1-p)} \le \frac{1-q}{q} < 1
\qquad\Rightarrow\qquad
f'(p) < 0.
\end{align*}
Therefore $f$ is decreasing on $[0,1/2]$, and thus
\begin{align*}
\mathrm{KL}\big(\mathrm{Bern}(p)\,\|\,\mathrm{Bern}(q)\big)=f(p)\ge f(1/2)
=\mathrm{KL}\big(\mathrm{Bern}(1/2)\,\|\,\mathrm{Bern}(q)\big).
\end{align*}

Now compute, letting $t:=\eta c$ and $q=\sigma(t)$,
\begin{align*}
\mathrm{KL}\big(\mathrm{Bern}(1/2)\,\|\,\mathrm{Bern}(q)\big)
&=\frac12\log\frac{1/2}{q}+\frac12\log\frac{1/2}{1-q} \\
&=\frac12\big(\log(1/2)-\log q\big)+\frac12\big(\log(1/2)-\log(1-q)\big) \\
&=\log(1/2)-\frac12\big(\log q+\log(1-q)\big) \\
&=-\log 2-\frac12\log\big(q(1-q)\big).
\end{align*}
Moreover,
\begin{align*}
q(1-q)
&=\sigma(t)\big(1-\sigma(t)\big)
=\sigma(t)\sigma(-t)
=\frac{1}{1+e^{-t}}\cdot\frac{1}{1+e^{t}} \\
&=\frac{1}{(1+e^{-t})(1+e^{t})}
=\frac{1}{2+e^{t}+e^{-t}}
=\frac{1}{2(1+\cosh t)}.
\end{align*}

Plugging this into the previous display gives
\begin{align*}
\mathrm{KL}\big(\mathrm{Bern}(1/2)\,\|\,\mathrm{Bern}(\sigma(t))\big)
&=-\log 2+\frac12\log\{2(1+\cosh t)\}\\
&=\frac12\log(1+\cosh t)-\frac12\log 2\\
&=\log\cosh(t/2),    
\end{align*}
where we used $1+\cosh t=2\cosh^2(t/2)$. This proves \eqref{eq:bern_kl_lc_step3}.

By symmetry, under $\theta_-$ we have $\pi^\star_{\theta_-}(1\vert x_0)=\sigma(-\eta c)<1/2$, and whenever $p\ge 1/2$ the same
lower bound \eqref{eq:bern_kl_lc_step3} holds with $\theta_-$.

We next connect the gap lower bound to a two-point testing error by explicitly constructing a test from the
algorithm output $\hat\pi$.

Let $\hat p:=\hat\pi(1\vert x_0)\in[0,1]$. Under $\theta_+$ we have
\[
\mathrm{Gap}(\hat\pi;\theta_+)\;\ge\;\frac{1}{\eta d}\,
\mathrm{KL}\big(\hat\pi(\cdot\vert x_0)\,\|\,\pi^\star_{\theta_+}(\cdot\vert x_0)\big)
=\frac{1}{\eta d}\,\mathrm{KL}\big(\mathrm{Bern}(\hat p)\,\|\,\mathrm{Bern}(\sigma(\eta c))\big).
\] 

If $\hat p<1/2$, then \eqref{eq:bern_kl_lc_step3} gives
$\mathrm{KL}\big(\mathrm{Bern}(\hat p)\,\|\,\mathrm{Bern}(\sigma(\eta c))\big)\ge \log\cosh(\eta c/2)$.
Therefore,
\[
\mathrm{Gap}(\hat\pi;\theta_+)\;\ge\;\frac{1}{\eta d}\log\cosh(\eta c/2)\cdot \mathbf{1}\{\hat p<1/2\}.
\]
Taking expectations under $\theta_+$ yields
\[
\mathbb{E}_{\theta_+}\!\big[\mathrm{Gap}(\hat\pi;\theta_+)\big]
\;\ge\;\frac{1}{\eta d}\log\cosh(\eta c/2)\cdot
\mathbb{P}_{\theta_+}\!\big(\hat\pi(1\vert x_0)<1/2\big).
\]

Similarly, under $\theta_-$ we have $\pi^\star_{\theta_-}(1\vert x_0)=\sigma(-\eta c)<1/2$. Using the symmetric version of
\eqref{eq:bern_kl_lc_step3} (i.e., when $\hat p\ge 1/2$),
\[
\mathbb{E}_{\theta_-}\!\big[\mathrm{Gap}(\hat\pi;\theta_-)\big]
\;\ge\;\frac{1}{\eta d}\log\cosh(\eta c/2)\cdot
\mathbb{P}_{\theta_-}\!\big(\hat\pi(1\vert x_0)\ge 1/2\big).
\]

Averaging the last two displays gives
\begin{equation}
\label{eq:gap_wrong_side_step3}
\frac{1}{2}\mathbb{E}_{\theta_+}[\mathrm{Gap}(\hat\pi;\theta_+)]+\frac{1}{2}\mathbb{E}_{\theta_-}[\mathrm{Gap}(\hat\pi;\theta_-)]
\;\ge\;
\frac{1}{\eta d}\log\cosh(\eta c/2)\cdot P_e,
\end{equation}
where
\[
P_e
:=\frac{1}{2}\mathbb{P}_{\theta_+}\!\big(\hat\pi(1\vert x_0)<1/2\big)
+\frac{1}{2}\mathbb{P}_{\theta_-}\!\big(\hat\pi(1\vert x_0)\ge 1/2\big).
\]

It is convenient to interpret $P_e$ as the error probability of an explicit test for distinguishing $\theta_+$ vs.\ $\theta_-$.
Define $\hat\theta=\hat\theta(D)\in\{+,-\}$ by
\[
\hat\theta(D) =
\left\{
\begin{aligned}
+, & \quad \text{if } \hat\pi(1\vert x_0)\ge 1/2, \\
-, & \quad \text{if } \hat\pi(1\vert x_0) < 1/2.
\end{aligned}
\right.
\]
Then, under the uniform prior on $\{\theta_+,\theta_-\}$, the Bayes error of $\hat\theta$ is exactly $P_e$:
\[
\mathbb{P}(\hat\theta(D)\neq \theta)
=\frac12\mathbb{P}_{\theta_+}(\hat\theta(D)=-)+\frac12\mathbb{P}_{\theta_-}(\hat\theta(D)=+)
=P_e.
\]
Moreover, since $\hat\theta$ is a deterministic function of the DP output $\hat\pi$, it is a post-processing of $\hat\pi$ and
hence $\hat\theta$ is also $(\varepsilon,\delta)$-DP.\\

\noindent\textbf{Step 4. Lower bound $P_e$ via DP Le Cam.}
We now lower bound the testing error $P_e$ established in \eqref{eq:gap_wrong_side_step3}, which is the Bayes error under the uniform prior on $\{\theta_+,\theta_-\}$. We invoke Lemma~\ref{lem:dp_lecam}, the DP Le Cam inequality presented in \cite{acharya2021differentially}, which requires bounding the expected Hamming distance under a coupling of the data generating distributions.

Let \(P_{\theta_+}\) and \(P_{\theta_-}\) denote the single-observation distributions induced by the two hard instances \(\theta_+\) and \(\theta_-\), respectively. Thus \(D_+\sim P_{\theta_+}^{\otimes n}\) and \(D_-\sim P_{\theta_-}^{\otimes n}\) mean that \(D_+\) and \(D_-\) are independent size-\(n\) datasets generated i.i.d. under \(\theta_+\) and \(\theta_-\), respectively. We now construct a coupling of \(D_+\) and \(D_-\) and compute
\[
D \;:=\; \mathbb{E}\big[d_{\mathrm{Ham}}(D_+,D_-)\big],
\]
where \(d_{\mathrm{Ham}}(D_+,D_-)\) denotes the Hamming distance between the two datasets, that is, the number of coordinates at which the two coupled samples differ.

Couple the datasets record-wise. For each $i\in\{1,\dots,n\}$, first couple contexts by setting
$X_i^+=X_i^-\sim \rho$ (recall $\rho$ is uniform on $\mathcal{X}$, so $\mathbb{P}(X_i^+=x_0)=1/d$).
Given $X_i^+=X_i^-$:
\begin{itemize}
\item If $X_i^+\neq x_0$, then under our hard instance the conditional label distributions coincide under $\theta_+$ and $\theta_-$,
so we set $Y_i^+=Y_i^-$ (this contributes zero to the Hamming distance).
\item If $X_i^+=x_0$, then under $\theta_+$ and $\theta_-$ the conditional label laws differ. In this case we couple
$Y_i^+\sim \mathrm{Bern}(\sigma(c))$ and $Y_i^-\sim \mathrm{Bern}(\sigma(-c))$ by a maximal coupling.
\end{itemize}
Under this construction,
\begin{align*}
\mathbb{P}\big((X_i^+,Y_i^+)\neq (X_i^-,Y_i^-)\big)
&=\mathbb{P}(X_i^+=x_0)\cdot \mathbb{P}(Y_i^+\neq Y_i^-\vert X_i^+=x_0)\\
&=\frac{1}{d}\,\mathrm{TV}\!\left(\mathrm{Bern}(\sigma(c)),\mathrm{Bern}(\sigma(-c))\right),
\end{align*}
where maximal coupling gives $\mathbb{P}(Y_i^+\neq Y_i^-)=\mathrm{TV}(\cdot,\cdot)$.

For Bernoulli distributions, $\mathrm{TV}(\mathrm{Bern}(u),\mathrm{Bern}(v))=|u-v|$. Hence,
\begin{align*}
\mathrm{TV}\!\left(\mathrm{Bern}(\sigma(c)),\mathrm{Bern}(\sigma(-c))\right)
&=\big|\sigma(c)-\sigma(-c)\big| \\
&=\left|\frac{1}{1+e^{-c}}-\frac{1}{1+e^{c}}\right| \\
&=\left|\frac{(1+e^{c})-(1+e^{-c})}{(1+e^{-c})(1+e^{c})}\right| \\
&=\left|\frac{e^{c}-e^{-c}}{2+e^{c}+e^{-c}}\right|
=\frac{e^{c}-e^{-c}}{e^{c}+2+e^{-c}} \\
&=\frac{2\sinh(c)}{2(1+\cosh(c))}
=\frac{\sinh(c)}{1+\cosh(c)}
=\tanh(c/2).
\end{align*}

Therefore,
\[
\mathbb{P}\big((X_i^+,Y_i^+)\neq (X_i^-,Y_i^-)\big)=\frac{1}{d}\tanh(c/2).
\]
Since
\begin{align*}
d_{\mathrm{Ham}}(D_+,D_-)
=\sum_{i=1}^n \mathbf{1}\!\left\{(X_i^+,Y_i^+)\neq (X_i^-,Y_i^-)\right\},
\end{align*}
taking expectations and using linearity gives
\begin{align*}
D
:=\mathbb{E}\big[d_{\mathrm{Ham}}(D_+,D_-)\big]
&=\mathbb{E}\Bigg[\sum_{i=1}^n \mathbf{1}\!\left\{(X_i^+,Y_i^+)\neq (X_i^-,Y_i^-)\right\}\Bigg] \\
&=\sum_{i=1}^n \mathbb{E}\Big[\mathbf{1}\!\left\{(X_i^+,Y_i^+)\neq (X_i^-,Y_i^-)\right\}\Big] \\
&=\sum_{i=1}^n \mathbb{P}\!\left((X_i^+,Y_i^+)\neq (X_i^-,Y_i^-)\right) \\
&=\sum_{i=1}^n \frac{1}{d}\tanh(c/2)
=\frac{n}{d}\tanh(c/2).
\end{align*}

Now apply Lemma~\ref{lem:dp_lecam}. For the $(\varepsilon,\delta)$-DP test $\hat\theta$,
\begin{align*}
P_e 
&\;\ge\;
\frac12\Big[0.9\,e^{-10\varepsilon D}-10D\delta\Big]_+ \\
&\;=\;\frac12\Big[0.9\,\exp\!\Big(-10\varepsilon\cdot \frac{n}{d}\tanh(c/2)\Big)-10\cdot \frac{n}{d}\tanh(c/2)\cdot \delta\Big]_+.
\end{align*}
Plugging this bound into \eqref{eq:gap_wrong_side_step3} yields the privacy-dependent lower bound
\begin{align*}
&\frac{1}{2}\mathbb{E}_{\theta_+}[\mathrm{Gap}(\hat\pi;\theta_+)]+\frac{1}{2}\mathbb{E}_{\theta_-}[\mathrm{Gap}(\hat\pi;\theta_-)]
\;\\\ge\;
&\frac{1}{\eta d}\log\cosh(\eta c/2)\cdot
\frac12\Big[0.9\,\exp\!\Big(-10\varepsilon\cdot \frac{n}{d}\tanh(c/2)\Big)-10\cdot \frac{n}{d}\tanh(c/2)\cdot \delta\Big]_+.    
\end{align*}
This completes the DP Le Cam step; it remains to choose $c$ and simplify the expression to obtain the two rate regimes.\\

\noindent\textbf{Step 5. Signal calibration and rate simplification.}
We now choose the signal level $c$ and simplify the privacy-dependent lower bound. Set
\[
c := \frac{d}{K\varepsilon n},
\]
where $K>0$ is a sufficiently large universal constant to be specified below. Since $\tanh(u)\le u$ for all $u\ge 0$, we have
\begin{align*}
10\varepsilon D
&=
10\varepsilon \cdot \frac{n}{d}\tanh(c/2)
\le
10\varepsilon \cdot \frac{n}{d}\cdot \frac{c}{2}
=
\frac{5}{K}.
\end{align*}
Likewise, using the theorem assumption $\delta\le \varepsilon$,
\begin{align*}
10D\delta
&=
10\cdot \frac{n}{d}\tanh(c/2)\cdot \delta
\le
10\cdot \frac{n}{d}\cdot \frac{c}{2}\cdot \delta
=
\frac{5\delta}{K\varepsilon}
\le
\frac{5}{K}.
\end{align*}
Therefore
\begin{align*}
0.9\,e^{-10\varepsilon D}-10D\delta
\ge
0.9\,e^{-5/K}-\frac{5}{K}.
\end{align*}
Choosing $K$ sufficiently large makes the right-hand side bounded below by a universal positive constant $b_0>0$. Hence
\[
P_e \ge b_1
\]
for some universal constant $b_1>0$. Plugging this into \eqref{eq:gap_wrong_side_step3} yields
\begin{equation}
\label{eq:privacy_lb_reduced}
R_n(\varepsilon,\delta)
\ge
\frac{b_2}{\eta d}\log\cosh\!\left(\frac{\eta d}{2K\varepsilon n}\right)
\end{equation}
for some universal constant $b_2>0$.

We now simplify \eqref{eq:privacy_lb_reduced}. Let
\[
u_n := \frac{\eta d}{2K\varepsilon n}.
\]
If $u_n \ge 1$, then $\log\cosh(u_n)\ge c_{\mathrm{lin}}u_n$ for a universal constant $c_{\mathrm{lin}}>0$, and therefore
\[
R_n(\varepsilon,\delta)
\ge
\frac{b_2 c_{\mathrm{lin}}}{\eta d}\,u_n
=
\frac{b_2 c_{\mathrm{lin}}}{2K}\cdot \frac{1}{n\varepsilon}.
\]
If $u_n \le 1$, then $\log\cosh(u_n)\ge c_{\mathrm{quad}}u_n^2$ for a universal constant $c_{\mathrm{quad}}>0$, and therefore
\[
R_n(\varepsilon,\delta)
\ge
\frac{b_2 c_{\mathrm{quad}}}{\eta d}\,u_n^2
=
\frac{b_2 c_{\mathrm{quad}}\eta}{4K^2}\cdot \frac{d}{n^2\varepsilon^2}.
\]

Since $\eta>0$ is fixed, the last two displays imply that there exist constants $c_\eta,C_\eta>0$ such that
\[
R_n(\varepsilon,\delta)
\ge
c_\eta\,\min\!\left\{\frac{1}{n\varepsilon},\;\frac{d}{n^2\varepsilon^2}\right\},
\]
and, moreover, whenever $n \ge n_{\mathrm P}:=C_\eta d/\varepsilon$, the quadratic branch is active and
\[
R_n(\varepsilon,\delta)
\ge
c_\eta\,\frac{d}{n^2\varepsilon^2}.
\]

Combining this privacy-dependent lower bound with the non-private lower bound from Step~1 yields
\[
R_n(\varepsilon,\delta)
\ge
c_\eta\,\max\!\left\{\frac{d}{n},\;\min\!\left(\frac{1}{n\varepsilon},\;\frac{d}{n^2\varepsilon^2}\right)\right\},
\]
for all $n \ge n_{\mathrm{NP}}:=C_\eta d$, and
\[
R_n(\varepsilon,\delta)
\ge
c_\eta\,\max\!\left\{\frac{d}{n},\;\frac{d}{n^2\varepsilon^2}\right\},
\]
for all $n \ge \max\{n_{\mathrm{NP}},n_{\mathrm P}\}$. This completes the proof.

\section{Supporting Lemma}

\begin{lemma}[Remark 5.3 in \citep{tropp2012user}]\label{lem:tropp-simplified} Let $\{X_i\}_{i=1}^n$ be independent, random, positive-semidefinite, symmetric matrices in $\mathbb{R}^{d\times d}$. Assume $\lambda_{\max}(X_i)\le R$ almost surely for all $i$, and define
\[
\mu_{\min} \;:=\; \lambda_{\min}\!\Big(\sum_{i=1}^n \mathbb{E}[X_i]\Big).
\]
Then for any $t\in[0,1]$,
\[
\mathbb{P}\!\left(\lambda_{\min}\!\Big(\sum_{i=1}^n X_i\Big)\le t\,\mu_{\min}\right)
\;\le\;
d\cdot \exp\!\left(-\frac{(1-t)^2\,\mu_{\min}}{2R}\right).
\]
\end{lemma}

\begin{lemma}[Harvey last-iterate bound, scaled]
\label{lem:harvey-scaled}
Let $f$ be $\mu$-strongly convex and $L$-Lipschitz on a closed convex set $\Theta$.
Consider iterates
\[
x_{t+1}=\Pi_\Theta\!\bigl(x_t-\eta_t(\nabla f(x_t)-z_t)\bigr),
\qquad
\eta_t=\frac{1}{\mu t},
\]
where $(z_t)$ is adapted and satisfies $\mathbb E[z_t\vert \mathcal F_{t-1}]=0$ and $\|z_t\|_2\le Z$ a.s.\ for all $t$.
Then there exists a universal constant $c>0$ such that for any $\delta\in(0,1)$,
\begin{equation}
\label{eq:harvey-scaled}
\mathbb P\!\left(
f(x_{T+1})-f(x^*)
\le
c\cdot \frac{\log T\cdot \log\!\frac{1}{\delta}}{T}\cdot\frac{(L+Z)^2}{\mu}
\right)\ge 1-\delta,
\qquad x^*:=\arg\min_{x\in\Theta} f(x).
\end{equation}
\end{lemma}

\begin{proof}[Proof of Lemma~\ref{lem:harvey-scaled}]
This is a direct rescaling of \citet[Theorem~3.1]{harvey2019tight}.
Let $G:=L+Z$ and define $\tilde f := f/G$. Then $\tilde f$ is $1$-Lipschitz and $(\mu/G)$-strongly convex.
Also $\tilde z_t:=z_t/G$ satisfies $\|\tilde z_t\|\le 1$ a.s.\ and $\mathbb E[\tilde z_t\vert\mathcal F_{t-1}]=0$.
Writing the update in terms of $\tilde f$ gives
\[
x_{t+1}
=
\Pi_\Theta\!\Bigl(x_t-\eta_t G\bigl(\nabla \tilde f(x_t)-\tilde z_t\bigr)\Bigr).
\]
With $\eta_t=1/(\mu t)$, we have $\eta_t G = 1/((\mu/G)t)$, i.e., the step size used by
\citet[Theorem~3.1]{harvey2019tight} for $\tilde f$. Applying that theorem yields
\[
\tilde f(x_{T+1})-\tilde f(x^*)
\le
c\cdot \frac{\log T\cdot \log(1/\delta)}{T}\cdot\frac{1}{\mu/G}
\]
with probability at least $1-\delta$. Multiplying by $G$ gives \eqref{eq:harvey-scaled}.
\end{proof}

\begin{lemma}[Tail bound for $\chi^2$, \citep{laurent2000adaptive}]
\label{lem:laurent-massart}
Let $U\sim \chi^2_d$. Then for every $x>0$,
\begin{align}
\label{eq:lm-upper}
\mathbb P\!\left(
U \ge d + 2\sqrt{dx} + 2x
\right)
\le e^{-x}.
\end{align}
Moreover, for every $x>0$,
\begin{align}
\label{eq:lm-lower}
\mathbb P\!\left(
U \le d - 2\sqrt{dx}
\right)
\le e^{-x}.
\end{align}
\end{lemma}

\begin{lemma}[Utility of noisy SGD for strongly convex ERM, adapted from Theorem~2.4, \citep{bassily2014private}]
\label{lem:dpsgd_utility}
Fix a dataset $D=\{z_i\}_{i=1}^n$ and define the average empirical loss
\[
\bar L_n(\theta)=\frac1n\sum_{i=1}^n \ell(\theta;z_i),
\qquad
\hat\theta\in\arg\min_{\theta\in\Theta}\bar L_n(\theta).
\]
Assume $\ell(\cdot;z)$ is $G$-Lipschitz on $\Theta$ for all $z$ and $\bar L_n$ is $\mu$-strongly convex on $\Theta$.
Then there exists an $(\varepsilon,\delta)$-DP noisy SGD algorithm whose output $\tilde\theta$ satisfies
\[
\mathbb E\!\left[\bar L_n(\tilde\theta)-\bar L_n(\hat\theta)\,\middle|\,D\right]
\;\le\;
O\!\left(
\frac{G^2\,d\,\log^2(n/\delta)\,\log(1/\delta)}{\mu\,n^2\,\varepsilon^2}
\right),
\]
where the expectation is over the algorithmic randomness conditional on $D$.
Equivalently, in $\widetilde O(\cdot)$ notation,
\[
\mathbb E\!\left[\bar L_n(\tilde\theta)-\bar L_n(\hat\theta)\,\middle|\,D\right]
=\widetilde O\!\left(\frac{G^2\,d}{\mu\,n^2\,\varepsilon^2}\right).
\]
\end{lemma}

\begin{lemma}[Non-private MLE statistical error under PD design (via \citep{zhu2023principled})]
\label{lem:mle_stat}
Let $D=\{(x_i,a_i^w,a_i^\ell,y_i)\}_{i=1}^n$ be i.i.d.\ and let
$\Delta\phi_i := \phi(x_i,a_i^w)-\phi(x_i,a_i^\ell)\in\mathbb R^d$.
Define the empirical covariance
\[
\Sigma_D \;:=\; \frac{1}{n}\sum_{i=1}^n \Delta\phi_i\Delta\phi_i^\top.
\]
Assume (i) $\|\Delta\phi_i\|_2\le L_\Delta$ a.s., (ii) $\|\theta\|_2\le R$ for all $\theta\in\Theta$,
and (iii) the population covariance is nondegenerate:
\[
\Sigma \;:=\; \mathbb E[\Delta\phi\,\Delta\phi^\top]\ \succeq\ \lambda I_d
\qquad\text{for some }\lambda>0.
\]
Let $\hat\theta\in\arg\min_{\theta\in\Theta}\bar L_n(\theta)$ be the (non-private) MLE.
Define the curvature constant
\[
\gamma \;:=\; \inf_{|t|\le L_\Delta R}\ \sigma(t)\bigl(1-\sigma(t)\bigr)
\;=\;\frac{1}{2+e^{L_\Delta R}+e^{-L_\Delta R}}.
\]
Fix any $\rho\in(0,1)$. If
\[
n \;\ge\; \frac{8L_\Delta^2}{\lambda}\log\!\Bigl(\frac{2d}{\rho}\Bigr),
\]
then with probability at least $1-\rho$,
\[
\|\hat\theta-\theta^\star\|_2
\;\le\;
\frac{\sqrt{2}}{\sqrt{\lambda}}\,
C_{\mathrm{MJ}}\,
\sqrt{\frac{d+\log(2/\rho)}{\gamma^2 n}},
\]
where $C_{\mathrm{MJ}}>0$ is the universal constant appearing in \citet[Lemma~3.1]{zhu2023principled}.
\end{lemma}

\begin{proof}
Set $\delta:=\rho/2$.

\noindent\emph{Step 1 (MJ bound with $\lambda_{\rm reg}=0$).}
Apply \citet[Lemma~3.1]{zhu2023principled} with regularization parameter $\lambda_{\rm reg}=0$.
With probability at least $1-\delta$,
\begin{align*}
\|\hat\theta-\theta^\star\|_{\Sigma_D}
&\le
C_{\mathrm{MJ}}\sqrt{\frac{d+\log(1/\delta)}{\gamma^2 n}}
=
C_{\mathrm{MJ}}\sqrt{\frac{d+\log(2/\rho)}{\gamma^2 n}}.
\end{align*}

\noindent\emph{Step 2 (Empirical PD from population PD).}
Since $\Delta\phi_i\Delta\phi_i^\top\succeq 0$, $\|\Delta\phi_i\|_2\le L_\Delta$ implies
$\lambda_{\max}(\Delta\phi_i\Delta\phi_i^\top)\le L_\Delta^2$ a.s.
Moreover,
\[
\mathbb E[\Delta\phi\Delta\phi^\top]=\Sigma\succeq \lambda I_d
\quad\Longrightarrow\quad
\lambda_{\min}\bigl(\mathbb E[\Sigma_D]\bigr)=\lambda_{\min}(\Sigma)\ge \lambda.
\]
Thus, a standard matrix Chernoff bound yields that if
$n \ge \frac{8L_\Delta^2}{\lambda}\log\bigl(\frac{2d}{\rho}\bigr)$,
then with probability at least $1-\delta$,
\[
\lambda_{\min}(\Sigma_D)\ \ge\ \frac{\lambda}{2}.
\]

\noindent\emph{Step 3 (Convert to $\ell_2$).}
On the intersection of the two events above (which holds with probability at least $1-2\delta=1-\rho$),
\begin{align*}
\|\hat\theta-\theta^\star\|_{\Sigma_D}^2
&=(\hat\theta-\theta^\star)^\top\Sigma_D(\hat\theta-\theta^\star)
\;\ge\;\lambda_{\min}(\Sigma_D)\,\|\hat\theta-\theta^\star\|_2^2
\;\ge\;\frac{\lambda}{2}\,\|\hat\theta-\theta^\star\|_2^2,
\end{align*}
so
\begin{align*}
\|\hat\theta-\theta^\star\|_2
&\le \sqrt{\frac{2}{\lambda}}\ \|\hat\theta-\theta^\star\|_{\Sigma_D}
\;\le\;
\frac{\sqrt{2}}{\sqrt{\lambda}}\,
C_{\mathrm{MJ}}\,
\sqrt{\frac{d+\log(2/\rho)}{\gamma^2 n}}.
\end{align*}
This proves the claim.
\end{proof}

\begin{proposition}[Non-private minimax lower bound for KL-regularized RLHF (Theorem~4.6 in \citep{zhao2024sharp})]
\label{prop:np_lb_from_2411}
Fix a target gap level $\mu\in(0,1/256)$ and $\eta>4$.
Consider the KL-regularized RLHF setting with preference feedback.
Then for any (possibly randomized) algorithm $\mathcal A$ that, given $n$ i.i.d.\ preference samples,
outputs a policy $\hat\pi=\mathcal A(D)$, there exists a KL-regularized preference learning instance
(with two actions, a finite context space, a reward function class $\mathcal R$, and coverage coefficient
of order $O(N_{\mathcal R}(\mu))$) such that achieving suboptimality gap at most $\mu$
requires
\[
n \;=\; \Omega\!\left(\min\left\{\frac{\eta \log N_{\mathcal R}(\mu)}{\mu},\ 
\frac{\log N_{\mathcal R}(\mu)}{\mu^{2}}\right\}\right).
\]
Equivalently, if $n$ is smaller than the above order, then $\mathcal A$ cannot guarantee gap $\le \mu$
uniformly over that problem class.
\end{proposition}

\begin{lemma}[Pinsker's Inequality]\label{lemma: pinsker inequ} If $\mathbb{P}_{1},\mathbb{P}_{2}$ are two probability measures on a common measurable space $(\Omega,\mathcal{F})$, then it holds that
\begin{equation*}
    \delta(\mathbb{P}_{1},\mathbb{P}_{2}) \leq \sqrt{\frac{1}{2}\text{KL}(\mathbb{P}_{1}\Vert \mathbb{P}_{2})},
\end{equation*}
    where $\delta(\cdot , \cdot)$ is the total variation distance and $\text{KL}(\mathbb{P}_{1}\Vert \mathbb{P}_{2})$ is the Kullback-Leibler divergence.
\end{lemma}

\begin{lemma}[(\texorpdfstring{$(\varepsilon,\delta)$}{(ε,δ)}-DP) Le Cam lower bound with coupling \citep{acharya2021differentially}]\label{lem:dp_lecam}
Let $p_1\in\mathrm{co}(P_1)$ and $p_2\in\mathrm{co}(P_2)$ be two distributions on $\mathcal Z^n$. Let $(X,Y)$ be any coupling of $p_1$ and $p_2$ such that $D:=\E[\Ham(X,Y)]<\infty$. Then for any $(\varepsilon,\delta)$-DP hypothesis testing algorithm $\hat\theta$ that outputs $\{1,2\}$, the (minimax) testing risk satisfies
\begin{equation}
\label{eq:acharya_thm1}
P_e(\hat\theta;P_1,P_2)
\ \ge\
\frac12\max\Big\{\,1-\TV(p_1,p_2),\ 0.9e^{-10\varepsilon D}-10D\delta\,\Big\}.
\end{equation}
\end{lemma}


\end{document}